%% file: main.tex
\documentclass[10pt,journal,compsoc]{IEEEtran}

\input{style/packages.tex}

\def\methodformat#1{{\ttfamily #1}\xspace}
\def\tblue{}
\newcommand{\IRICPLEX}{\methodformat{Our-CPLEX}}
\newcommand{\IRITRWS}{\methodformat{Our-TRWS}}
\newcommand{\PBPCPLEX}{\methodformat{\cite{SwobodaPersistencyCVPR2014}-CPLEX}}
\newcommand{\PBPTRWS}{\methodformat{\cite{SwobodaPersistencyCVPR2014}-TRWS}}
\newcommand{\PBPTRWSO}{\methodformat{\cite{Swoboda-PAMI-16}-TRWS}}
\newcommand{\Shekhovt}{\methodformat{$\varepsilon$-L1\,\cite{shekhovtsov-14}}}
\newcommand{\Kovtun}{\methodformat{Kovtun\cite{Kovtun03}}}
\newcommand{\MQPBO}{\methodformat{MQPBO}}
\newcommand{\MMQPBO}{\methodformat{MQPBO-10}}

\def\mytitle{Maximum Persistency via Iterative Relaxed Inference in Graphical Models}

\begin{document}

\title{\mytitle}

\author{Alexander Shekhovtsov,
        Paul Swoboda,
        and~Bogdan~Savchynskyy
\IEEEcompsocitemizethanks{
\IEEEcompsocthanksitem 
Alexander Shekhovtsov is with the 
Institute for Computer Graphics and Vision (ICG), Graz University of Technology, Inffeldgasse 16, Graz 8010, Austria.
\protect\\
E-mail: shekhovtsov@icg.tugraz.at.


\IEEEcompsocthanksitem 
Paul Swoboda is with the Discrete Optimization Group, Institute of Science and Technology Austria, Am Campus 1, 3400 Klosterneuburg, Austria.
\protect\\
E-mail: pswoboda@ist.ac.at.

\IEEEcompsocthanksitem 
Bogdan Savchynskyy is with Computer Vision Lab, Faculty of Computer Science, Institute for Artificial Intelligence, Dresden University of Technology, 01062 Dresden, Germany. 
\protect\\
E-mail: bogdan.savchynskyy@tu-dresden.de.
}%
\thanks{}}

\markboth{}{}

\IEEEtitleabstractindextext{%
\begin{abstract}
\input{tex/abstract.tex}
\end{abstract}
}
\IEEEdisplaynontitleabstractindextext

\maketitle

\newlength{\figwidth}
\setlength{\figwidth}{\linewidth}

\setlength\cftparskip{0pt}
\setlength\cftbeforesecskip{2pt}
\setlength\cftbeforesubsecskip{1pt}

\begin{textblock}{14}(1,0.5)
\noindent
\textcopyright\ Copyright notice.
\end{textblock}

\SetAlgorithmName{Algorithm}{Algorithm}{alg.} 
\SetAlgoProcName{Procedure}{Procedure}{}

\input{tex/intro.tex}
\input{tex/preliminaries.tex}
\input{tex/alg_primal.tex}
\input{tex/alg_dual.tex}
\input{tex/reduction.tex}
\input{tex/alg_dual_approx.tex}
\input{tex/speedups.tex}
\input{tex/experiments.tex}
\input{tex/conclusion.tex}


\input{tex/acknowlegement.tex}

{\small
\bibliographystyle{apa}
\bibliography{bib/strings,bib/persistence-nips2014,bib/books,bib/optim,bib/max-plus-en,bib/kiev-en}
}
\clearpage
\input{tex/appendix.tex}

%
\end{document}

%% file: style/packages.tex
\usepackage{etex}
\usepackage{url}
\usepackage[numbers]{natbib}
\renewcommand\bibname{References}
\usepackage{cvpr-abbriv}
\usepackage{epsfig}
\usepackage{graphicx}
\usepackage{amsmath}
\usepackage{amssymb}
\usepackage{array}
\usepackage{tabularx}
\usepackage{url}
\usepackage{graphicx}
\usepackage{amsmath,amssymb} 
\makeatletter
\let\proof\@undefined                        
\let\endproof\@undefined                  
\makeatother
\usepackage{amsthm}
\usepackage{amsthm}
\usepackage{amsxtra}
\usepackage{stmaryrd}
\usepackage{comment}
\usepackage[ruled,vlined,linesnumbered,procnumbered,longend]{algorithm2e}

\SetCommentSty{mycommfont}
\usepackage{multirow,rotating}
\usepackage{array}
\usepackage{afterpage}
\usepackage{dblfloatfix}
\usepackage{setspace}
\usepackage{booktabs}
\usepackage{color, colortbl}
\usepackage{tabularx,calc}
\usepackage{tikz}                 
\usetikzlibrary{patterns,fit,external}
\usepackage{pgfplots}             
\usepackage{overpic}
\usepackage{tocloft}
\usepackage{bbm}
\usepackage[absolute]{textpos}
\usepackage{atbeginend}
\usepackage{units}
\usepackage{xifthen}

\usepackage{times}
\usepackage{tikzscale}
\usepackage{marginnote}
\usepackage[textwidth=2cm,figwidth=\columnwidth,colorinlistoftodos]{todonotes}
\makeatletter
\renewcommand{\todo}[2][]{\tikzexternaldisable\@todo[#1]{#2}\tikzexternalenable}
\makeatother

\newcounter{mycomment} 

\usepackage{booktabs}
\usepackage{multirow}
\usepackage{tabularx}
\usepackage{array}
\usepackage[
  position=bottom,
  font=small,
  labelfont={bf},
]{caption,subfig}
\usepackage{longtable}
\usepackage{tabu} 
\usepackage[most]{tcolorbox}
\usepackage{pbox}
\newlength{\luw}
\newlength{\luh}
\newcommand{\flabels}[1]{%
    \settowidth{\luw}{\begin{tabular}{l}\,#1\,\end{tabular}}%
		\settototalheight{\luh}{%
		\begin{tcolorbox}[width=\luw+5pt, boxsep=1pt, left=0pt, right=0pt, top=0pt, bottom=0pt,nobeforeafter]%
		\begin{tabular}{l}%
		\,#1\,\vphantom{$f^1$fq}%
		\end{tabular}%
		\end{tcolorbox}%
		}%
		\pbox[t][][b]{\textwidth}{%
		\resizebox{!}{0.8\luh}{%
		\begin{tcolorbox}[width=\luw+5pt, boxsep=1pt, left=0pt, right=0pt, top=0pt, bottom=0pt,nobeforeafter]%
		\begin{tabular}{l}%
		\,#1\,\vphantom{$f^1$fq}%
		\end{tabular}%
		\end{tcolorbox}%
		}%
		}%
}%

\usepackage{xifthen}

\newcommand{\labelgraphics}[3][width=0.5\linewidth]{%
\setlength{\fboxsep}{0pt}%
\tikzset{external/export next=false}%
\begin{tikzpicture}%
  \node[inner sep = 0pt] (a) {\includegraphics[#1]{#2}};
  \node[anchor=north west,inner sep=1pt] at (a.north west) {\flabels{#3}};
\end{tikzpicture}%
}%

\usepackage[draft=false, pagebackref=false,breaklinks=true,colorlinks,bookmarks=false]{hyperref}
\hypersetup{hidelinks,colorlinks=false,linkbordercolor={1 1 1}}
\input{style/tex_defs.tex}
\input{style/myspace.tex}
\usepackage{thmtools, thm-restate}
\usepackage[capitalise]{cleveref}
\crefformat{equation}{(#2#1#3)}
\crefname{section}{\parSym}{\parSym\parSym}
\Crefname{section}{\parSym}{\parSym\parSym}

\usepackage[bottom]{footmisc}

%% file: style/tex_defs.tex
\newcommand{\myparagraph}[1]{\subsection{#1}}

\newcommand{\SE}{{\cal E}}

\newcommand{\SX}{{\cal X}}

\newcommand{\SY}{{\cal Y}}

\newcommand{\SM}{{\cal M}}
\newcommand{\M}{{\cal M}}

\newcommand{\SI}{{\cal I}}
\newcommand{\SO}{{\cal O}}
\newcommand{\SP}{{\cal P}}

\newcommand{\SV}{{\cal V}}
\newcommand{\V}{{\mathcal V}}

\newcommand{\BR}{\mathbb R}

\newcommand{\BS}{\mathbb S}

\newcommand{\Nb}{{\N}}

\newcommand{\lan}{\left\langle}
\newcommand{\ran}{\right\rangle}
\newcommand{\<}{\langle}
\renewcommand{\>}{\rangle}
\newcommand{\T}{^{\mathsf T}}

\newcommand{\X}{{\SX}}
\newcommand{\Y}{{\SY}}
\makeatletter
\let\parSym\S
\def\S{\mathcal{S}}
\makeatother

\renewcommand{\O}{{\SO}}
\newcommand{\Real}{\mathbb{R}}

\newcommand{\I}{\mathcal{I}}
\newcommand{\E}{\mathcal{E}}

\newcommand{\Algorithm}[1]{{Algorithm\,\ref{#1}}}
\newcommand{\Proposition}[1]{{Proposition\,\ref{#1}}}
\newcommand{\tab}{{\hphantom{bla}}}
\def\N{\mathcal{N}}

\def\mathrlap{\mathpalette\mathrlapinternal} 
\def\mathllap{\mathpalette\mathllapinternal}
\def\mathllapinternal#1#2{\llap{$\mathsurround=0pt#1{#2}$}}
\def\mathrlapinternal#1#2{\rlap{$\mathsurround=0pt#1{#2}$}}
\def\leftbb{\mathrlap{[}\hskip1.3pt[}
\def\rightbb{]\hskip1.36pt\mathllap{]}}

\DeclareMathOperator*{\conv}{conv}

\newcommand{\argmin}{\mathop{\rm argmin}}

\newcommand{\Section}[1]{\cref{#1}}
\newcommand{\Figure}[1]{\cref{#1}}
\newcommand{\Theorem}[1]{\cref{#1}}
\newcommand{\Lemma}[1]{\cref{#1}}

\let\emptyset\varnothing

\def\subset{\subseteq}
\def\supset{\supseteq}

\usepackage{array}
\newcolumntype{L}[1]{>{\raggedright\let\newline\\\arraybackslash\hspace{0pt}}m{#1}}
\newcolumntype{C}[1]{>{\centering\let\newline\\\arraybackslash\hspace{0pt}}m{#1}}
\newcolumntype{R}[1]{>{\raggedleft\let\newline\\\arraybackslash\hspace{0pt}}m{#1}}
\definecolor{Gray}{gray}{0.9}
\newcolumntype{g}{>{\columncolor{Gray}}r}
\newcolumntype{G}{>{\columncolor{Gray}}c}

%% file: style/myspace.tex
\newlength{\myskip}
\renewenvironment{enumerate}%
  {\begin{list}{\arabic{enumi}.}%
     {\topsep=0in\itemsep=0in\parsep=0pt\partopsep=0in\usecounter{enumi}}%
   }{\end{list}}

\renewenvironment{itemize}%
  {\begin{list}{$\bullet$}%
     {\topsep=0in\itemsep=0pt\parsep=0pt\partopsep=0in\usecounter{itemi}}%
   }{\end{list}\addvspace{0pt}}

\SetAlgoSkip{skipalgomyskip}
\SetAlgoInsideSkip{}

\makeatletter
\let\corollary\@undefined
\let\endcorollary\@undefined
\let\definition\@undefined
\let\enddefinition\@undefined
\let\proof\@undefined
\let\endproof\@undefined
\let\theorem\@undefined
\let\c@theorem\@undefined
\let\endtheorem\@undefined
\let\lemma\@undefined
\let\endlemma\@undefined
\let\example\@undefined
\let\c@example\@undefined
\let\endexample\@undefined
\let\remark\@undefined
\let\endremark\@undefined
\let\proposition\@undefined
\let\endproposition\@undefined
\let\property\@undefined
\let\endproperty\@undefined
\makeatother

\usepackage{thmtools}

\newtheoremstyle{tightItalic}
  {0.5\myskip}
  {0\myskip}
  {}
  {}
  {\itshape}
  {.}
  { }
  {}

\newtheoremstyle{tightBf}
  {0.5\myskip}
  {0.5\myskip}
  {}
  {}
  {\bf}
  {.}
  {.5em}
  {}

\theoremstyle{definition}
\theoremstyle{tightBf}
\declaretheorem[style=tightBf,parent=section]{thm}
\numberwithin{thm}{section}
\declaretheorem[style=tightBf,sibling=thm,name=Theorem]{theorem}
\declaretheorem[style=tightBf,sibling=theorem,name=Lemma]{lemma}
\declaretheorem[style=tightBf,sibling=theorem,name=Definition]{definition}
\declaretheorem[style=tightBf,sibling=theorem,name=Corollary]{corollary}
\declaretheorem[style=tightBf,sibling=theorem,name=Proposition]{proposition}

\declaretheorem[style=tightBf,parent=section,name=Example]{example}
\numberwithin{example}{section}

\declaretheorem[style=tightItalic,sibling=theorem,name=Remark]{remark}

\theoremstyle{tightItalic}
\newtheorem*{proof}{Proof}
\BeforeEnd{proof}{\qed}

%% file: tex/abstract.tex
We consider the NP-hard problem of MAP-inference for undirected discrete graphical models. We propose a polynomial time and practically efficient algorithm for finding a part of its optimal solution. Specifically, our algorithm marks some labels of the considered graphical model either as (i) {\em optimal}, meaning that they belong to all optimal solutions of the inference problem; (ii)~{\em non-optimal} if they provably do not belong to any solution.
With access to an exact solver of a linear programming relaxation to the MAP-inference problem, our algorithm marks the maximal possible (in a specified sense) number of labels. 
We also present a version of the algorithm, which has access to a suboptimal dual solver only and still can ensure the (non-)optimality for the marked labels, although the overall number of the marked labels may decrease.
We propose an efficient implementation, which runs in time comparable to a single run of a suboptimal dual solver. 
Our method is well-scalable and shows state-of-the-art results on computational benchmarks from machine learning and computer vision.

%% file: tex/intro.tex
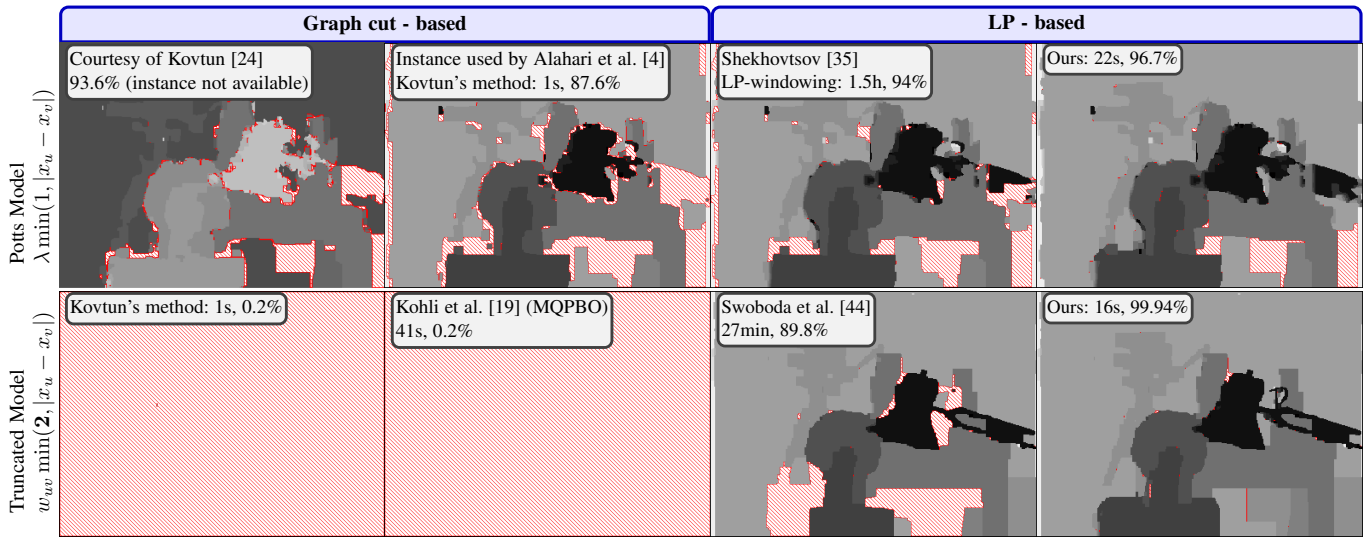
\begin{figure*}[!tb]
\centering
\resizebox{\linewidth}{!}{
\parbox{1.2\linewidth}{
\input{tsukuba/tsukuba.tex}
}
}
\caption{\label{fig:teaser-progress}
Progress of partial optimality methods. The top row corresponds to a stereo model with Potts interactions and large aggregating windows for unary costs used in~\cite{Kovtun03,Alahari-08} (instance published by~\cite{Alahari-08}). The bottom row is a more refined stereo model with truncated linear terms~\cite{SzeliskiComparativeStudyMRF} (instance from~\cite{OpenGMBenchmarkWebsite}). 
The hashed red area indicates that the optimal persistent label in the pixel is not found (but some non-optimal labels might have been eliminated). Solution completeness is given by the percentage of persistent labels.
Graph cut based methods are fast but only efficient for strong unary terms. LP-based methods are able to determine larger persistent assignments but are extremely slow prior to this work. 
}
\end{figure*}
\IEEEraisesectionheading{\section{Introduction}\label{sec:intro}}

\IEEEPARstart{W}{e} consider {\em the energy minimization} or {\em maximum a posteriori (MAP) inference} problem for discrete graphical models. 	
In the most common pairwise case it has the form 
\begin{equation}\label{equ:energyMinimization}
 \min_{x\in\X}E_{ f}(x):=f_\emptyset + \sum_{v\in \SV} f_v(x_v) + \sum_{uv\in \SE} f_{uv}(x_u,x_v)\,,
\end{equation}
where minimization is performed over vectors $x$, containing the discrete-valued components $x_v$. Further notation is to be detailed in \Section{sec:preliminaries}. The problem has numerous applications in computer vision, machine learning, communication theory, signal processing, information retrieval and statistical physics,  see~\cite{kappes-2015-ijcv,GraphicalModelsWainwrightJordan,PIC2011} for an overview of applications. Even in the binary case, when each coordinate of $x$ can be assigned two values only, the problem is known to be NP-hard and is also hard to approximate~\cite{Li2016}.

Hardness of the problem justifies a number of existing approximate methods addressing it~\cite{kappes-2015-ijcv}. Among them, solvers addressing its linear programming (LP) relaxations 
and in particular, the LP dual~\cite{Schlesinger-76,ALinearProgrammingApproachToMaxSumWerner,TRWSKolmogorov}, count among the most versatile and efficient ones. However, apart from some notable exceptions (see the overview of related work below), approximate methods can not guarantee neither optimality of their solutions as a whole, nor even optimality of any individual solution coordinates.
That is, if~$x$ is a solution returned by an approximate method and~$x^*$ is an optimal one, there is no guarantee that $x^*_v=x_v$ for any coordinate~$v$.

In contrast, our method provides such guarantees for some coordinates. More precisely, for each component $x_v$ it eliminates those of its values (henceforth called {\em labels}), which {\em provably} can not belong to any optimal solution. We call these eliminated labels {\em persistent non-optimal}. Should a single label $a$ remain non-eliminated, it implies that for all optimal solutions $x^*$ it holds $x^*_v=a$ and the label $a$ is called {\em persistent optimal}. 
Our elimination method is polynomial and is applicable with {\em any} (approximate) solver for a dual of a linear programming relaxation of the problem, employed as a subroutine. 							

\subsection{Related Work}\label{sec:related}
A trivial but essential observation is that any method identifying persistency has to be based on tractable sufficient conditions in order to avoid solving the NP-hard problem~\eqref{equ:energyMinimization}. 

{\bf Dead-end elimination methods (DEE)~\cite{DeadEndEliminationDesmet}} verify local sufficient conditions by inspecting a given node and its immediate neighbors at a time. When a label in the node can be substituted with another one such that the energy for all configurations of the neighbors does not increase, this label can be eliminated without loss of optimality.
\par
A similar principle for eliminating interchangeable labels was proposed in constraint programming~\cite{Freuder:1991}.
It's generalization to a related problem of Weighted Constraint Satisfaction (WCS) is known as {\em dominance rules} or {\em soft neighborhood substitutability}. 
However, because the WCS in general considers a bounded ``$+$'' operation, the condition appears to be intractable and therefore weaker sufficient local conditions were introduced, \eg,~\cite{conf/cp/LecoutreRD12,conf/cp/GivryPO13}. The way~\cite{conf/cp/GivryPO13} selects a local substitute label using equivalence preserving transforms is related to our method, in which we use an approximate solution based on the dual of the LP relaxation as a tentative substitute (or test) labeling. 
\par
Although the local character of the DEE methods allows for an efficient implementation, it also significantly limits their quality, \ie,~the number of found persistencies.
As shown in~\cite{shekhovtsov-14,SwobodaPersistencyCVPR2014,Wang_2016_CVPR}, considering more global criteria may significantly increase the algorithm's quality.
\par
{\bf The roof dual relaxation} in quadratic pseudo-Boolean Optimization (QPBO)~\cite{PseudoBooleanOptimizationBorosHammer,ExtendedRoofDuality} (equivalent to pairwise energy minimization with binary variables) has the property that all variables that are integer in the relaxed solution are persistent. 
Several generalizations of roof duality to higher-order energies were proposed (\eg, \cite{Adams:1998, Kolmogorov12-bisub}).
The MQPBO method~\cite{PartialOptimalityInMultiLabelMRFsKohli} and the generalized roof duality~\cite{GeneralizedRoofDualityWindheuser} extend roof duality to the multi-label case by reducing the problem to binary variables and generalizing the concept of submodular relaxation~\cite{Kolmogorov12-bisub}, respectively.
Although for binary pairwise energies these methods provide a very good trade-off between computational efficiency and a number of found persistencies, their efficacy drops as the number of label grows.

{\bf Auxiliary submodular problems} were proposed in~\cite{Kovtun03,Kovtun-10} as a sufficient persistency condition for multilabel energy minimization. In the case of Potts model, the method has a very efficient specialized algorithm~\cite{Gridchyn-13}.
Although these methods have shown very good efficacy for certain problem classes appearing in computer vision, the number of  persistencies they find drastically decreases when the energy does not have strong unary terms (see \cref{fig:teaser-progress}).
\par
In contrast to the above methods that technically rely either on local conditions or on computing a maximum flow (min-cut), the works~\cite{PartialOptimalityByPruningPotts,SwobodaPersistencyCVPR2014,Swoboda-PAMI-16} and~\cite{shekhovtsov-14} proposed {\bf persistency approaches relying on a general linear programming relaxation}. Authors of~\cite{PartialOptimalityByPruningPotts,SwobodaPersistencyCVPR2014,Swoboda-PAMI-16} demonstrated applicability of their approach to large-scale problems by utilizing existing efficient approximate MAP-inference algorithms, while in~\cite{shekhovtsov-14} the large-scale problems are addressed using a windowing technique. Despite the superior persistency results, the running time of the approximate-LP-based methods remained prohibitively slow for practical applications as illustrated by an example in~\Figure{fig:teaser-progress}.

Not only LP-based methods can achieve superior results in practice, but they are even theoretically guaranteed to do so, as proven for the method~\cite{shekhovtsov-14,shekhovtsov-15-HO}. 
In this method, the problem of determining the {\em maximum number of persistencies} is formulated as a polynomially solvable linear program. It is guaranteed to find a {\em provably larger} persistency assignment than most of the above mentioned approaches. However, solving this linear program for large scale instances is numerically unstable/intractable and applying it to multiple local windows is prohibitively slow. 
This poses a challenge of designing an LP-based method that would be indeed practical.

\myparagraph{Contribution\label{sec:contrib}}
In this work we propose a method which solves the same {\em maximum persistency problem} as in~\cite{shekhovtsov-14} and therefore delivers {\em provably} better results than other methods. Similar to~\cite{Swoboda-PAMI-16}, our method requires to iteratively (approximately) solve the linear programming relaxation of~\eqref{equ:energyMinimization} as a subroutine.
However, our method is significantly faster than~\cite{shekhovtsov-14,Swoboda-PAMI-16} due to a substantial theoretical and algorithmic elaboration of this subroutine.
\par
We demonstrate the efficiency of our approach on benchmark problems from machine learning and computer vision. We outperform {\em all} competing methods in terms of the number of persistent labels and method~\cite{shekhovtsov-14} in speed and scalability.
On randomly generated small problems, we show that the set of persistent labels found using approximate LP solver is close to the maximal one as established by the (costly and not scalable) method~\cite{shekhovtsov-14}.
\par
The present paper is a revised version of~\cite{SSS-15-IRI}. Besides reworked explanations, shortened and clarified proofs, one new technical extension is a more general dual algorithm, with termination guarantees for a larger class of approximate solvers.

%% file: tsukuba/tsukuba.tex
\newcommand*\rot{\rotatebox{90}}

\newcommand{\flabelw}[1]{%
    \settowidth{\luw}{\begin{tabular}{l}#1\end{tabular}}
		\pbox[t][][b]{\textwidth}{
		\resizebox{1\luw}{!}{%
		\begin{tcolorbox}[width=\luw+5pt, boxsep=1pt, left=0pt, right=0pt, top=0pt, bottom=0pt,nobeforeafter]
		\begin{tabular}{l}
		#1
		\end{tabular}
		\end{tcolorbox}
		}%
		}
}

\newcommand{\ftitle}[1]{%
\begin{tcolorbox}[colback=blue!10!white,colframe=blue!75!black,width=2\figwidth, boxsep=3pt, left=0pt, right=0pt, top=0pt, bottom=0pt, 
after={},nobeforeafter, sharp corners=southwest, sharp corners=southeast, box align = top, center upper, equal height group=C]%
\bf #1
\end{tcolorbox}%
}

\centering
\setlength{\figwidth}{0.24\linewidth}
\setlength{\tabcolsep}{0pt}
\begin{tabular}{@{} ccccc}
& \multicolumn{4}{c}{%
\ftitle{Graph cut - based}\ftitle{LP - based}}\\
\rot{\rlap{\ \ \ \ \pbox{\textwidth}{Potts Model \newline $\lambda \min(1,|x_u-x_v|)$} }}\: &
\begin{overpic}[tics=10,width=\figwidth]{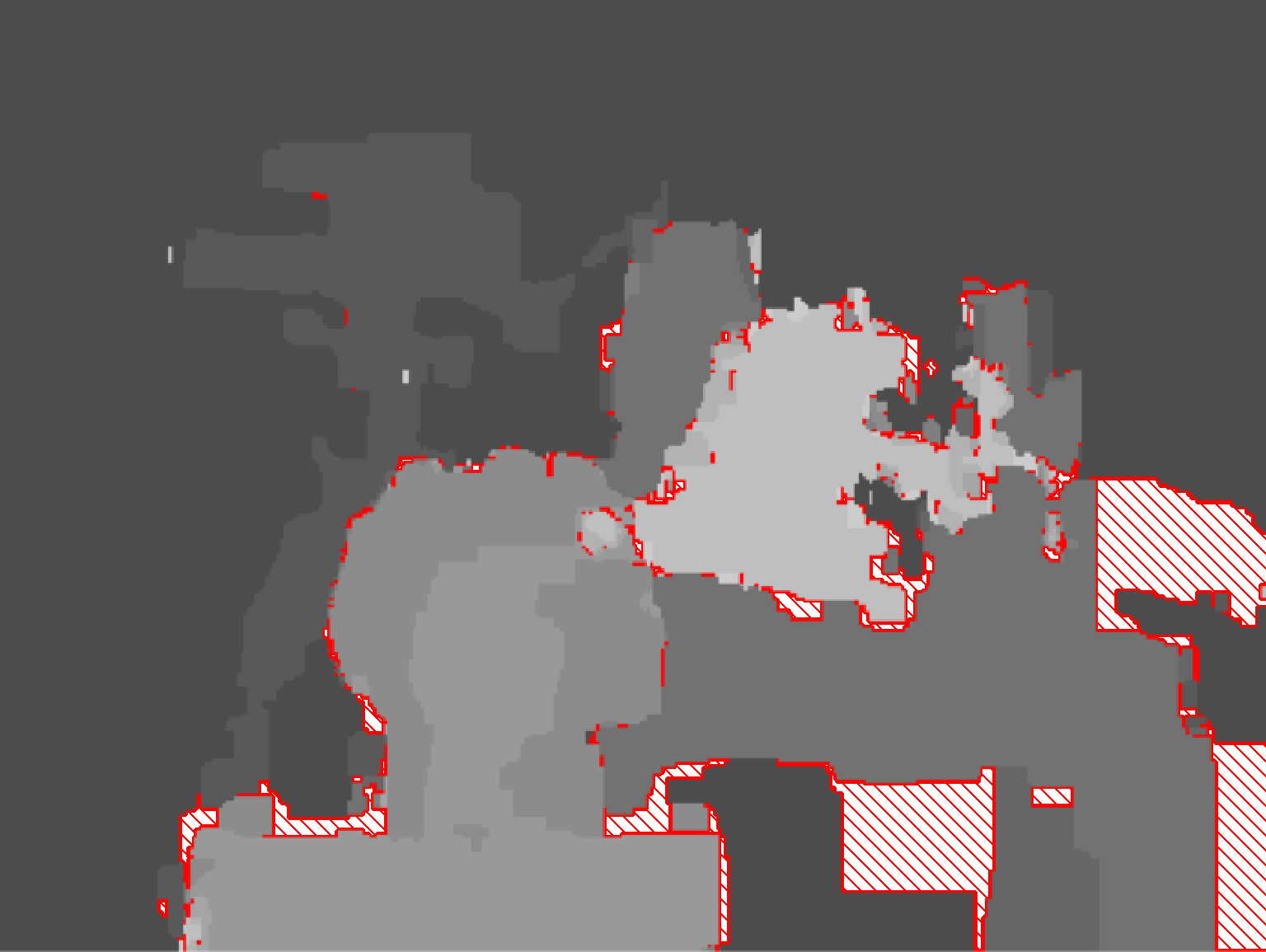}
\put (0,75){\flabelw{Courtesy of~\citet{Kovtun03}\\ 93.6\% (instance not available)}}%
\end{overpic}&
\begin{overpic}[tics=10,width=\figwidth]{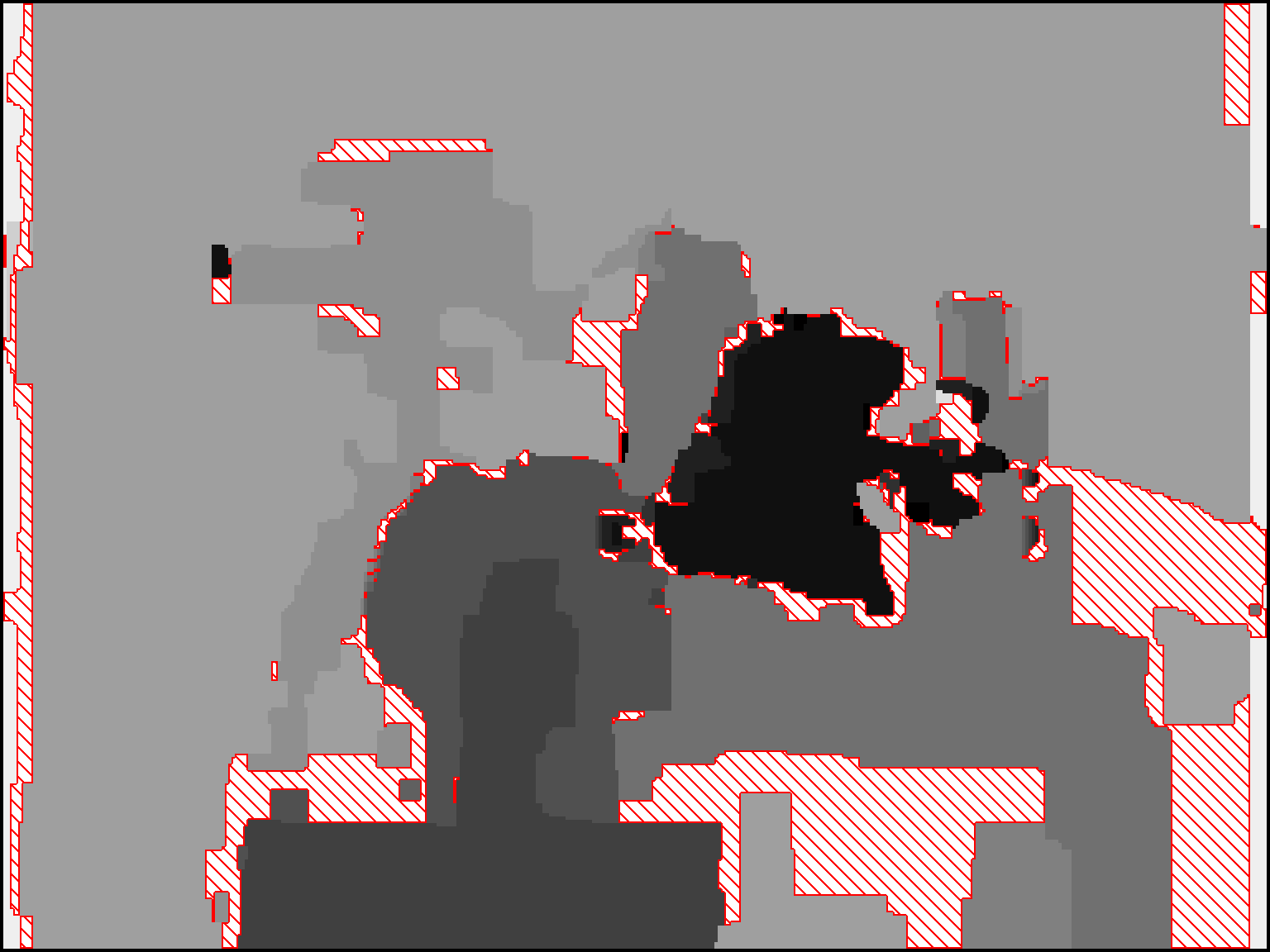}
\put (0,75){\flabelw{Instance used by~\citet{Alahari-08} \\ Kovtun's method: 1s, 87.6\%}}
\end{overpic}
&
\begin{overpic}[tics=10,width=\figwidth]{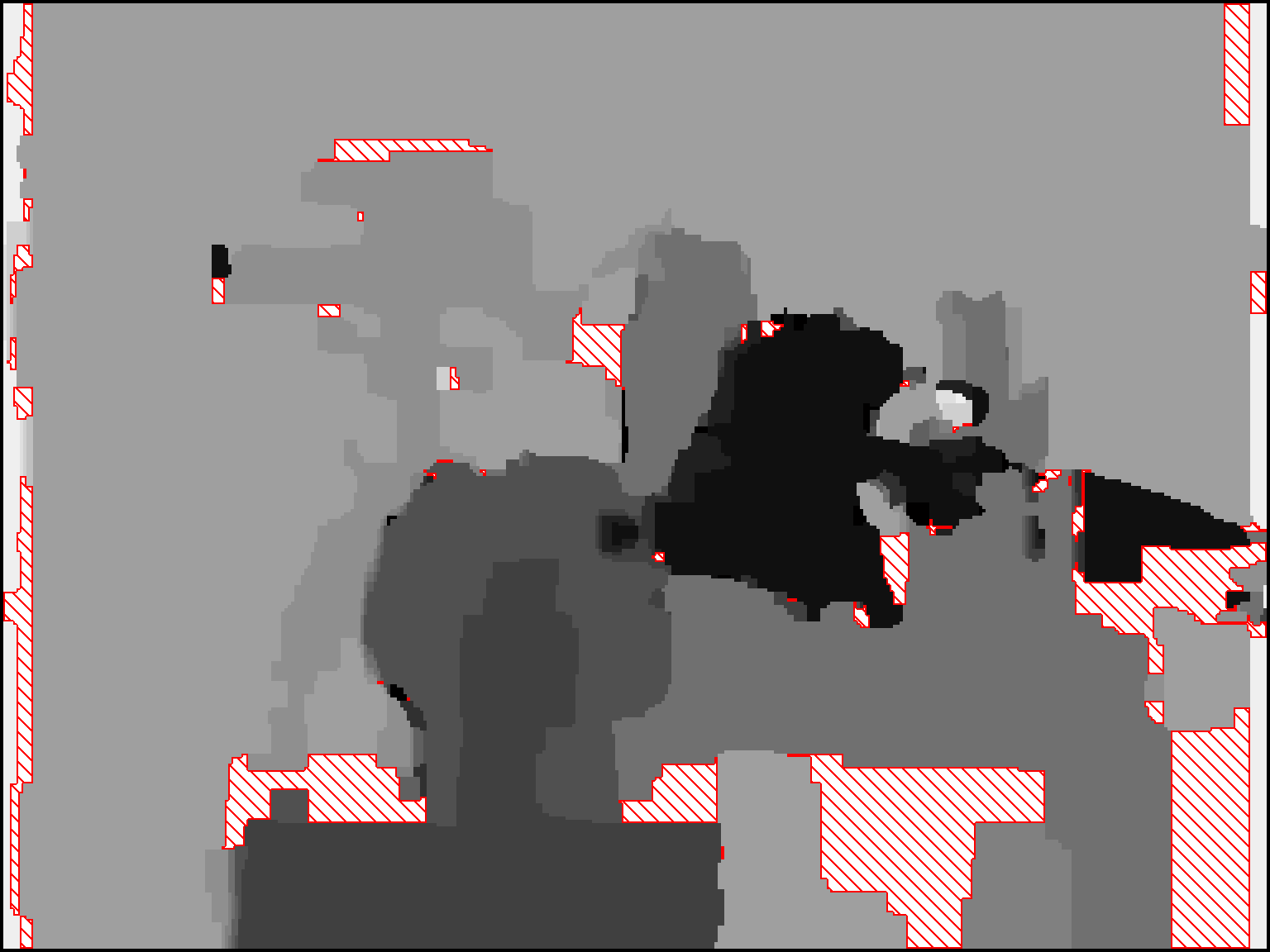}
\put (0,75){\flabelw{\citet{shekhovtsov-14} \\ LP-windowing: 1.5h, 94\%}}%
\end{overpic}
&
\begin{overpic}[tics=10,width=\figwidth]{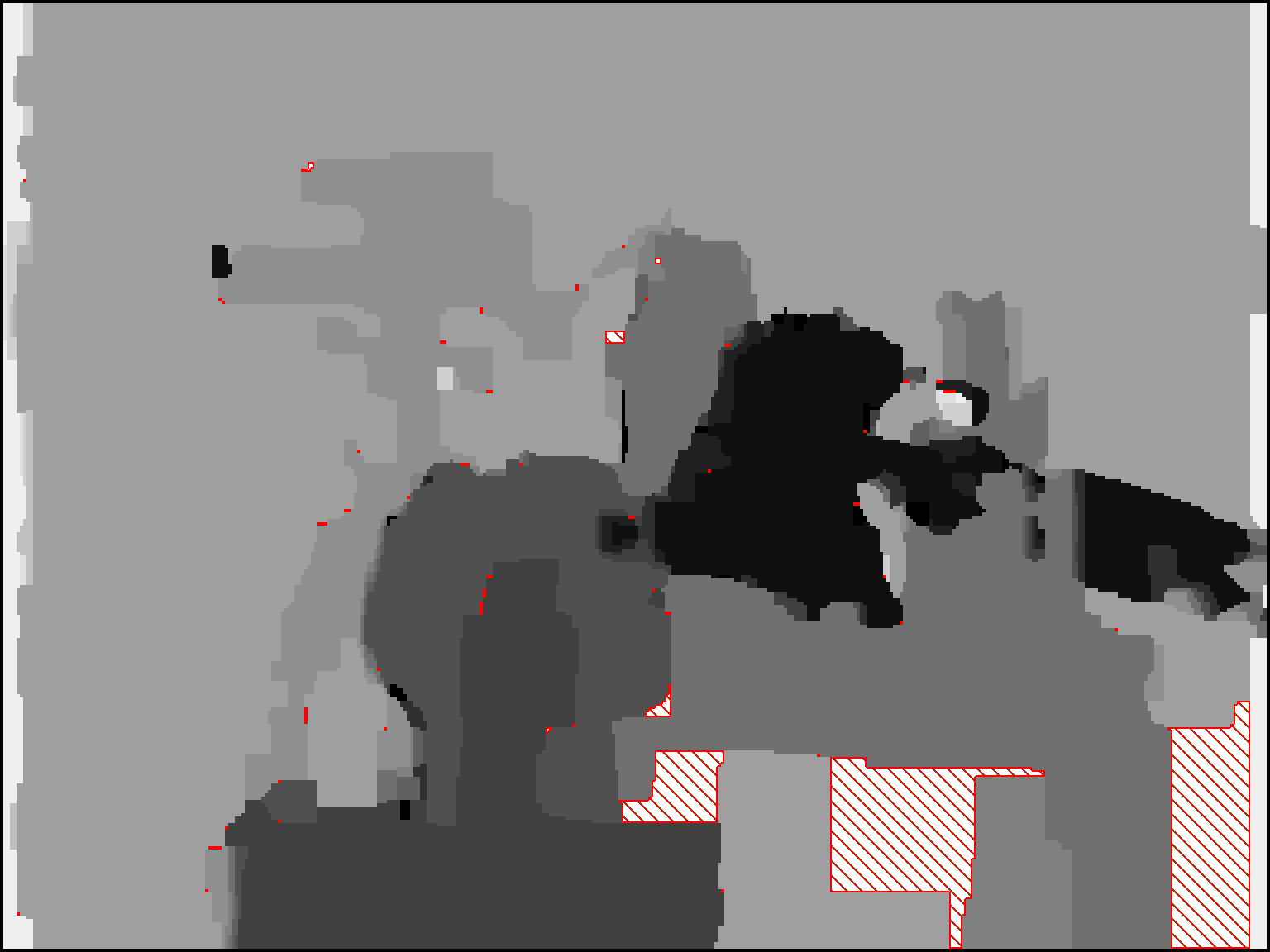}
\put (0,75){\flabelw{Ours: 22s, 96.7\%}}%
\end{overpic}
\\
\rot{\rlap{\ \ \ \ \pbox{\textwidth}{Truncated Model \newline $w_{uv} \min({\bf 2},|x_u-x_v|)$} }}\:
&
\begin{overpic}[tics=10,width=\figwidth]{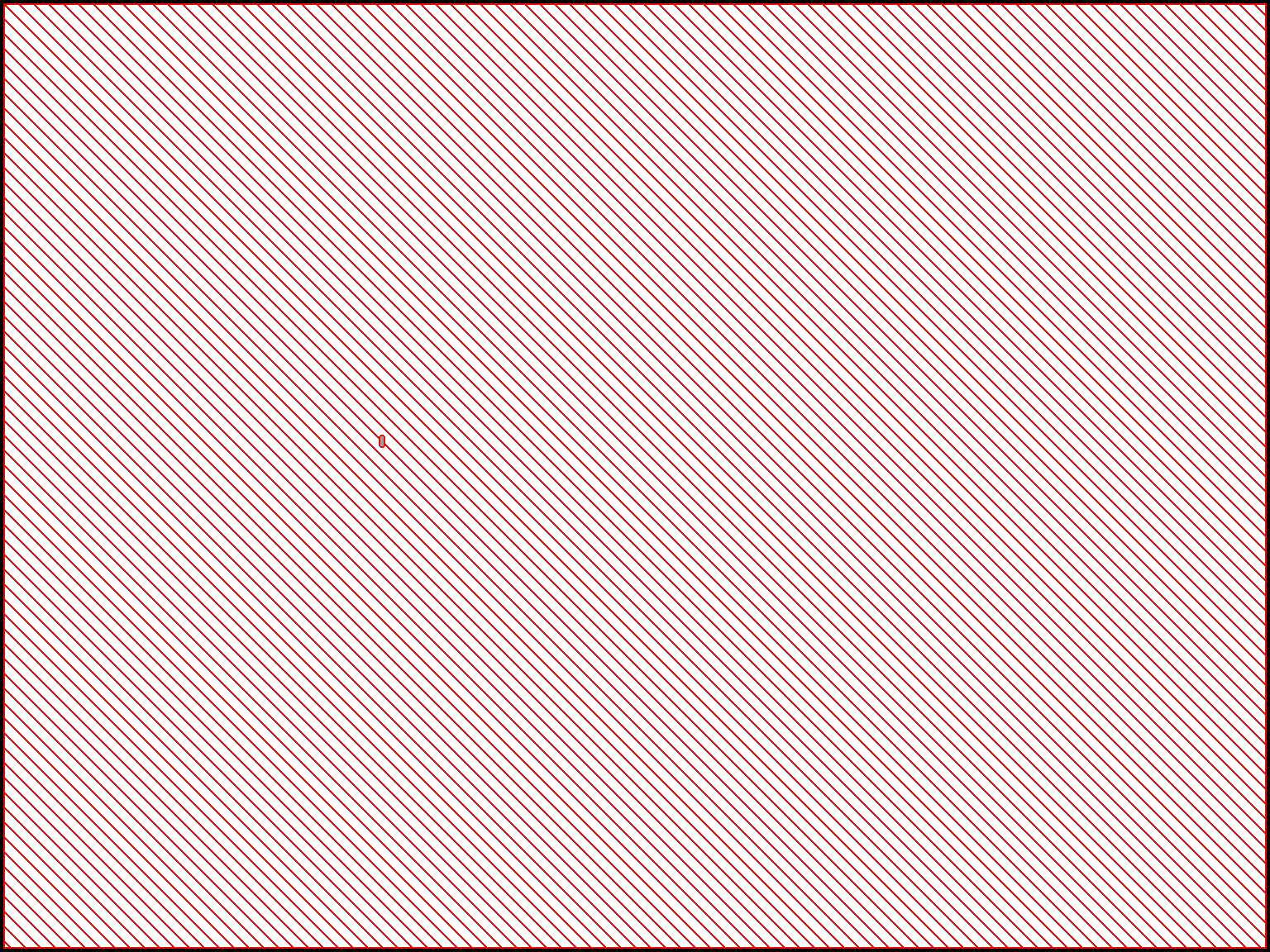}
\put (0,75){\flabelw{Kovtun's method: 1s, 0.2\%}}%
\end{overpic}
&
\begin{overpic}[tics=10,width=\figwidth]{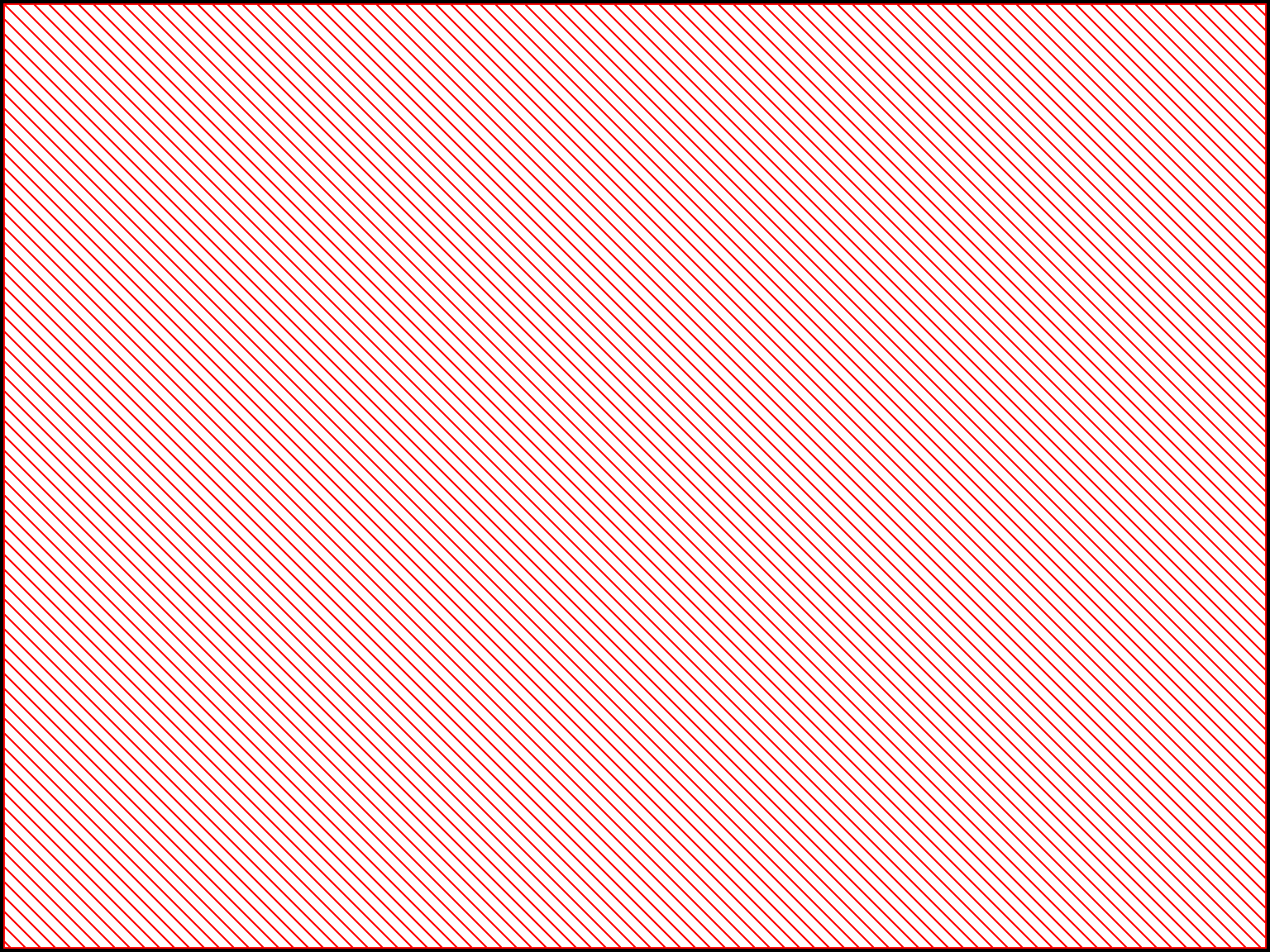}
\put (0,75){\flabelw{\citet{PartialOptimalityInMultiLabelMRFsKohli} (MQPBO) \\ 41s, 0.2\%}}%
\end{overpic}
&
\begin{overpic}[tics=10,width=\figwidth]{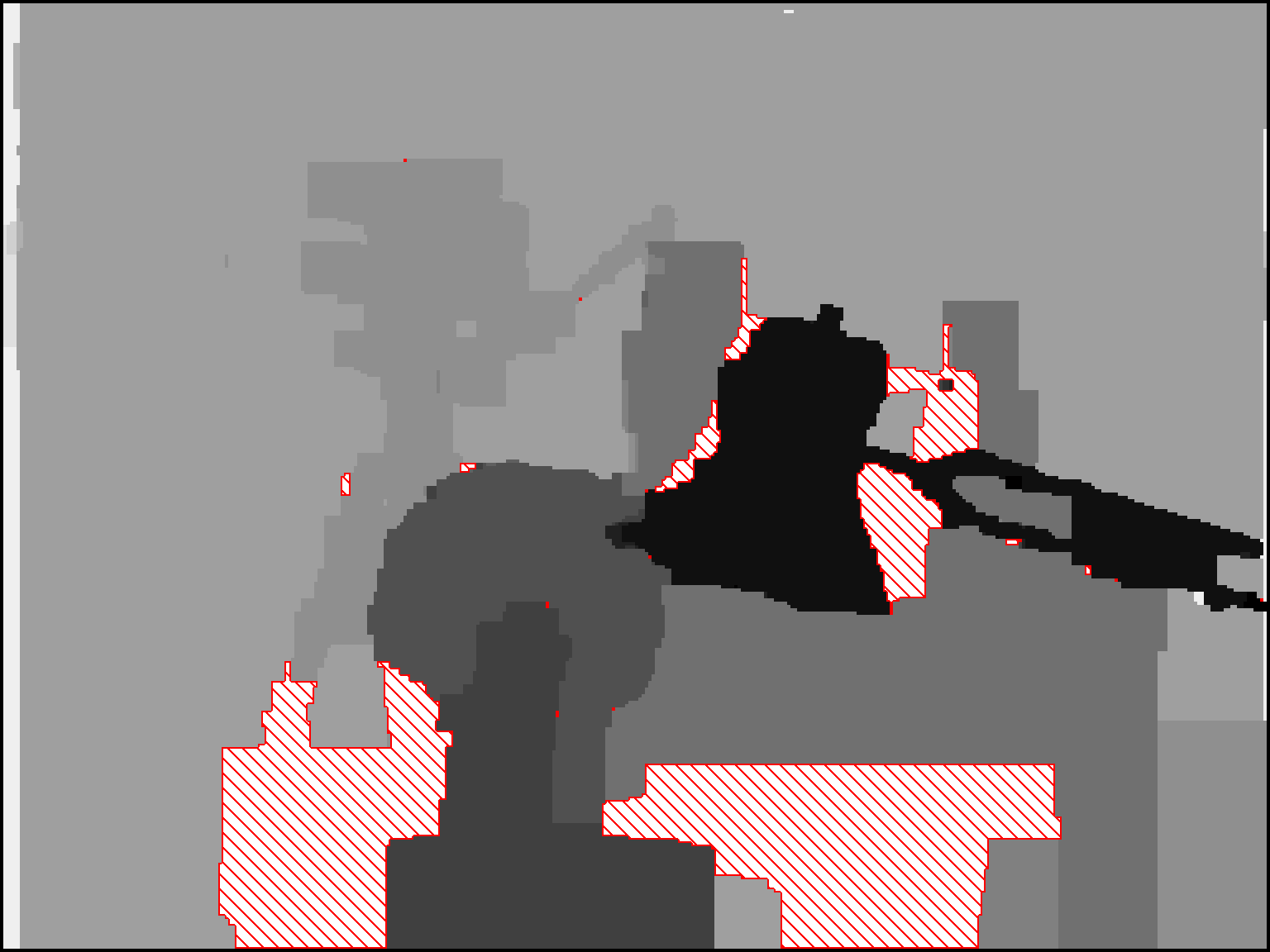}
\put (0,75){\flabelw{\citet{Swoboda-PAMI-16} \\ 27min, 89.8\%}}%
\end{overpic}
&
\begin{overpic}[tics=10,width=\figwidth]{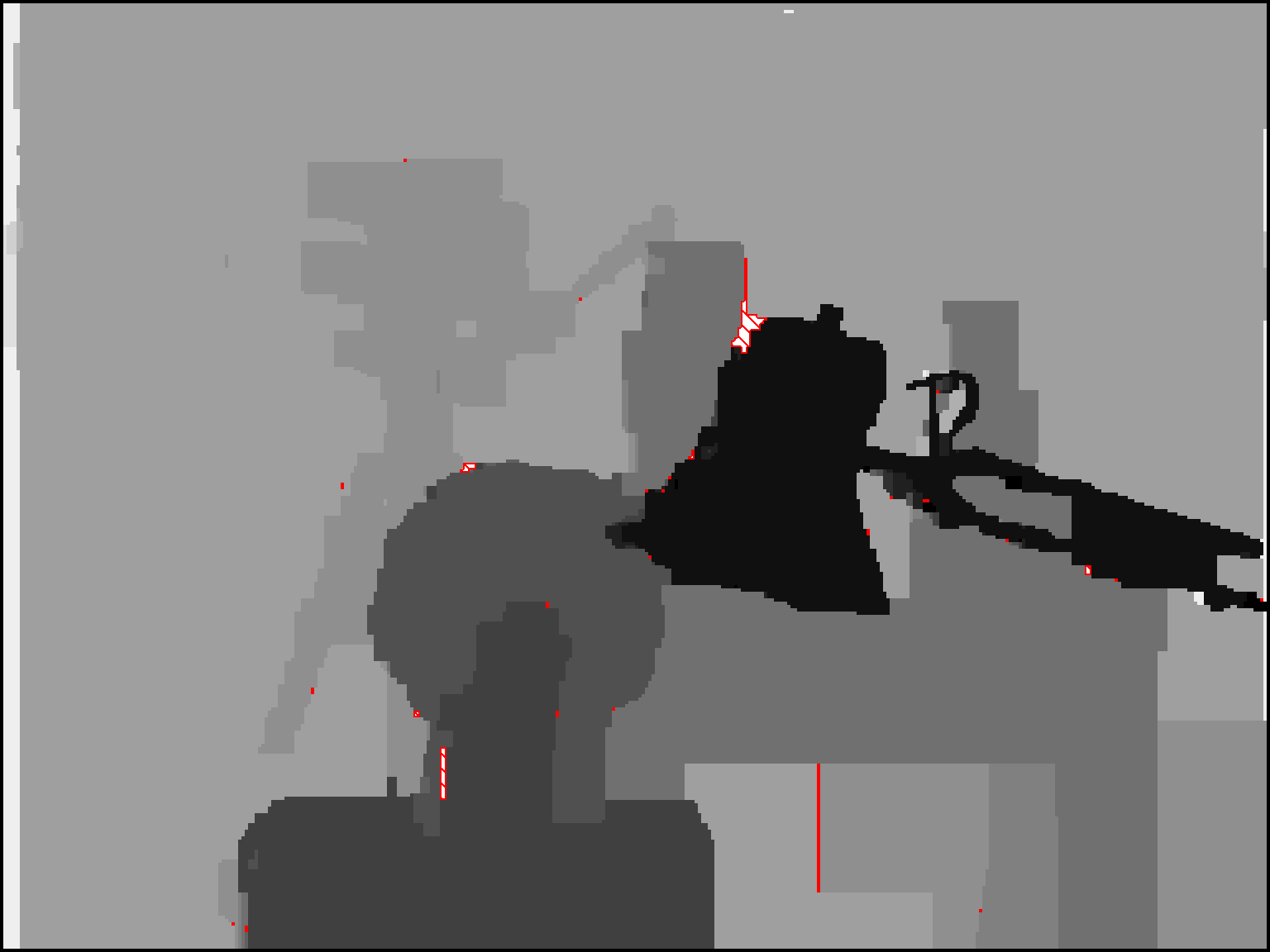}
\put (0,75){\flabelw{Ours: 16s, 99.94\%}}%
\end{overpic}
\end{tabular}

%% file: tex/preliminaries.tex
\section{Work Overview}\label{sec:preliminaries}
This section serves as an overview of our method, where we give the most general definitions, formulate the maximum persistency problem and briefly describe a generic method to solve it. 
				This description, equipped with references to subsequent sections, should serve as a road map for the rest of the paper.

\myparagraph{Notation} In the MAP-Inference Problem~\eqref{equ:energyMinimization} we assume $(\SV,\SE)$ to be a directed graph with the set of {\em nodes}~$\SV$  and the set of {\em edges}~$\SE\subset \V\times \V$. Let $uv$ denote an ordered pair $(u,v)$ and $\N(u) = \{v \mid uv \in \SE \vee vu \in \SE \}$ stands for the set of {\em neighbors} of $u$. 
Each node $v\in\V$ is associated with a variable~$x_v$ taking its values in a finite {\em set of labels} $\X_v$. {\em Cost functions} or {\em potentials} $f_v\colon\X_v\to\BR$, $ f_{uv}\colon\X_u\times\X_v\to\BR$ are associated with nodes and edges respectively. Let $f_\emptyset\in\Real$ be {\em a constant term}, which we introduce for the sake of notation. Finally, $\X$ stands for the Cartesian product $\prod_{v\in\SV}\X_{v}$ and its elements $x\in \X$ are called {\em labelings}.

We represent all potentials of energy~\eqref{equ:energyMinimization} by a single {\em cost vector}  
$ f\in\BR^\SI$, where the set $\SI$ enumerates all components of all terms: 
$\SI = \{\emptyset \}\cup\{(u,i) \mid u\in\V,\ i\in\X_u\}\cup\{(uv,ij) \mid uv\in\SE,\ i\in\X_u,\ j\in \X_v\}$.
\par
\begin{figure}[!tp]
\centering
\begin{tabular}{c}
\includegraphics[width=0.6\linewidth]{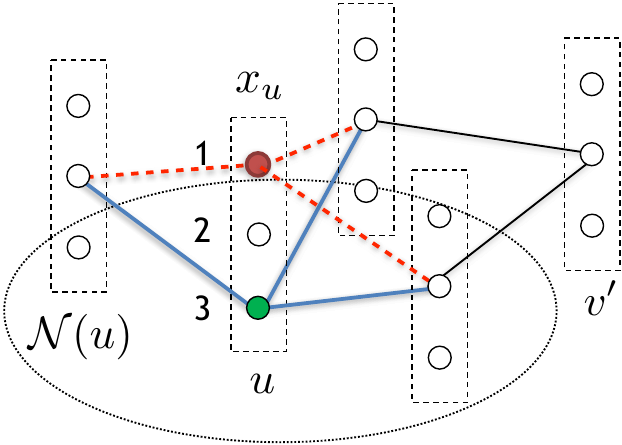}
\end{tabular}
\caption{
\label{fig:dee}
Dead end elimination / dominance. Variables are shown as boxes and their possible labels as circles. Label $x_u=1$ is substituted with label $x_u=3$. If for any configuration of neighbors $x_{\N(u)}$
the energy does not increase (only the terms inside $\{u\}\cup\N(u)$ contribute to the difference), label $x_u=1$ can be eliminated without loss of optimality.
}
\end{figure}
\begin{figure}[!tp]
\centering
\begin{tabular}{c}
\includegraphics[width=0.45\linewidth]{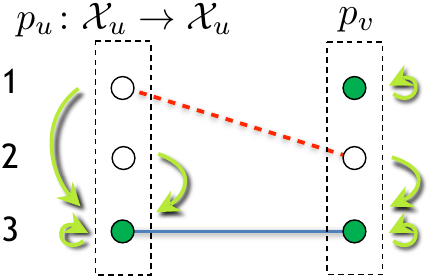}
\end{tabular}
\caption{
\label{fig:substitution}
Simultaneous substitution of labels in two variables. $\X_u=\X_v = \{1,2,3\}$, labels at arrow tails are substituted with labels at arrow heads. So the joint configuration $(1,2)$ (dashed) is substituted with the configuration $(3,3)$ (solid).
}
\end{figure}
\myparagraph{Improving Substitutions}
We formulate our persistency method in the framework of (strictly) improving substitutions, called {\em improving mappings} in our previous works~\cite{shekhovtsov-14,shekhovtsov-15-HO}. It was shown in~\cite{shekhovtsov-14} that most existing persistency techniques can be expressed as improving substitutions. A mapping $p\colon \X\to \X$ is called {\em a substitution}, if it is idempotent, \ie,~$p(x)=p(p(x))$. 
\begin{definition}\label{def:improvingMapping} A substitution $p\colon \X\to \X$  is called {\em strictly improving} for the cost vector $f$ if 
\begin{align}\label{def-eq:improving}
(\forall x \mid p(x)\neq x) \ \ E_f(p(x)) < E_f(x).
\end{align}
\end{definition}
When a strictly improving substitution is applied to any labeling $x$, it is guaranteed that $p(x)$ has equal or better energy. In particular, strictly improving substitutions generalize the strong autarky property~\cite{PseudoBooleanOptimizationBorosHammer}. When applied to the whole search space $\X$ we obtain its image $p(\X)$ -- a potentially smaller search space containing all optimal labelings. 

In what follows we will restrict ourselves to {\em node-wise} substitutions, \ie,~those defined locally for each node: $p(x)_u = p_u(x_u)$, where $p_u \colon \X_u \to \X_u$. Indeed, already this class of substitutions covers most existing persistency methods.

\begin{example}
Let us consider the dead-end elimination (DEE)~\cite{DeadEndEliminationDesmet,Goldstein-94-dee}. It is a test whether a given label in a single node, \eg, $x_u=1$ in \cref{fig:dee} can be substituted with another one, \eg, $x_u=3$ in \cref{fig:dee}. The change of the energy under this substitution depends only on the configuration of neighbors $x_{\N(u)}$, and the value of the change is additive in neighbors, so that it can be verified for all $x_{\N(u)}$ whether the substitution always improves the energy. If it is so, the label $x_u=1$ can be eliminated and the test is repeated for a different label in the reduced problem.
\end{example}

A general substitution we consider is applied to labels in all nodes simultaneously, as illustrated in \cref{fig:substitution} for two variables. We obtain the following principle for identifying persistencies.
\begin{proposition}\label{prop:ImprovingMappingProperty}
 If $p$ is a strictly improving substitution, then {\em any} optimal solution $x^*$ of~\eqref{equ:energyMinimization} must satisfy $(\forall v\in\V)\ p_v(x^*_v)  = x_v^*$.
\end{proposition}
Indeed, otherwise $E_f(p(x^*)) < E_f(x^*)$, which is a contradiction. If $p_v(i) \neq i$, then idempotency implies that label $(v,i)$ is non-optimal persistent and can be excluded from consideration. 


\par
\myparagraph{Verification Problem}
Verifying whether a given substitution is strictly improving is an NP hard decision problem~\cite{shekhovtsov-14}. 
In order to obtain a polynomial sufficient condition we will first rewrite~\eqref{def-eq:improving} as an energy minimization problem and then relax it. 
To this end we reformulate Definition~\ref{def:improvingMapping} in an optimization form:
\begin{proposition}\label{P:improving-eq-reformulation} Substitution $p$ is strictly improving {\em iff}
\begin{subequations}\label{improving-eq-reformulation}
\begin{align}\label{improving-eq-reformulation-a}
&\min_{x\in\X} \big( E_f(x) - E_f(p(x))\big) \geq 0,\\
&p(x) = x \ \mbox{for all minimizers}. \label{improving-eq-reformulation-b}
\end{align}
\end{subequations}
\end{proposition}
\begin{proof}
Indeed, condition~\eqref{improving-eq-reformulation-a} is equivalent to $(\forall x)\ E_f(x) \geq E_f(p(x))$.
Sufficiency: if $x \neq p(x)$, then $x$ is not a minimizer and $E_f(x) > E_f(p(x))$.
Necessity: for $x = p(x)$ we have that $E_f(x) = E_f(p(x))$, therefore from \cref{def:improvingMapping} it follows that condition~\eqref{improving-eq-reformulation-a} holds and any $x = p(x)$ is a minimizer, moreover, for any minimizer $x$ it must be $E_f(x) - E_f(p(x)) = 0$ and from \cref{def:improvingMapping} it follows $x = p(x)$.
\par
\end{proof}

In \cref{sec:lift-relax} we will show that the difference of the energies in~\eqref{improving-eq-reformulation-a} can be represented as a pairwise energy with an appropriately constructed cost vector $g$ so that there holds
\begin{align}\label{equ:single-improving-energy}
E_f(x) - E_f(p(x)) = E_g(x).
\end{align}
Therefore, according to Proposition~\ref{P:improving-eq-reformulation} the verification of the strictly improving property reduces to minimizing the energy~\eqref{equ:single-improving-energy} and checking that~\eqref{improving-eq-reformulation-b} is fulfilled.
To make the verification problem {\em tractable}, we relax it as
\begin{subequations}\label{improving-eq-reformulation-relax}
\begin{align}\label{improving-eq-reformulation-relax-a}
& \min_{\mu\in\Lambda} E_g(\mu) \geq 0, \\
& p(\mu) = \mu \ \mbox{for all minimizers},
\end{align}
\end{subequations}
where $\Lambda$ is a tractable polytope such that its integer vertices correspond to labelings (the standard LP relaxation that we use will be defined in \cref{sec:ApproximativeDualAlgorithm}), $\mu$ is a relaxed labeling and $E_g(\mu)$ and $p(\mu)$ are appropriately defined extensions of discrete functions $E_g(x)$ and $p(x)$, defined in \cref{sec:lift-relax}.
%
%
By construction, the objective value~\eqref{improving-eq-reformulation-relax} matches exactly to that of~\eqref{def-eq:improving} for all integer labelings $\mu\in\Lambda \cap \{0,1\}^\I$, which is sufficient for~\eqref{def-eq:improving} to hold. The sufficient condition~\eqref{improving-eq-reformulation-relax} (made precise in \cref{def:linearStrictLPImprovingMapping}) means that $p$ strictly improves not only integer labelings but also all relaxed labelings and therefore such substitutions will be called {\em strictly relaxed-improving} for a cost vector $f$. 
Assuming $\Lambda$ is fixed by the context, let $\BS_f$ denotes the set of all substitutions $p$ satisfying~\eqref{improving-eq-reformulation-relax}.
%

\myparagraph{Maximum Persistency and Subset-to-One Substitutions}
The {\em maximum persistency} approach~\cite{shekhovtsov-14} consists of finding a relaxed-improving substitution $p\in\BS_f$ that eliminates {\em the maximal} number of labels:
\begin{equation}\label{equ:maxPersistency}
\max\limits_{p\in \SP}\sum_{v\in\V}\big|\{i \in \X_v \mid p(i) \neq i \}\big|,\ \mbox{s.t. $p\in\BS_f$}\,,
\end{equation}
where $\SP$ is a class of substitutions. 

While maximizing over all substitutions is not tractable, maximizing over the following restricted class is~\cite{shekhovtsov-14}. Assume we are given a {\em test} labeling $y$, which in our case will be an approximate solution of the MAP inference~\eqref{equ:energyMinimization}. 
We then consider substituting in each node $v$ a subset of labels with $y_v$.


%
\begin{definition}[\cite{shekhovtsov-14}] A substitution $p$ is in the class of {\em subset-to-one} substitutions $\SP^{2,y}$, where $y\in\X$, if there exist subsets $\Y_v \subset \X_v \backslash \{ y_v \}$ for all $v$ such that
\begin{equation}\label{equ:pixelWiseMapping}
 p_v(i)=\left\{
 \begin{array}{ll}
  y_v, & \mbox{if}\ i\in \Y_v; \\
  i, & \mbox{if}\ i\notin \Y_v.
 \end{array}
 \right.
\end{equation}
\end{definition}
%
See~\cref{fig:substitution,fig:alg-steps} for examples. Note, this class is rather large: there are $2^{|\X_v|-1}$ possible choices for $p_v$ and most of the existing methods for partial optimality still can be represented using it~\cite{shekhovtsov-14} (in particular, methods~\cite{Kovtun-10,Gridchyn-13} can be represented using a constant test labeling, $y_v = \alpha$ for all $v$).


%
\par
The restriction to the class $\SP^{2,y}$ allows to represent the search of the substitution that eliminates the maximum number of labels as the one with the largest (by inclusion) sets $\Y_v$ of substituted labels. This allows to propose a relatively simple algorithm.
%
\myparagraph{Cutting Plane Algorithm}
The algorithm is a cutting plane method in a general sense: we maintain a substitution $p^t$ which is in all iterations better or equal than the solution to~\eqref{equ:maxPersistency} and achieve feasibility by iteratively constraining it. 
\begin{itemize}
\item[\em Initialization:] 
Define the substitution $p^0$ by the sets $\Y^0_v = \X_v \backslash \{y_v\}$. It substitutes everything with $y$ and clearly maximizes the objective~\eqref{equ:maxPersistency}.
\item[\em Verification:] Check whether current $p^t$ is strictly relaxed-improving for $f$ by solving relaxed problem~\eqref{improving-eq-reformulation-relax}.
      If yes, return~$p^t$. If not, the optimal relaxed solution $\mu^*$ corresponds to the most violated constraint.
\item[\em Cutting plane:] Assign $p^{t+1}$ to the substitution defined by the largest sets $\Y_v^{t+1}$ such that $\Y_v^{t+1} \subset \Y_v^t$ and the constraints ${E_{g}(\mu^*)\geq 0}$, ${p(\mu^*)=\mu^*}$, are satisfied. Repeat the verification step. 
\end{itemize}

The steps of this meta-algorithm are illustrated in \cref{fig:alg-steps}. It is clear that when the algorithm stops the substitution $p^t$ is strictly improving, although it could be the identity map that does not eliminate any labels. The exact specification of the cutting plane step will be derived in~\cref{sec:GenerilAlgorithm} and it will be shown that this algorithm solves the maximum persistency problem~\eqref{equ:maxPersistency} over $\SP^{2,y}$ optimally.

\myparagraph{Work Outline} 
In \cref{sec:lift-relax} we give a precise formulation of the relaxed condition~\eqref{improving-eq-reformulation-relax} and its components. In~\cref{sec:GenerilAlgorithm} we specify details of the algorithm and prove its optimality. These results hold for a general relaxation $\Lambda \supseteq \M$ but require to solve linear programs~\eqref{improving-eq-reformulation-relax-a} precisely.

The rest of the paper is devoted to an approximate solution of the problem~\eqref{equ:maxPersistency}, \ie\ finding a relaxed improving mapping, which is {\em almost} maximum. We consider specifically the standard LP relaxation and reformulate the algorithm to use a dual solver for the problem~\eqref{improving-eq-reformulation-relax-a}, \cref{sec:ApproximativeDualAlgorithm}. We then gradually relax requirements on the optimality of the dual solver while keeping persistency guarantees,~\cref{sec:ApproximativeDualAlgorithm,sec:any-dual,sec:problemReduction}, and propose several theoretical and algorithmic tools to solve the series of verification problems incrementally and overall efficiently,~\cref{sec:speedups}.
Finally, we provide an exhaustive experimental evaluation in~\cref{sec:experiments}, which clearly demonstrates efficacy of the developed method.


%

\begin{figure}
\centering
\includegraphics[width=0.65\linewidth]{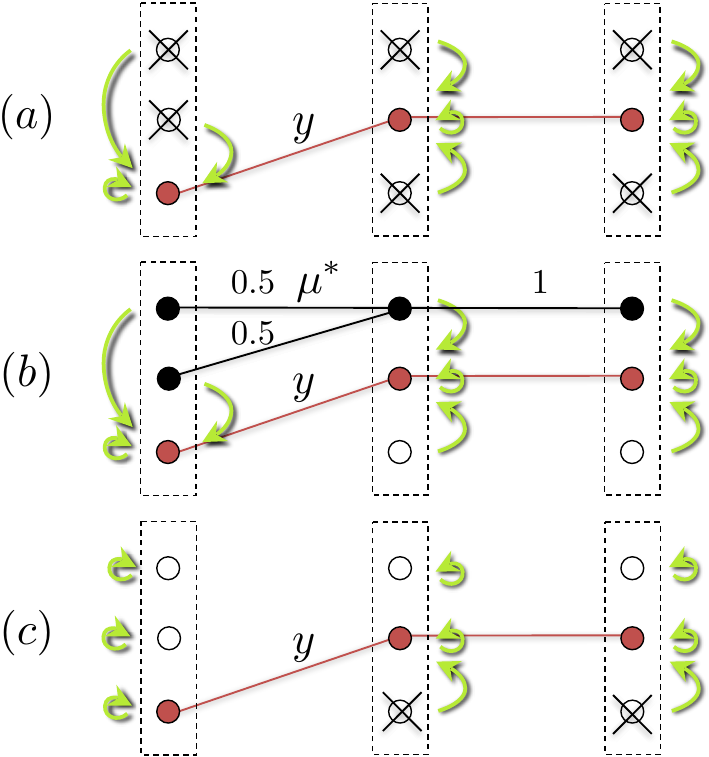}\\[5pt]
\caption{Steps of the discrete cutting-plane algorithm. (a) Starting from substitution that maps everything to the test labeling $y$ (red), crossed labels would be eliminated if $p$ passes the sufficient condition. (b) A relaxed solution $\mu^*$ violating the sufficient condition is found (black). (c) Substitution $p$ is pruned.
\label{fig:alg-steps}}
\end{figure}
\section{Relaxed-Improving Substitutions}\label{sec:lift-relax}
\myparagraph{Overcomplete Representation}
In this section we formally derive the strictly relaxed-improving sufficient condition~\cref{improving-eq-reformulation-relax}. To obtain the relaxation we use the standard lifting approach (\aka overcomplete representation~\cite{GraphicalModelsWainwrightJordan}), in which a labeling is represented using the $1$-hot encoding. This lifting allows to linearize the energy function, the substitution and consequently both the non-relaxed~\cref{improving-eq-reformulation} and relaxed~\cref{improving-eq-reformulation-relax} improving substitution criteria.
\par
The lifting is defined by the mapping $\delta\colon \X \to \Real^\I\colon$
\begin{subequations}
\begin{align}
&\delta(x)_{\emptyset} = 1,\\
&\delta(x)_{u}(i) = \leftbb x_u {=} i\rightbb,\\
&\delta(x)_{uv}(i,j) = \leftbb x_u {=} i\rightbb \leftbb x_v {=} j\rightbb,
\end{align}
\end{subequations}
where $\leftbb \cdot \rightbb$ is the Iverson bracket, \ie, $\leftbb A \rightbb$ equals $1$ if $A$ is true and $0$ otherwise. Using this lifting, we can linearize unary terms as $f_v(x_v) = \sum_{k} f_v(k) \leftbb x_v{=}k \rightbb = \sum_{k} f_v(k) \delta(x)_v(k)$ and similarly for the pairwise terms. This allows to linearize the energy function $E_f$ and write it as a scalar product $E_{ f}(x)=\lan  f,\delta(x)\ran$ in $\Real^\I$.
The energy minimization problem~\eqref{equ:energyMinimization} can then be written as
\begin{equation}\label{energy-M}
 \min_{x\in\X}E_{f}(x)=\min_{x\in\X}\lan f,\delta(x)\ran=\min_{\mu\ \in \SM}\lan f,\mu\ran\,,
\end{equation}
where $\SM = \conv \delta(\X)$ is the convex hull of all labelings in the lifted space, also known as marginal polytope~\cite{GraphicalModelsWainwrightJordan}. The last equality in~\eqref{energy-M} uses the fact that the minimum of a linear function on a finite set equals the minimum on its convex hull. Expression~\eqref{energy-M} is an equivalent reformulation of the energy minimization problem as a linear program, however over a generally intractable polytope $\SM$.
%
%
%
\myparagraph{Lifting of Substitutions}
Next we show how a substitution $p\colon \X \to \X$ can be represented as a linear map in the lifted space~$\Real^\I$. This will allow to express the term $E_f(p(x))$ as a linear function of $\delta(x)$ and hence also to represent the non-relaxed criterion~\eqref{improving-eq-reformulation}.
\begin{proposition}\label{Pro:P-transpose}
Given a substitution $p$, 
let $P\T \colon \BR^\SI \to \BR^\SI$ be defined by its action on a cost vector $f\in\Real^\I$ as follows: 
\begin{subequations}\label{P-transpose}
\begin{align}
(P\T f)_\emptyset & = f_\emptyset,\\
(P\T f)_u(i) & = f_u(p_u(i)),\label{P-transpose-b}\\
(P\T f)_{uv}(i,j) & = f_{uv}(p_u(i),p_v(j)) \label{P-transpose-c}
\end{align}
\end{subequations}
$\forall$ $u\in \V$, $uv\in\SE$, $ij\in \X_{uv}$. Then $P$ satisfies
\begin{align}\label{lenear-extension-def}
(\forall x\in\X)\ \ \delta(p(x)) = P\delta(x).
\end{align}
\end{proposition}
\begin{proof}
Let $x\in\X$. From~\eqref{P-transpose} it follows that $E_{P\T f}(x) = E_f(p(x))$, which can be expressed as a scalar product
$$
\<P\T f, \delta(x) \>  = \<f, P \delta(x) \> = \<f,\delta(p(x))\>.
$$
Since this equality holds for all $f\in\Real^\I$ it follows that $P \delta(x) = \delta(p(x))$.
\end{proof}
The expression~\eqref{lenear-extension-def} allows to write the energy of the substituted labeling $p(x)$ as 
\begin{align}
E_f(p(x)) = \<f, \delta(p(x))\> = \<f, P \delta(x)\> = E_{P\T f}(x).
\end{align}
For this reason, the mapping $P$ is called the {\em linear extension} of $p$ and will be denoted with the symbol $[p]$. The following example illustrates how $[p]$ looks in coordinates.
\begin{example}\label{example:linear-extension}
Consider the substitution $p$ depicted in \cref{fig:substitution} and defined by $p_u  \colon 1,2,3 \mapsto 3,3,3$; $p_v \colon 1,2,3 \mapsto 1,3,3$.
The relaxed labeling $\mu\in\Real^I$ has the structure $(\mu_\emptyset, \mu_u, \mu_v, \mu_{uv})$. The linear extension $[p]\colon \Real^\I\to \Real^\I$ can be written as a block-diagonal matrix
\begin{align}
&\left(\begin{smallmatrix}
1\\
& P_u \\
& & P_v \\
& & & P_{uv} 
\end{smallmatrix}\right)\hskip-0.5ex, \mbox{where} \ 
P_u = \left(\begin{smallmatrix}
0 & 0 & 0 \cr
0 & 0 & 0 \cr
1 & 1 & 1
\end{smallmatrix}\right)\hskip-0.5ex,\ 
P_v = \left(\begin{smallmatrix}
1 & 0 & 0 \cr
0 & 0 & 0 \cr
0 & 1 & 1
\end{smallmatrix}\right)
\end{align}
and $P_{uv}$ is defined by $P_{uv} \mu_{uv} = P_u \mu_{uv} P_v$, where $\mu_{uv}$ is shaped as a $3\times 3$ matrix. The action of the block $P_u$ expresses as
\begin{align}
P_u \left(\begin{smallmatrix}
\mu_u(1) \cr
\mu_u(2) \cr
\mu_u(3)
\end{smallmatrix}\right) = \left(\begin{smallmatrix}
0 \cr
0 \cr
1
\end{smallmatrix}\right)\hskip-0.5ex,
\end{align}
\ie all relaxed labels are mapped to the indicator of the label $x_u = 3$. And the adjoint operator $P\T$ acts as follows ({\em cf.}~\eqref{P-transpose-b}):
\begin{align}
P_u^{\top}  (\begin{smallmatrix} f_u(1) & f_u(2) & f_u(3) \end{smallmatrix})^{\top}
& =  (\begin{smallmatrix} f_u(3) & f_u(3) & f_u(3) \end{smallmatrix})^{\top},\\
P_v^{\top}  (\begin{smallmatrix} f_u(1) & f_u(2) & f_u(3) \end{smallmatrix})^{\top}
&=  (\begin{smallmatrix} f_u(1) & f_u(3) & f_u(3) \end{smallmatrix})^{\top}\,.
\end{align}
Similarly, due to~\eqref{P-transpose-c} we have, \eg,
$(P^{\top}_{uv}f_{uv})(1,2) = f_{uv}(3,3)$, $(P^{\top}_{uv}f_{uv})(1,1) = f_{uv}(3,1)$ and so on.
\qed
\end{example}

\myparagraph{Strictly improving substitutions}
Let $I$ denote the identity mapping $\Real^\I\to \Real^\I$. Using \Cref{P:improving-eq-reformulation} and the linear extension $[p]$, we obtain that substitution $p$ is strictly improving iff the value of
\begin{multline}\label{equ:linearStrictImprovingMapping}
\min_{x\in\X}\lan f,\delta(x)-\delta(p(x))\ran = 
\min_{x\in\X}\lan f,(I-[p])\delta(x)\ran\\
 = \min_{x\in\X}\<(I-[p])^{\top} f,\delta(x)\>
= \min_{\mu\in\SM}\<(I-[p])^{\top} f,\mu\>
\end{multline}
is zero and $[p]\mu=\mu$ for all minimizers. 
Note that problem~\eqref{equ:linearStrictImprovingMapping} is of the same form as the energy minimization~\eqref{energy-M} with the cost vector $g = (I-[p]\T)f$, as introduced in~\eqref{equ:single-improving-energy}. 
\par
The sufficient condition for persistency~\cref{improving-eq-reformulation-relax} is obtained by relaxing the intractable marginal polytope $\M$ in~\eqref{equ:linearStrictImprovingMapping} to a tractable outer approximation $\Lambda \supset \SM$. 
\begin{definition}[\cite{shekhovtsov-14}]\label{def:linearStrictLPImprovingMapping}
Substitution $p$ is {\em strictly $\Lambda$-improving for the cost vector $f\in\BR^{\SI}$} (shortly, {\em strictly relaxed-improving}, or $p \in \BS_f$) if 
\begin{subequations}\label{equ:linearStrictLPImprovingMapping}
 \begin{align}\label{equ:linearStrictLPImprovingMapping-a}
	\textstyle \min_{\mu\in \Lambda}\<(I-[p])^{\top} f,\mu\>=0,\\
	\label{equ:linearStrictLPImprovingMapping-b}
	[p]\mu^*=\mu^* \mbox{ for all minimizers}.
 \end{align}
\end{subequations}
\end{definition}
In~\cref{sec:ApproximativeDualAlgorithm}, $\Lambda$ will be defined as the polytope of the standard LP relaxation but until then the arguments are general and require only that $\Lambda \supset \SM$.
Since $\Lambda$ includes all integer labelings, it is a sufficient condition for improving substitution and hence persistency.
\begin{corollary}\label{prop:LambdaImprovingIsImproving}
 If substitution $p$ is strictly $\Lambda$-improving for $f$ and~${\Lambda \supset \SM}$, then $p$ is strictly improving for $f$.
\end{corollary}

%
 
%
The problem~\eqref{equ:linearStrictLPImprovingMapping-a} will be called the {\em verification LP} and the decision problem to test for $p\in \BS_{ f}$, \ie to verify conditions~\eqref{equ:linearStrictLPImprovingMapping}, will be called the {\em verification problem}.


%
%
%

%% file: tex/alg_primal.tex
\section{Generic Persistence Algorithm}\label{sec:GenerilAlgorithm}
%
\myparagraph{Structure of $P^{2,y}$}
%
In~\cite{shekhovtsov-14} it was shown that the maximum persistency problem~\eqref{equ:maxPersistency} over the class of substitutions $\SP^{2,y}$ can be formulated as a single linear program, where the substitution is represented using auxiliary (continuous) variables.
Here we take a different approach based on observing a lattice-like structure of improving substitutions.
\par
Throughout this section we will assume that the test labeling $y\in\X$ is fixed. 
Let us compare two substitutions $p$ and $q$ by the sets of the labels they eliminate. A substitution $p\in\SP^{2,y}$ eliminates all labels in $\Y_v$, or equivalently all labels not in $p_v(\X_v)$.
\begin{definition}\label{def:order}
A substitution ${p\in\SP^{2,y}}$ is {\em better equal} than a substitution ${q\in\SP^{2,y}}$, denoted by $p \geq q$, if $(\forall v\in\V)$ it holds ${p_v(\X_v) \subset q_v(\X_v)}$.
\end{definition}
\begin{proposition}\label{Pro:lattice}
Let the partially ordered set ${(\BS_{ f} \cap \SP^{2,y}, \geq)}$ of subset-to-one strictly relaxed-improving substitutions has the maximum and let it be denoted $r$. Then $r$ is the unique solution of~\eqref{equ:maxPersistency} with $\SP=\SP^{2,y}$. 
\end{proposition}
\begin{proof}
Since $r$ is the maximum, it holds $r \geq q$ for all $q$. From Definition~\ref{def:order} we have $r_v(\X_v) \subset q_v(\X_v)$ for all $v$ and thus $\sum_{v}|r_v(\X_v)| \leq \sum_v |q_v(\X_v)|$. Therefore $r$ is optimal to~\eqref{equ:maxPersistency}. Additionally, if $r \neq q$, then it holds $\sum_v|r_v(\X_v)| < \sum_v|q_v(\X_v)|$ and therefore $r$ is the unique solution to~\eqref{equ:maxPersistency}. 
\end{proof}
The existence of the maximum will formally follow from the correctness proof of the algorithm,~\cref{thm:maxMapping}. A stronger claim, which is not necessary for our analysis, but which may provide a better insight is that $(\BS_{ f} \cap \SP^{2,y}, \geq)$ is a lattice isomorphic to the lattice of sets with union and intersection operations. This is seen as follows. If both $p$ and $q$ are strictly-improving then so is their composition $r(x) = p(q(x))$, as can be verified by chaining inequalities~\eqref{def-eq:improving}. In $\SP^{2,y}$ the composition satisfies the property $r_u(\X_u) = p_u(\X_u) \cap q_u(\X_u)$, and can be identified with the {\em join} of $p$ and $q$ (the least $r$ such that $r \geq p$ and $r \geq q$). This can be shown to hold also for $\BS_{ f} \cap \SP^{2,y}$. It is this structure that allows to find the maximum in $\BS_{ f} \cap \SP^{2,y}$ by a relatively simple algorithm.
\myparagraph{Generic Algorithm}
Our generic primal algorithm, displayed in~\cref{alg:iterative-LP}, represents a substitution $p \in \SP^{2,y}$ by the sets $\Y_v$ of labels to be substituted with $y$, via~\eqref{equ:pixelWiseMapping}. 
\cref{alg1:init} initializes these sets to all labels but $y$. \cref{alg1:init-p} constructs the cost vector of the verification LP in condition~\eqref{equ:linearStrictLPImprovingMapping}.
\cref{alg1:test-p,alg1:check-condition} solve the verification LP and test whether sufficient conditions~\eqref{equ:linearStrictLPImprovingMapping} are satisfied via the following reformulation. 
\begin{restatable}{proposition}{SIcomponents}\label{SI-components}
For a given substitution $p$, let $\O^*\subseteq\Lambda$ denote the set of minimizers of the verification LP~\eqref{equ:linearStrictLPImprovingMapping} and 
\begin{equation}\label{equ:SI-components}
\O^*_v := \{i\in \X_v \mid (\exists \mu\in\O^*) \ \ \mu_v(i) > 0\},
\end{equation}
which is the support set of all optimal solutions in node $v$. Then $p\in\BS_f$ iff $(\forall v\in\SV\ \ \forall i\in\O^*_v)\ \ p_v(i) = i$.
\end{restatable}
\begin{corollary}\label{prop:linearStrictLPImprovingMapping}
For substitution $p\in\SP^{2,y}$ defined by~\eqref{equ:pixelWiseMapping} 
it holds ${p\in \BS_f}$
iff 
$
{(\forall v\in\SV) \ \ \O^*_v \cap \Y_v = \emptyset}.$
\end{corollary}
In the remainder of the paper we will relate the notation $\O^*_v$ to $\O^*$ as in~\eqref{equ:SI-components}. The set of optimal solutions $\O^*$ can, in general, be a $d$-dimensional face of $\Lambda$. We need to determine in~\eqref{equ:SI-components} whether an optimal solution $\mu\in\O^*$ exists such that the coordinate $\mu_v(i)$ is strictly positive. From the theory of linear programming~\cite{Vanderbei2001}, it is known that if one takes an optimal solution $\mu$ in the relative interior of $\O^*$ (\ie, if $\O^*$ is a 2D face the relative interior excludes vertices and edges of $\O^*$) then its support set $\{i \mid \mu_v(i) > 0 \}$ is the same for all such points and it matches $\O^*_v$~\eqref{equ:SI-components}. Therefore, it is practically feasible to find the support sets $\O^*_v$ by using a single solution found by an interior point/barrier method (which are known to converge to a ``central'' point of the optimal face) or methods based on smoothing~\cite{Savchynskyy11,AdaptiveDiminishingSmoothing}. Obtaining an {\em exact} solution by these methods may become computationally expensive as the size of the inference problem~\eqref{equ:energyMinimization} grows. Despite that, \Algorithm{alg:iterative-LP} is implementable and defines the baseline for its practically efficient variants solving~\eqref{equ:maxPersistency} approximately. These are developed further in the paper.
\par
%
Since \cref{alg1:check-condition} of \Algorithm{alg:iterative-LP} verifies precisely the condition of \cref{prop:linearStrictLPImprovingMapping}, the algorithm terminates as soon as $p\in\BS_f$ and hence $p$ is strictly improving. In the opposite case, \cref{alg1:update-U} {\em prunes} the sets $\Y_v$ by removing labels corresponding to the support set $\O_v^*$ of all optimal solutions of the verification LP, which have been identified now to violate the sufficient condition. These labels may be a part of some optimal solution to~\eqref{equ:energyMinimization} and will not be eliminated.
\par
To complete the analysis of~\cref{alg:iterative-LP} it remains to answer two questions: i) does it terminate and ii) is it optimal for the maximum persistency problem~\eqref{equ:maxPersistency}? 
\begin{proposition}
\label{A1-P1}
 Algorithm~\ref{alg:iterative-LP} runs in polynomial time and returns a substitution $p\in\BS_f\cap \SP^{2,y}$. 
\end{proposition}
\begin{proof}
As we discussed above, sets $\O^*_v$ in \cref{alg1:O_v} can be found in polynomial time. At every iteration, if the algorithm has not terminated yet, at least one of the sets $\Y_v$ strictly shrinks, as can be seen by comparing termination condition in~\cref{alg1:check-condition} with pruning in~\cref{alg1:update-U}. Therefore the algorithm terminates in at most $\sum_v (|\X_v|-1)$ iterations. On termination, $p\in\BS_f$ by \cref{prop:linearStrictLPImprovingMapping}.
\end{proof}
\SetKwFor{forever}{while true}{}{end while}
\let\oldnl\nl
\newcommand{\nonl}{\renewcommand{\nl}{\let\nl\oldnl}}
\begin{algorithm}[t]
\KwIn{Cost vector $ f \in\Real^\SI$, test labeling $y\in\X$\;}
\KwOut{Maximum strictly improving substitution $p$\;}
$(\forall v\in\V)\ \Y_v := \X_v \backslash\{y_v\}$\label{alg1:init}\;
\forever{}{
Construct the verification problem potentials $ g:=(I-[p])^\top f$ with $p$ defined by~\eqref{equ:pixelWiseMapping}\label{alg1:init-p}\;
	$\O^* = \argmin_{\mu\in\Lambda} \< g, \mu\>$\label{alg1:test-p}\;
	$(\forall v\in\V)\ \O_v^* = \{ i\in\X_v \mid (\exists \mu\in\O^*)\ \ \mu_{v}(i)>0 \}$\label{alg1:O_v}\;
	\lIf{$(\forall v\in\SV) \ \ \O^*_v \cap \Y_v = \emptyset$\label{alg1:check-condition}}{\Return{$p$}}
	\For{$v\in\V$}{
		Pruning of substitutions: $\Y_v:= \Y_v \backslash \O_v^*$\label{alg1:update-U}\;
	}
}
\caption{Iterative Pruning LP-Primal\label{alg:iterative-LP}}
\end{algorithm}

\begin{restatable}{theorem}{TmaxMapping}\label{thm:maxMapping}
Substitution $p$ returned by Algorithm~\ref{alg:iterative-LP} is the maximum of $\BS_f\cap \SP^{2,y}$ and thus it solves~\eqref{equ:maxPersistency}. 
\end{restatable}

It is noteworthy that \Algorithm{alg:iterative-LP} can be used to solve problem~\eqref{equ:maxPersistency} with {\em any} polytope $\Lambda$ satisfying $\SM\subset\Lambda$, \ie, with any LP relaxation of~\eqref{equ:energyMinimization} that can be expressed in the lifted space~$\Real^\I$.
Moreover, in order to use the algorithm with higher order models one needs merely to (straightforwardly) generalize the linear extension~\cref{P-transpose} as done in~\cite{shekhovtsov-15-HO}.

The test labeling $y$ can itself be chosen using the approximate solution of the LP-relaxation, \eg, via the zeroth iteration of the algorithm with $g = f$ and picking $y_v$ from $\O^*_v$. This choice is motivated by the fact that a strict relaxed-improving substitution cannot eliminate the labels from the support set of optimal solutions of the LP relaxation~\cite{shekhovtsov-14}, and thus these labels may not be substituted with anything else. 
\myparagraph{Comparison to Previous Work\label{sec:prev-work}}
Substitutions in $\SP^{2,y}$ are related to the expansion move algorithm~\cite{Boykov:2001:FAE:505471.505473} in the following sense. While~\cite{Boykov:2001:FAE:505471.505473} seeks to improve a single current labeling $x$ by calculating an optimized crossover (fusion) with a candidate labeling $y$, we seek which labels can be moved with a guaranteed improvement to $y$ for all possible labelings $x$.
\par
\cref{alg:iterative-LP} is similar in structure to~\cite{SwobodaPersistencyCVPR2014}. The later finds an improving substitution in a small class $\SP^{1,y}$ by incrementally shrinking the set of potentially persistent variables. 
More specifically, given a {\em} test labeling $y\in\X$, the {\em all-to-one} class of substitutions $\SP^{1,y}$ 
contains substitutions~$p$, which in every node $v$ either replace {\em all} labels with $y_v$ or leaves all labels unchanged. 
There are only two possible choices for $p_v$: either $i \mapsto y_v$ for all $i\in\X_v$ or the identity $i \mapsto i$. 
Methods~\cite{Kovtun03,PartialOptimalityByPruningPotts,SwobodaPersistencyCVPR2014} can be explained as finding an improving mapping in this class~\cite{shekhovtsov-14}. 
We generalize the method~\cite{SwobodaPersistencyCVPR2014} to substitutions in $\SP^{2,y}$. The original sufficient condition for persistency in~\cite{SwobodaPersistencyCVPR2014} does not extend to such substitutions. Even for substitutions in $\SP^{1,y}$ it is generally weaker than condition~\eqref{equ:linearStrictLPImprovingMapping} unless a special reparametrization is applied~\cite{Swoboda-PAMI-16}. Criterion~\eqref{equ:linearStrictLPImprovingMapping} extends to general substitutions and does not depend on reparametrization.
Similarly to~\cite{SwobodaPersistencyCVPR2014}, we will use approximate dual solvers in this more general setting.
\par
\par
In~\cite[($\varepsilon$-L1)]{shekhovtsov-14} problem~\eqref{equ:maxPersistency} is formulated as one big linear program.
We solve problem~\eqref{equ:maxPersistency}, and hence also the LP problem~\cite{shekhovtsov-14} in a more combinatorial fashion \wrt to the variables defining the substitution.
\par
It may seem that solving a series of linear programs rather than a single one is a disadvantage of the proposed approach. However, as we show further, the proposed iterative algorithm can be implemented using a dual, possibly suboptimal, solver for the relaxed verification problem~\eqref{equ:linearStrictLPImprovingMapping}. This turns out to be much more beneficitial in practice since the verification problems can be incrementally updated from iteration to iteration and solved overall very efficiently. This approach achieves scalability by exploiting available specialized approximate solvers for the relaxed MAP inference. 
Essentially, any dual (approximate) solver can be used as a black box in our method.

%% file: tex/alg_dual.tex

\section{Persistency with Dual Solvers}\label{sec:ApproximativeDualAlgorithm}

Though Algorithm~\ref{alg:iterative-LP} is quite general, its practical use is limited by the strict requirements on the solver, which must be able to determine the exact support set of all optimal solutions. 
However, finding even a single solution of the relaxed problem with standard methods like simplex or interior point can be practically infeasible and one has to switch to specialized solvers developed for this problem.
Although there are scalable algorithms based on smoothing techniques~\cite{Savchynskyy11,AdaptiveDiminishingSmoothing}, which converge to an optimal solution, waiting until convergence in {\em each} iteration of Algorithm~\ref{alg:iterative-LP} can make the whole procedure impractical. In general, we would like to avoid restricting ourselves to certain selected solvers to be able to choose the most efficient one for a given problem.
In the standard LP relaxation (introduced below in~\eqref{LP-big}), the number of primal variables grows quadratically with the number of labels, while the number of dual variables grows only linearly. It is therefore desirable to use solvers working in the dual domain, including suboptimal ones, (\eg~\cite{TRWSKolmogorov,AdaptiveDiminishingSmoothing,GlobersonJakollaNIPS2007,Schlesinger-diffusion}, performing block-coordinate descent) as they offer the most performance for a limited time budget. 
Furthermore, fast parallel versions of such methods have been developed to run on GPU/FPGA~\cite{Discrete-Continuous-16, ChoiR12,hurkatfast-15}, making the LP approach feasible for more vision applications.
\par
We will switch to the dual verification LP and gradually relax our requirements on the solution returned by a dual solver. This is done in the following steps:
\begin{enumerate}
\item \label{point-optimal} an optimal dual solution;
\item \label{point-ac} an arc consistent dual point;
\item \label{point-any} any dual point.
\end{enumerate}
 Our main objective is to ensure in each of these cases that the found substitution $p$ is strictly improving, while possibly compromising its maximality. The final practical algorithm operating in the mode~\ref{point-any} relies on the persistency problem reduction introduced in~\Section{sec:problemReduction}. Intermediate steps~\ref{point-optimal} and \ref{point-ac} are considered right after defining the standard LP relaxation and its dual.

\myparagraph{LP Relaxation} We consider the standard local polytope relaxation~\cite{Schlesinger-76,ALinearProgrammingApproachToMaxSumWerner,GraphicalModelsWainwrightJordan} of the energy minimization problem~\eqref{equ:energyMinimization} given by the following primal-dual pair: 
\begin{equation}\label{LP-big}
\begin{array}{lr}
\mbox{\ \ \ (primal)} & \mbox{(dual)\ \ \ } \\
\min \<f,\mu\>\tab\tab\tab\tab\tab = & \max f^{\varphi}_\emptyset\\ 
\sum_{j}\mu_{uv}(i,j) = \mu_u(i),  & \varphi_{uv}(i) \in \Real,\\
\sum_{i}\mu_{uv}(i,j)  = \mu_{v}(j),  & \varphi_{vu}(j) \in \Real,\\
\sum_{i}\mu_{u}(i) = \mu_\emptyset, & \varphi_{u} \in \Real,\\
\mu_u(i) \geq  0, & f^\varphi_u(i) \geq 0, \\
\mu_{uv}(i,j) \geq  0, & f^\varphi_{uv}(i,j) \geq 0, \\
\mu_\emptyset = 1
\end{array}
\end{equation}
where $f^\varphi$ abbreviates
\begin{subequations}
\begin{align}
& \textstyle
f^{\varphi}_u(i) = f_{u}(i)+\sum_{{v}\in \N({u})}\varphi_{uv}(i)-\varphi_u,\\
&\textstyle
f^{\varphi}_{uv}(i,j) = f_{uv}(i,j) -\varphi_{uv}(i)-\varphi_{vu}(j),\\
&\textstyle
f^{\varphi}_{\emptyset} = f_\emptyset + \sum_{u}\varphi_u.
\end{align}
\end{subequations}
The constraints of the primal problem~\eqref{LP-big} define the {\em local polytope} $\Lambda$. The cost vector $f^\varphi$ is called a {\em reparametrization} of $f$. There holds cost equivalence: $\<f^{\varphi},\mu\> = \<f,\mu\>$ for all $\mu \in \Lambda$ (as well as $E_f = E_{f^\varphi}$), see~\cite{ALinearProgrammingApproachToMaxSumWerner}.
Using the reparametrization, the dual problem~\eqref{LP-big} can be briefly expressed as
\begin{equation}\label{LP-dual}
\max_{\varphi}f^\varphi_\emptyset\ \ \ \mbox{s.t.}\  \ \ (\forall \omega\in\SV\cup\SE)\ f^\varphi_\omega \geq 0.
\end{equation}
Note that for a feasible $\varphi$ the value $f^\varphi_\emptyset$ is a lower bound on the primal problem~\eqref{LP-big}. In what follows we will assume that $\varphi$ in \eqref{LP-dual} additionally satisfies the following {\em normalization}: $\min_i f^\varphi_u(i) = 0$ and $\min_{ij} f^\varphi_{uv}(i,j) = 0$ for all $u$, $v$, which is automatically satisfied for any optimal solution.
\myparagraph{Expressing $\O^*_v$ in the Dual Domain} Let $(\mu,\varphi)$ be a pair of primal and dual optimal solutions to~\eqref{LP-big}.
 From complementary slackness we know that if $\mu_{v}(i)>0$ then the respective dual constraint holds with equality:
\begin{align}\label{c-slackness}
\mu_{v}(i)>0 \Rightarrow f^\varphi_v(i) = 0,
\end{align}
in this case we say that $f^\varphi_v(i)$ is {\em active}. The set of such active dual constraints matches the sets of local minimizers of the reparametrized problem,
\begin{equation} 
\O_v(\varphi) := \big\{i\in\X_v \mid  f^{\varphi}_v(i) = 0 \big\} = \argmin_{i} f^{\varphi}_v(i).
\end{equation}
From complementary slackness~\eqref{c-slackness} we obtain that
\begin{align}\label{O-inclusion}
\O_v^* \subset \O_v(\varphi).
\end{align}
This inclusion is insufficient for an exact reformulation of \Algorithm{alg:iterative-LP}, however it is sufficient for correctness if we make sure that $\Y_v \cap \O_v(\varphi) = \emptyset$ on termination, \ie, that the substitution $p$ does not displace labels in $\O_v(\varphi)$. Then, by~\autoref{prop:linearStrictLPImprovingMapping}, $p\in\BS_f$ follows.
\par
There always exists an optimal primal solution $\mu$ and dual $\varphi$ satisfying {\em strict complementarity}~\cite{Vanderbei2001}, in which case relation~\eqref{c-slackness} becomes an equivalence:
\begin{align}
\mu_{v}(i) > 0 \Leftrightarrow f^\varphi_v(i) = 0.
\end{align}
It is the case when both $\mu$ and $\varphi$ are relative interior points of the optimal primal, resp. optimal dual, faces~\cite{Vanderbei2001}. For a relative interior optimal $\varphi$, the set of constraints that are satisfied as equalities $f^\varphi(i) = 0$ is the smallest and does not depend on the specific choice of such $\varphi$. Under strict complementarity,~\eqref{O-inclusion} turns into equality $\O_v^*  = \O_v(\varphi)$, which allows to compute the exact maximum persistency using a dual algorithm without reconstructing a primal solution. However finding such $\varphi$ appears more difficult: e.g. the most efficient dual block-coordinate ascent solvers~\cite{TRWSKolmogorov,GlobersonJakollaNIPS2007,Cooper-10-SoftAC,Schlesinger-diffusion} only have convergence guarantees (see~\cite{TRWSKolmogorov,Schlesinger-diffusion}) 
allowing to find a sub-optimal solution, satisfying arc consistency.


\begin{definition}[\cite{ALinearProgrammingApproachToMaxSumWerner}]\label{def:arcConsitency}
A reparametrized problem $ f^\varphi$ is called {\em arc consistent} if:
(i) for all $uv\in\SE$ from $f^{\varphi}_{uv}(i,j)$ being active follows that $f^{\varphi}_u(i)$ and $f^{\varphi}_v(j)$ are active;
(ii) for all $u\in\SV$ from $f^{\varphi}_u(i)$ active follows that for all $v\in \Nb(u)$ there exists a $j\in\X_v$ such that $f^{\varphi}_{uv}(i,j)$ is active.
\end{definition}
An optimal dual solution need not be arc consistent, but 
it can be reparametrized without loss of optimality to enforce arc consistency~\cite{ALinearProgrammingApproachToMaxSumWerner}. Labels that become inactive during this procedure are not in the support set of primal solutions. 
In general the following holds.
\begin{restatable}{proposition}{propstrictCSisAC}\label{prop:strictCSisAC} Arc consistency is a necessary condition for relative interior optimality: if $\O_v(\varphi)=\O^*_v$ for all $v\in\SV$ then  $f^\varphi$ is arc consistent. 
\end{restatable}
This property is in our favor, since we are ideally interested in the equality $\O_v(\varphi)=\O^*_v$. Next, we propose an algorithm utilizing an arc consistent solver and prove that it is guaranteed to output $p\in \BS_f$.

\subsection{Persistency with an Arc Consistency Solver}\label{sec:persistency-AC}
\begin{algorithm}[t]
\KwIn{Cost vector $ f \in\Real^\SI$, test labeling $y\in\X$\;}
\KwOut{\underline{Strictly improving} substitution $p$\;}
$(\forall v\in\V)\ \Y_v := \X_v \backslash\{y_v\}$\;
\forever{}{
Construct verification problem $ g:=(I-[p])^\top f$ with $p$ defined by~\eqref{equ:pixelWiseMapping}\;
\underline{Use dual solver to find $\varphi$ such that $g^\varphi$ is arc consistent}\label{alg2:test-p}\;
	\underline{$\O_v(\varphi) : = \{i\in\X_v \mid  g^{\varphi}_v(i) = 0\}$}\;
	\lIf{$(\forall v\in\SV) \ \ \O_v(\varphi) \cap \Y_v = \emptyset$}{\Return{$p$}}\label{alg2:check-condition}
 	\For{$v\in\SV$}{
		Pruning of substitutions: $\Y_v := \Y_v \backslash \O_v(\varphi)$\label{alg2:update-U}\;
 	}
}
\caption{Iterative Pruning Arc Consistency\label{alg:iterative-DualLP}}
\end{algorithm}

We propose \Algorithm{alg:iterative-DualLP} which is based on a dual solver attaining the arc consistency condition (differences to~\cref{alg:iterative-LP} underlined). If the dual solver (in line~\ref{alg2:test-p}) finds a relative interior optimal solution, \Algorithm{alg:iterative-DualLP} solves~\eqref{equ:maxPersistency} exactly. Otherwise it is suboptimal and we need to reestablish its correctness and termination.
\begin{lemma}[Termination of \Algorithm{alg:iterative-DualLP}]\label{prop:correctnessAlg2}
\Algorithm{alg:iterative-DualLP} terminates in at most $\sum_v(|\X_v|-1)$ iterations. 
\end{lemma}
\begin{proof}
In case the return condition in line~\ref{alg2:check-condition} is not satisfied, $\O_v \cap \Y_v \neq \emptyset$ for some $v$ and the pruning in line~\ref{alg2:update-U} excludes at least one label from $\Y_v$.
\end{proof}

\begin{restatable}[Correctness of \Algorithm{alg:iterative-DualLP}]{lemma}{PdualstopAC}\label{P:dual-stop-AC}
If $(\forall v\in\SV) \ \ \O_v(\varphi) \cap \Y_v = \emptyset$ holds for an arc consistent dual vector $\varphi$, then $\varphi$ is optimal. 
\end{restatable}
It follows that when \Algorithm{alg:iterative-DualLP} terminates, the found arc consistent solution $\varphi$ is optimal, in which case inclusion ~\eqref{O-inclusion} is satisfied and the found substitution $p$ is guaranteed to be in $\BS_f$.
\par
\myparagraph{Solvers Converging to Arc Consistency}
One can see that arc consistency is only required on termination of \Algorithm{alg:iterative-DualLP}. In the intermediate iterations we may as well perform the pruning step, line~\ref{alg2:update-U}, without waiting for the solver to converge. This motivates the following  practical strategy: 
\begin{itemize}
\item Perform a number of iteration towards finding an arc-consistent dual point $\varphi$;
\item Check whether there are some labels to prune, \ie, $(\exists u) \O_u(\varphi) \cap \Y_u \neq \emptyset$;
\item Terminate if $\varphi$ is arc consistent and there is nothing to prune; otherwise, perform more iterations towards arc consistency.
\end{itemize}
\begin{figure}[tb]
\centering
\setlength{\tabcolsep}{0pt}
\begin{tabular}{cc}
\begin{tabular}{c}\includegraphics[width=0.45\linewidth]{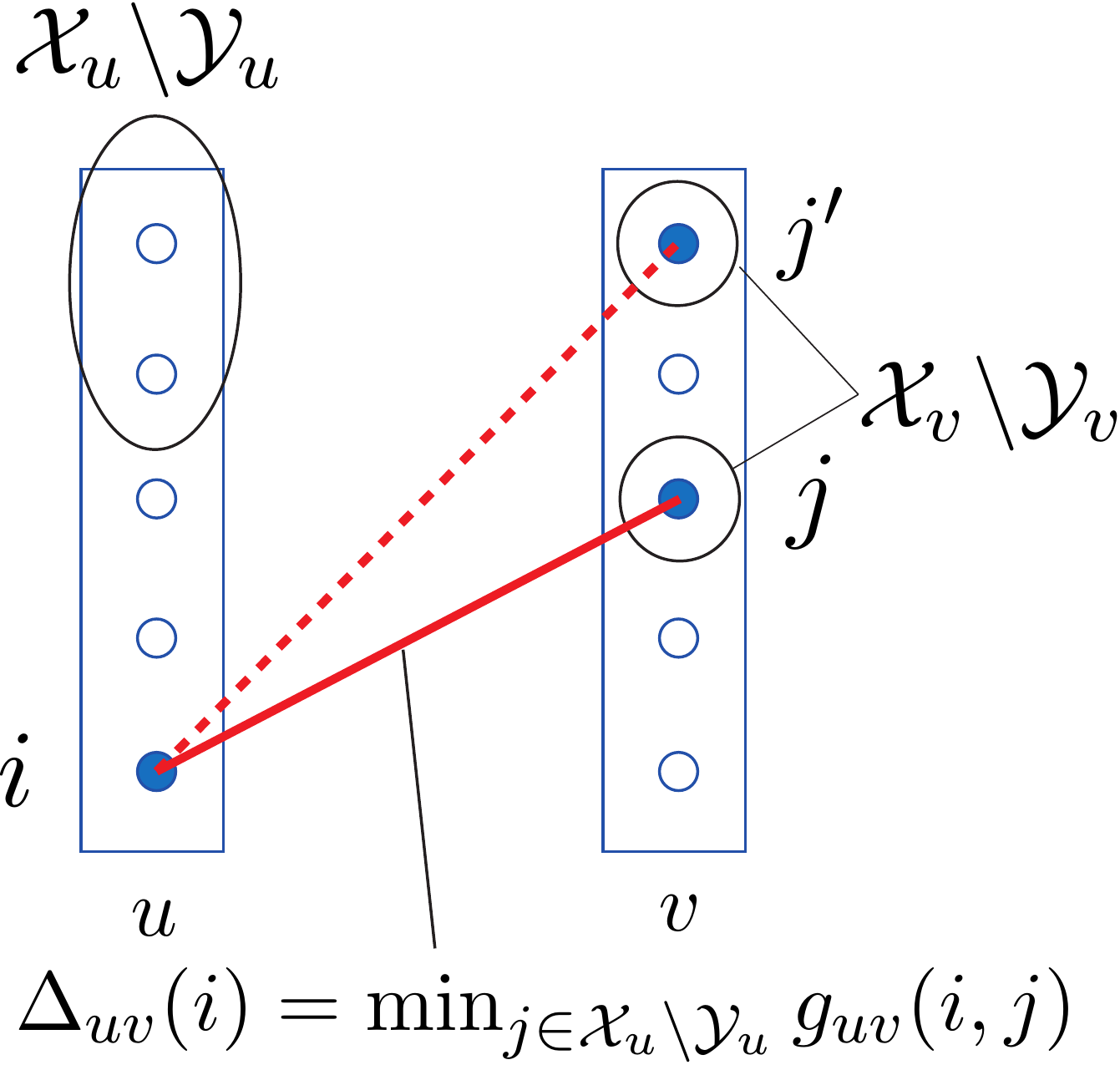}\end{tabular}\ \ \ \ & 
\begin{tabular}{c}\includegraphics[width=0.45\linewidth]{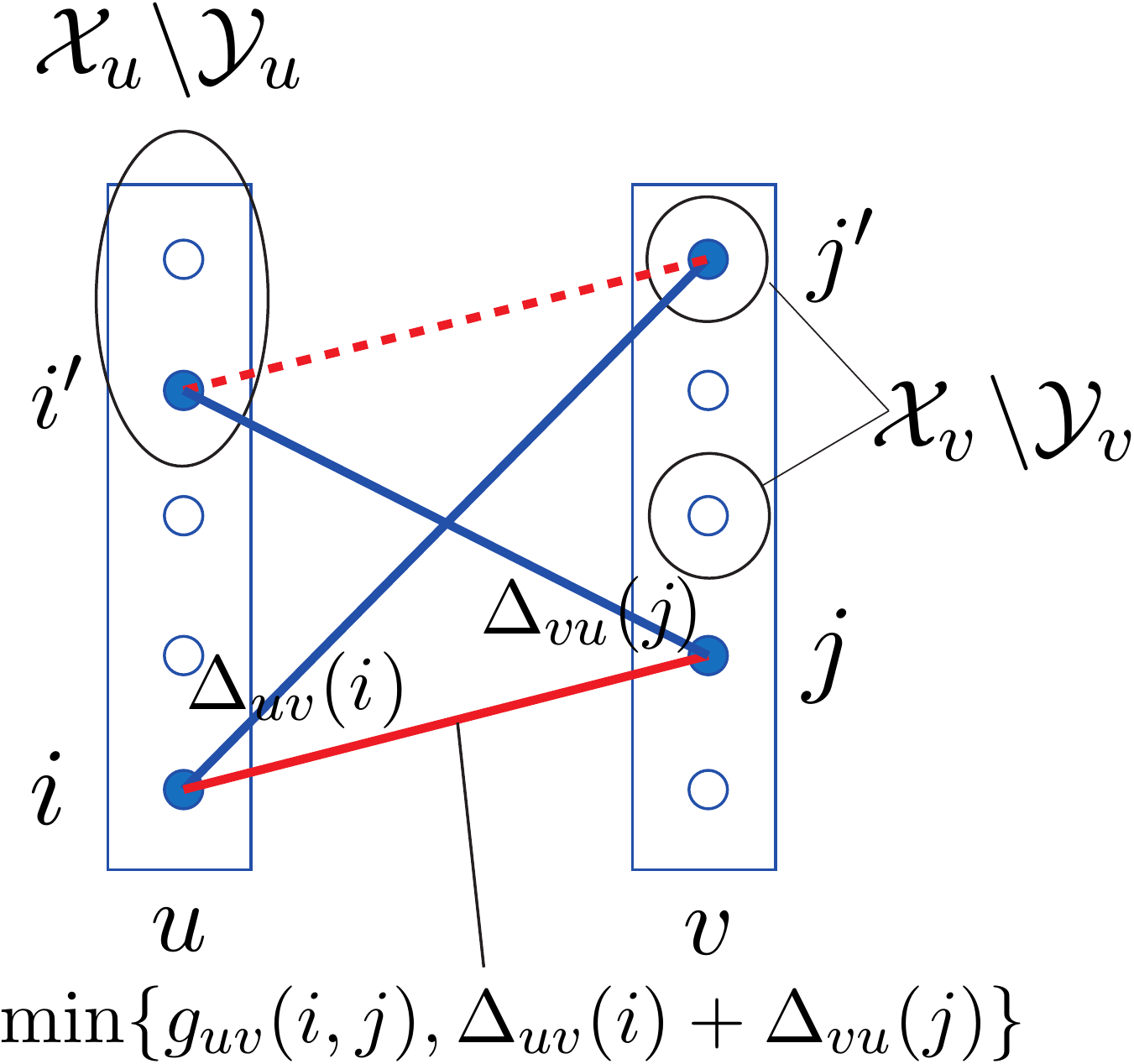}\end{tabular}\\
(a) & (b)
\end{tabular}
\caption{\label{fig:reduction}
Illustration for the reduction. Labels $\X_u \backslash \Y_u$ are not displaced by $p$ hence their associated unary and pairwise costs are zero in $g = (I-[p]\T)f$. In case (a) the indicated pairwise costs are replaced with their minimum. In case (b) the value of $g_{uv}(i,j)$ can be decreased, assuming all reductions of type (a) and their symmetric counterparts are already performed. The amount of decrease matches the value of the mixed derivative (non-submodularity) associated to $i,i'$ paired with $j,j'$.}
\end{figure}
If the solver is guaranteed to eventually find an arc consistent solution, the overall algorithm will either terminate with an arc consistent and (by~\autoref{P:dual-stop-AC}) optimal $\varphi$ or there will be some labels to prune.
However, we have to face the question what happens if the dual solver does not find an arc consistent solution in finite time. In this case the algorithm can be iterating infinitely with no pruning available. At the same time there is no guarantee that a pruning step will not occur at some point and thus if we simply terminate the algorithm we get no persistency guarantees. Even if the dual solver was guaranteed to converge in a finite number of iterations, it is in principle possible that the time needed for a pruning to succeed would be proportional to the time of convergence, making the whole algorithm  very slow. Instead, it is desirable to guarantee a valid result while allowing only a fixed time budget for the dual solver. We will overcome this difficulty with the help of the reduced verification LP presented next.

%% file: tex/reduction.tex
\section{Verification Problem Reduction}\label{sec:problemReduction}
\SetKwFunction{dualSolver}{{\ttfamily dual\_update}}
\SetKwFunction{dualCorrect}{{\ttfamily dual\_correct}}
Algorithms \ref{alg:iterative-LP} and \ref{alg:iterative-DualLP} iteratively solve verification problems. We can replace the verification LP solved in step~\ref{alg1:test-p} by a simpler, reduced one, {\em without loss of optimality} of the algorithms.
\begin{definition}\label{def:reduction}
Let $g:=(I-[p])^\top f$ be the cost vector of the verification LP. The {\em reduced cost vector} $\bar g$ is defined as
\begin{subequations}\label{equ:reductionTheorem}
\begin{align}
& \bar g_v(i) := g_v(i),\ v \in\SV;\ \bar g_\emptyset = 0; \\
	\label{equ:reductionTheorem-c}
  &\overline g_{uv}(i,j) := \\
	\notag
	& \left\{
	\setlength{\arraycolsep}{0pt}
  \begin{array}{ll}
   0, & i\notin \Y_u, j\notin \Y_v\,,\\
   \Delta_{vu}(j):=\min_{i'\notin \Y_u} g_{uv}(i',j), & i\notin \Y_u, j\in \Y_v\,,\\
   \Delta_{uv}(i):=\min_{j'\notin \Y_v} g_{uv}(i,j'), & i\in \Y_u, j\notin \Y_v\,,\\
	 \min\{\Delta_{vu}(j) +\Delta_{uv}(i), g_{uv}(i,j)\}, \ \ \ \ & i\in \Y_u, j\in \Y_v\,.
  \end{array}
\right.
\end{align}
\end{subequations}
\end{definition}
The reduction is illustrated in \Figure{fig:reduction}. Taking into account that $g_{uv}(i',j') = 0$ for $i'\in \X_u\backslash{\Y_u}$, $j'\in \X_v\backslash{\Y_v}$, the reduction can be interpreted as forcing the inequality
\begin{equation}\label{partial-submodularity}
g_{uv}(i,j')+g_{uv}(i',j) - g_{uv}(i,j) - g_{uv}(i',j') \geq 0, 
\end{equation}
\ie, the non-negativity of mixed discrete derivatives, for all four-tuples $i\in \Y_{u}$, $j\in \Y_{v}$, $i'\notin \Y_v$, $j'\notin \Y_u$.
The cost vector~$\bar g$ is therefore a {\em partial submodular truncation} of $g$. 

Recall that \Algorithm{alg:iterative-LP} on each iterations prunes all substitutions $q\leq p$ that do not belong to $\BS_f$ based on the solutions of the verification LP.
The following theorem reestablishes optimality of this step with the above reduction.

\begin{restatable}[Reduction]{theorem}{TReduction}\label{T:reduction 1-2}
Let $p\in\SP^{2,y}$ and $\bar g$ be the corresponding reduced cost vector constructed as in Def.~\ref{def:reduction}. 
Let also $q\in\SP^{2,y}$, $q\le p$. Then $q\in \BS_f$ iff $q\in \BS_{\bar g}$.
\end{restatable}

From \cref{T:reduction 1-2} and 
\cref{prop:linearStrictLPImprovingMapping} it follows that $q\in \BS_{f}$ iff $q_u(\O^*_u) = \O^*_u$, where $\O^*_u$ are the support sets of optimal solutions to the {\em reduced verification LP},
\begin{align}\label{reduced-LP}
\textstyle \argmin_{\mu\in\Lambda} \<\bar g, \mu\>.
\end{align}
Therefore it is valid for algorithms \ref{alg:iterative-LP} and \ref{alg:iterative-DualLP} 
to consider this reduced LP and prune all substitutions $q$ that do not satisfy the property $q_v(\O_v^*) = \O_v^*$.
The optimal relaxed solutions and their support sets can in general differ from those of the original verification LP, however for the purpose of the algorithm 
it is an equivalent replacement 
potentially affecting only the order in which substitutions are pruned.
\par
The reduction has the following advantages:
\begin{itemize}
\item subsets of labels $\X_v\backslash \Y_v$ can be contracted to a single representative label $y_v$, because associated unary and pairwise costs are equal;
\item It will allow (see~\Section{sec:any-dual}) to relax the requirements on approximate dual solvers needed to establish termination and correctness of the algorithm.
\item It is useful for the speed up heuristics (\Section{sec:speedups}). In particular, it is easier to find a labeling with a negative cost since we have decreased many edge costs. It will be shown that such a labeling allows for an early stopping of the dual solver and a pruning of substitution without loss of maximality. 
\end{itemize}

\par

%% file: tex/alg_dual_approx.tex
\section{Persistency with a Finite Number of Dual Updates}\label{sec:any-dual}
We assume that a suboptimal dual solver is iterative and can be represented by a procedure $\dualSolver$, which given a current dual point $\varphi$ makes a step resulting in a new dual point and a guess of a primal integer solution $x$. 
\par
In this setting we propose \Algorithm{Alg:iterative-AC}. In its inner loop, the algorithm calls \dualSolver (\autoref{trws-dual-update}) checks whether a speed-up shortcut is available (\autoref{trws-cut}) and verifies whether it can already terminate (lines \ref{trws-verification}-\ref{trws-return-step}). If neither occurs in a certain number of iterations (stopping condition in \autoref{TRWS-stop-condition-line}), the pruning based on the currently active labels is executed (line~\ref{trws-mark-step}). After that the cost vector $\bar g$ is rebuilt, but the dual solver continues from the last found dual point (warm start).
\par
The speed-ups will be explained in the next section, they are not critical for the overall correctness. Now we focus on the new termination conditions (lines \ref{trws-verification}-\ref{trws-return-step}).
A correction step (line~\ref{trws-verification}) is introduced whose purpose is to move the slacks from pairwise terms to unary terms so that active labels become more decisive. This procedure is defined in \autoref{Alg:correction}. The correction is not intermixed with dual updates but serves as a proxy between the solver and the termination conditions. It has the following property.

\begin{procedure}[t]
\SetKwProg{myproc}{Procedure}{}{}
\For{$uv\in\E$}{
$(\forall i\in\X_u)\ \ \varphi_{uv}(i) := \varphi_{uv}(i) + \min_{ij}{\bar g^{\varphi}_{uv}(i,j)}$\; \label{correct1}
$(\forall i\in\X_u)\ \ \varphi_{uv}(i) := \varphi_{uv}(i) + \min_{j}{\bar g^{\varphi}_{uv}(i,j)}$\;
\label{correct2} 
$(\forall j\in\X_v)\ \ \varphi_{vu}(j) := \varphi_{vu}(j) + \min_{i}{\bar g^{\varphi}_{uv}(i,j)}$\;
\label{correct3} 
}
$(\forall u\in\V)$\ $\varphi_{u} := \varphi_{u} + \min_{i}g^{\varphi}_{u}(i)$\label{correct4}\tcc*{Normalize}
\Return{$\varphi$}\;
\SetProcNameSty{texttt}
\caption{dual correct($\varphi$,$\bar g$)\label{Alg:correction}}
\end{procedure}

\begin{lemma}\label{lem:correction-props}Output $\varphi$ of \autoref{Alg:correction} is feasible and satisfies
\begin{align}
(\forall u\in\SV) \min_{i\in\SX_u}g^\varphi_{u}(i)=0, \label{correction-props-a}\\
\hspace{-10pt}(\forall uv\in\SE, ij\in \X_{uv}) \min_{i'\in\SX_u}\bar g^\varphi_{uv}(i',j)=\hspace{-2pt}\min_{j'\in\SX_v}\bar g^\varphi_{uv}(i,j') \hspace{-1.5pt} = \hspace{-1pt} 0. \label{correction-props-b}
\end{align}
Moreover, if the input $\varphi$ is feasible, the lower bound $f^{\varphi}_{\emptyset}$ does not decrease.
\end{lemma}
\begin{proof}
Line~\ref{correct1} of \autoref{Alg:correction} moves a constant from an edge to node. This turns the minimum of terms $g^{\varphi}_{uv}(i,j)$ to zero. Lines~\ref{correct2} and~\ref{correct3} turn to zero the minimal pairwise value attached to each label, which provides~\eqref{correction-props-b}. Line~\ref{correct4} provides~\eqref{correction-props-a}.
In case of feasibility of the initial $\varphi$, which implies $g^{\varphi}\ge 0$, all values of $\varphi$ can only increase during steps~\ref{correct1}-\ref{correct3} and hence the unary potentials $g^{\varphi}_u$ remain non-negative. Therefore step~\ref{correct4} can not decrease the lower bound value~$f^{\varphi}_{\emptyset}$.
\end{proof}
According to Lemma~\ref{lem:correction-props} \autoref{Alg:correction} can not worsen the lower bound attained by a dual solver. 
The following theorem guarantees that when no further pruning is possible, the corrected dual point constitutes an optimal solution, ensuring persistency.
\begin{theorem}\label{T:Correction}
Let $\varphi$ be a dual point for reduced problem $\bar g$ satisfying~\eqref{correction-props-a}-\eqref{correction-props-b}. Then either
\begin{enumerate}
\item $g^{\varphi}_{\emptyset} = 0$, $\varphi$ is dual optimal and $\delta(y)$ is primal optimal, or
\item $(\exists u \in \V) \ \O_u(\varphi) \cap \Y_u \neq \emptyset$.
\end{enumerate}
\end{theorem}
\begin{proof}
Assume (b) does not hold: $(\forall u\in\V)$ \ \ $\O_u(\varphi) \subseteq \X_u \backslash \Y_u$. Let us pick in each node $u$ a label $z_u\in\O_u(\varphi)$. As ensured by~\eqref{correction-props-b}, for each edge $uv$ there is a label $j\in \X_v$ such that  $\bar g^{\varphi}_{uv}(z_u,j) = 0$ and similarly, there exists $i\in \X_u$ such that $\bar g^{\varphi}_{uv}(i,z_v) = 0$. By partial submodularity of $\bar g$, we have 
\begin{align}\label{ac-U-eq1-n}
\bar g^{\varphi}_{uv}(z_u,z_v) + \bar g^{\varphi}_{uv}(i,j) \leq \bar g^{\varphi}_{uv}(z_u,j)+ \bar g^{\varphi}_{uv}(i,z_v) = 0.
\end{align}
Therefore, $\bar g^{\varphi}_{uv}(z_u,z_v) \leq -\bar g^{\varphi}_{uv}(i,j) \leq 0$. Hence $\bar g^{\varphi}_{uv}(z_u,z_v) = 0$ and it is active. Therefore $\delta(z)$ and dual point $\varphi$ satisfy complementarity slackness conditions and hence they are primal-dual optimal and $g^{\varphi}_{\emptyset} = E_{\bar g}(z) = 0 = E_{\bar g}(y)$.
\end{proof}
\begin{theorem}[Termination and Correctness of \Algorithm{Alg:iterative-AC}]\label{T:TRWS-term}
For any stopping condition in \autoref{TRWS-stop-condition-line}, \Algorithm{Alg:iterative-AC} terminates in at most $\sum_v (|\X_v|-1)$ outer iterations and returns $p\in\BS_f$.
\end{theorem}
\begin{proof}
When the algorithm has not yet terminated some further pruning is guaranteed to be possible (compare conditions in lines~\ref{trws-return-step} and \ref{trws-mark-step}). The iteration limit follows.
When \Algorithm{Alg:iterative-AC} terminates, from \Theorem{T:Correction} it follows that $\varphi'$ is dual optimal and hence $\O_u(\varphi') \supset \O_u^*$. Therefore, $(\forall u\in\SV) \ \ \O^*_u \cap \Y_u = \emptyset$, which is sufficient for $p$ to be strictly $\Lambda$-improving according to \autoref{prop:linearStrictLPImprovingMapping}.
\end{proof}
\par
In \cite{SSS-15-IR-Ar} we prove that a similar result holds for a TRW-S iteration without correction by arguing on complete chain subproblems instead of individual nodes. The correction might be needed in case the algorithm does not keep slacks on the nodes, \eg for SRMP~\cite{ReweightedMessagePassingRevisitedKolmogorov}.
\par
The stopping condition in \autoref{TRWS-stop-condition-line} of Algorithm~\ref{Alg:iterative-AC} controls the aggressiveness of pruning. Performing fewer iterations may result only in the found $p$ not being the maximum, but in any case it is guaranteed that the Algorithm~\ref{Alg:iterative-AC} does not stall and identifies {\em a correct persistency}. In the case when the solver does have convergence and optimality guarantees, the time budget controls the degree of approximation to the maximum persistency. 
\par
\begin{algorithm}[t]
\KwIn{Problem $f\in\Real^\SI$, test labeling $y\in\X$\;}
\KwOut{Improving substitution $p \in \SP^{2,y} \cap {\BS_f}$\;}
$(\forall u\in\V)$ $\Y_u := \X_u\backslash\{y_u\}$\;
Set $\varphi$ to the initial dual solution if available\;
\forever{}{
  Apply single node pruning \label{rebuild_energy}\tcc*{speed-up}
	Construct reduced verification LP $\bar g$ from $f$ and current sets $\SY_u$, according to 
	\autoref{def:reduction}\;
	\Repeat{{\rm any stopping condition (\eg, iteration limit)}}{
	\label{ip-trws-repeat}
		$(\varphi,x) := \dualSolver(\bar g, \varphi)$\label{trws-dual-update}\;
		\If(\tcc*[h]{speed-up}){$E_{\bar g}(x) < 0$}{\label{trws-cut}
		Apply pruning cut with $x$\;
		{\bf goto} step~\ref{rebuild_energy} to rebuild $\bar g$\;
		}
		\tcc{Verification of Optimality}
		$\varphi' := \dualCorrect(\bar g, \varphi)$\label{trws-verification}\;
		$\O_u := \{i \mid \bar g^{\varphi'}_u(i) = 0 \}$\; 
		\lIf{$(\forall u\in\V)\ \O_u\cap \Y_u = \emptyset$}{\Return{$p$} by~\eqref{equ:pixelWiseMapping}}\label{trws-return-step}
	}\label{TRWS-stop-condition-line}
	Prune: $(\forall u\in\V)$ $\Y_u := \Y_u \backslash \O_u$\label{trws-mark-step}\;
}
\SetKwProg{myproc}{Procedure}{}{}
  \myproc{\dualSolver($\bar g$,$\varphi$)}{
	\KwIn{Cost vector $\bar g$, dual point $\varphi$\;}
  \KwOut{New dual point $\varphi$, approximate primal integer solution $x$\;}
	}
\caption{\label{Alg:iterative-AC}
Efficient Iterative Pruning
}
\end{algorithm}

%% file: tex/speedups.tex
\section{Speed-ups}\label{sec:speedups}
\subsection{Inference Termination Without Loss of Maximality} 
Next, we propose several sufficient conditions to quickly prune some substitutions without worsening the final solution found by the algorithm. As follows from Definition~\ref{def:linearStrictLPImprovingMapping}, existence of a labeling $x$ such that $\lan (I-[p])^\top f,\delta(x)\ran \leq 0$ and $x\neq p(x)$ is sufficient to prove that substitution $p$ is {\em not} strictly $\Lambda$-improving. Hence one could consider updating the current substituttion $p$ without waiting for an exact solution of the inference problem in line~\ref{alg1:test-p}. 
The tricky part is to find labels that can be pruned without loss of optimality of the algorithm.  
Lemma~\ref{lemma:LPrefinement} below suggests to solve a simpler verification LP, $\min_{\mu\in\Lambda'} \<\bar g, \mu\>$ over a subset $\Lambda'$ of $\Lambda$. This does not guarantee to remove all non-improving substitutions (which implies one has to switch to $\Lambda$ afterwards), but can be much more efficient than the optimization over $\Lambda$. After the lemma we provide two examples of such efficient procedures.
\begin{restatable}{lemma}{LPrefinement}\label{lemma:LPrefinement}
Let $p\in\SP^{2,y}$ and $\bar g$ be defined by~\eqref{equ:reductionTheorem} (depends on $p$). Let $q\in\BS_{f}\cap\SP^{2,y}$, $q\leq p$, $Q=[q]$. Let $\Lambda'\subset\Lambda$, $Q(\Lambda') \subset \Lambda'$ and $\O^* = \argmin_{\mu\in\Lambda'} \<\bar g, \mu\>$.
Then $(\forall v\in\SV) \ \ q_v(\O^*_v) =  \O^*_v$.
\end{restatable}
Note, while \Theorem{T:reduction 1-2} is necessary and sufficient for pruning, \Lemma{lemma:LPrefinement} is only sufficient.
\par
\myparagraph{Pruning of Negative Labelings} 
Assume we found an integer labeling $x$ such that $E_{\bar g}(x) \leq 0$ and $p(x)\neq x$. 
Lemma~\ref{lemma:LPrefinement} gives an answer, for which nodes $v$ the label $x_v$ can be pruned from the set $\Y_v$ without loss of optimality. Define the following restriction of the polytope $\Lambda$:
\begin{align}
{\Lambda_{x}=\{\mu \in \Lambda \mid (\forall v\in\V)\ \mu(y_v)+\mu(x_v)=1\}\subset \Lambda}.
\end{align} 
Polytope $\Lambda_{x}$ corresponds to the restriction of $\Lambda$ to the label set $\{y_v,x_v\}$ in each node $v\in\SV$.
According to Lemma~\ref{lemma:LPrefinement} we need to solve the problem 
\begin{equation}\label{eq:cut}
\textstyle
\O^* := \argmin_{\mu\in\Lambda_x} \<\bar g, \mu\>\,
\end{equation}
and exclude $x_v$ from $\Y_v$ if $x_v \in \O^*_v$. Due to the partial submodularity of $\bar g$ the problem~\eqref{eq:cut} is submodular and can be solved by min-cut/max-flow algorithms~\cite{EnergiesAndGraphCutsKolmogorov}. Because $x$ was found to have non-positive energy, it is necessarily that for some nodes $v$ there will hold $x_v \in \O^*_v \cap \Y_v$ and therefore some pruning will take place.

\myparagraph{Single Node Pruning} Let us consider ``a single node'' polytope $\Lambda_{u,i}:=\{\mu\in\Lambda \mid \mu_u(y_u)+\mu_u(i)=1; (\forall v\neq u)\ \mu_v(y_v)=1\}$. It is a special case of $\Lambda_x$ when $y$ and $x$ differ in a single node $u$ only and $x_u=i$. In this case problem~\eqref{eq:cut} amounts to calculating
$
\bar g_u(x_u) + \sum_{v\in\N(u)} \bar g_{uv}(x_u,y_v)\,.
$ 
If this value is non-positive, $x_u$ must be excluded from $\Y_u$. 
The single node pruning can be applied to all pairs $(u,i)$ exhaustively, but it is more efficient to keep track of the nodes for which sets $\Y_v$ have changed (either due to a negative labeling pruning, active labels pruning in \autoref{trws-mark-step} or the single node pruning itself) and check their neighbors. 

\subsection{Efficient Message Passing}
The main computational element in dual coordinate ascent solvers like TRWS or MPLP is {\em passing a message}, \ie, an update of the form $\min_{i\in\X_u}(f_{uv}(i,j) + a(i))$. 
In many practical cases the message passing for $f$ can be computed in time linear in the number of labels~\cite{Hirata-1996,Meijster2000,Aggarwal-1987}.
This is the case when $f_{uv}$ is a convex function of $i{-}j$ (\eg, $|i{-}j|$, $(i{-}j)^2$) or a minimum of few such functions (\eg Potts model is $\min(1,|i{-}j|)$). 
However, in \Algorithm{alg:iterative-LP} we need to solve the problem with the cost vector $g = (I-P\T)f$, resp. $\bar g$~\eqref{equ:reductionTheorem} if we apply the reduction.
It turns out that whenever there is a fast message passing method for $f$, the same holds for $\bar g$. 
\begin{restatable}[Fast message passing]{theorem}{TFastMsg}\label{T:message passing reduced complexity}
Message passing for an edge term $\bar g_{uv}$~\eqref{equ:reductionTheorem} can be reduced to that for $f_{uv}$ in time $O(|\Y_u|+|\Y_v|)$.
\end{restatable}
This complexity is proportional to the size of the sets $\Y_u$. The more labels are pruned from sets $\Y_u$ in the course of the algorithm, the less work is required.
\par
Note, that contrary to limiting the number of iterations of a dual solver, described in \Section{sec:any-dual}, the speedups presented in this section do {\em not} sacrifice the persistence maximality~\eqref{equ:maxPersistency}. In our experiments for some instances, Algorithm~\ref{Alg:iterative-AC} finished before ever reaching step~\ref{trws-mark-step}. In such cases the found substitution $p\in\BS_f$ is the maximum. 

%% file: tex/experiments.tex
\section[Experimental Evaluation]{Experimental Evaluation\footnote{The implementation of our method is available through 
\url{http://cmp.felk.cvut.cz/~shekhovt/persistency}}
}\label{sec:experiments}
\input{tex/tables/table_methods.tex}
\input{tex/tables/table_random.tex}
\input{tex/tables/table_average.tex}
In the experiments we study how well we approximate the maximum persistency~\cite{shekhovtsov-14}, \autoref{table:random}; illustrate the contribution of different speedups, \autoref{exp:speedups}; give an overall performance comparison to a larger set of relevant methods, \autoref{tab:Experiments}; and provide a more detailed direct comparison to the most relevant scalable method~\cite{SwobodaPersistencyCVPR2014} using exact and approximate LP solvers, \autoref{tab:proteins}.
As a measure of persistency we use the percentage of labels eliminated by the improving substitution~$p$:
\input{tex/tables/table_speedups.tex}
\input{tex/tables/table_proteins.tex}
\begin{align}\label{measure-labels}
\begin{array}{c}
\frac{\sum_{v\in\V}|\X_v \backslash p_v(\X_v)|}{\sum_{v\in\V}(|\X_v|-1)}  = \frac{\sum_{v\in\V}|\Y_v|}{\sum_{v\in\V}(|\X_v|-1)}. 
\end{array}
\end{align}
\myparagraph{Random Instances}
Table~\ref{table:random} gives comparison to~\cite{SwobodaPersistencyCVPR2014} and \cite{shekhovtsov-14} on random instances generated as in~\cite{shekhovtsov-14} (small problems on 4-connected grid with uniformly distributed integer potentials for ``{\tt full}'' model and of the Potts type for ``{\tt Potts}'' model, all not LP-tight). It can be seen that our exact \Algorithm{alg:iterative-LP} performs identically to the $\varepsilon$-L1 formulation~\cite{shekhovtsov-14}. Although it solves a series of LPs, as opposed to a single LP solved by $\varepsilon$-L1, it scales better to larger instances. Instances of size 20x20 in the $\varepsilon$-L1 formulation are already too difficult for CPLEX: it takes excessive time and sometimes returns a computational error. The performance of the dual \Algorithm{Alg:iterative-AC} confirms that we loose very little in terms of persistency but gain significantly in speed. 
\par
\myparagraph{Benchmark Problems}
Table~\ref{tab:Experiments} summarizes average performance on the OpenGM MRF benchmark~\cite{kappes-2015-ijcv}. The datasets include previous benchmark instances from computer vision~\cite{SzeliskiComparativeStudyMRF} and protein structure prediction~\cite{PIC2011,SideChainPredictionYanover} as well as other models from the literature. 
Results per instance are given in~\cref{sec:results-per-instance}.
\par
\myparagraph{Speedups}
In this experiment we report how much speed improvement was achieved with each subsequent technique of \Section{sec:speedups}. The evaluation in Table~\ref{exp:speedups} starts with a basic implementation (using only a warm start). The solver is allowed to run at most 50 iterations in the partial optimality phase until pruning is attempted. We expect that on most datasets the percentage of persistent labels improves when we apply the speedups (since they are without loss of maximality). 

\par
\myparagraph{Discussion\label{sec:discussion}}
Tables~\ref{table:random} and~\ref{tab:proteins} demonstrate that \IRITRWS, which is using a suboptimal dual solver, closely approximates the maximum persistency~\cite{shekhovtsov-14}. Our method is also significantly faster and scales much better.
The method~\cite{SwobodaPersistencyCVPR2014} is the closest contender to ours in terms of algorithm design.
Tables~\ref{table:random},~\ref{tab:Experiments} and~\ref{tab:proteins} clearly show that our method determines a larger set of persistent variables. This holds true with exact (CPLEX) as well as approximate (TRWS) solvers.
There are two reasons for that as discusssed in~\cref{sec:prev-work}. First, we optimize over a larger set of substitutions than~\cite{SwobodaPersistencyCVPR2014}, \ie, we identify per-label persistencies while~\cite{SwobodaPersistencyCVPR2014} is limited to the whole-variable persistencies. Second, even in the case of the whole-variable persistencies the criterion in~\cite{SwobodaPersistencyCVPR2014} is in general weaker than~\eqref{equ:linearStrictLPImprovingMapping} and depends on the initial reparametrization of the problem. 
This later difference does not matter for Potts models~\cite{Swoboda-PAMI-16}, the examples~\cref{fig:pfau,fig:ted}, but does matter, \eg, in \cref{fig:pano}.
%
%
%
 Although our method searches over a significantly larger space of possible substitutions, it needs fewer TRW-S iterations due to speedup techniques. Details on iteration counts can be found in~\cref{sec:results-per-instance}.
In the comparison of running time it should be taken into account that different methods are optimized to a different degree. 
Nevertheless, it is clear that the algorithmic speedups were crucial in making the proposed method much more practical than~\cite{SwobodaPersistencyCVPR2014} and~\cite{shekhovtsov-14} while maintaining high persistency recall quality.

\par
To provide more insights to the numbers reported, we illustrate in \cref{fig:teaser-progress,fig:pfau,fig:easy,fig:ted,fig:pano} some interesting cases. \cref{fig:pfau} shows ``the hardest'' instance of {\tt color-seg-n4} family. Identified persistencies allow to fix a single label in most of the pixels, but for some pixels more than one possible label remains. The {\em remainder} of the problem has the reduced search space $p(\X)$, which can be passed to further solvers. The {\tt tsukuba} image~\cref{fig:teaser-progress} is interesting because it has appeared in many previous works. The performance of graph-cut based persistency methods relies very much on strong unary costs, while the proposed method is more robust. 
\cref{fig:easy} shows an easy example from {\tt object-seg}, where LP relaxation is tight, the dual solver finds the optimal labeling $y$ and our verification LP confirms that this solution is unique. In \cref{fig:ted} we show a hard instance of {\tt mrf-stereo}. Partial reason for its hardness is integer costs, leading to non-uniqueness of the optimal solution. 
In \cref{fig:pano}, {\tt photomontage/pano} instance, we report $79\%$ solution completeness, but most of these $79\%$ correspond to trivial forbidden labels in the problem (very big unary costs). At the same time other methods perform even worse. This problem has hard interaction constraints. It seems that hard constraints and ambiguous solutions pose difficulties to all methods including ours.
\begin{figure}[t]
\centering
\setlength{\tabcolsep}{0pt}
\begin{tabular}{ll}
\begin{tabular}{c}
\labelgraphics[width=0.45\linewidth]{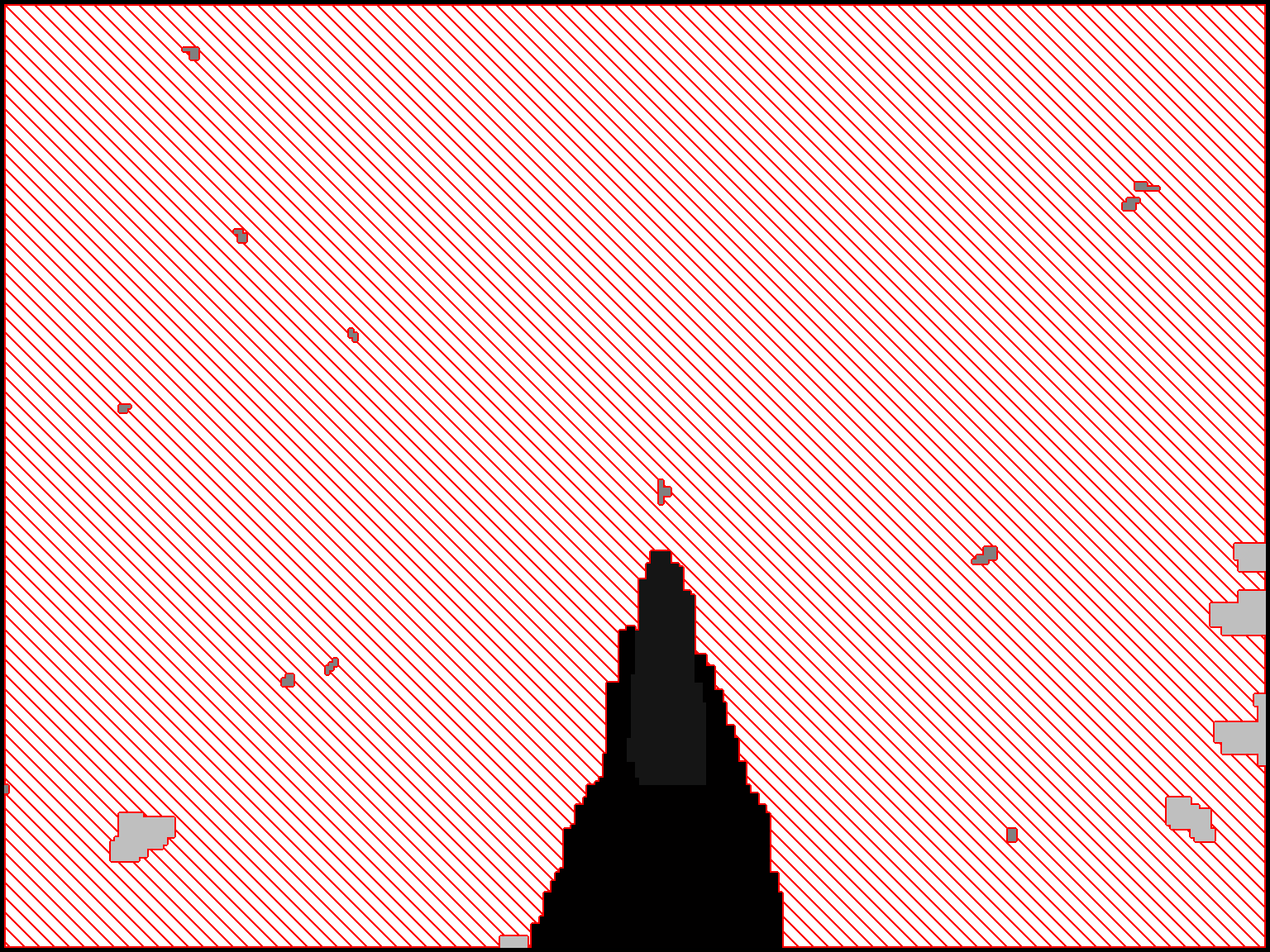}{\Kovtun: 1s, 5.6\%}%
\end{tabular}&\
\begin{tabular}{c}
\setlength{\fboxsep}{0pt}%
\fbox{\labelgraphics[width=0.45\linewidth]{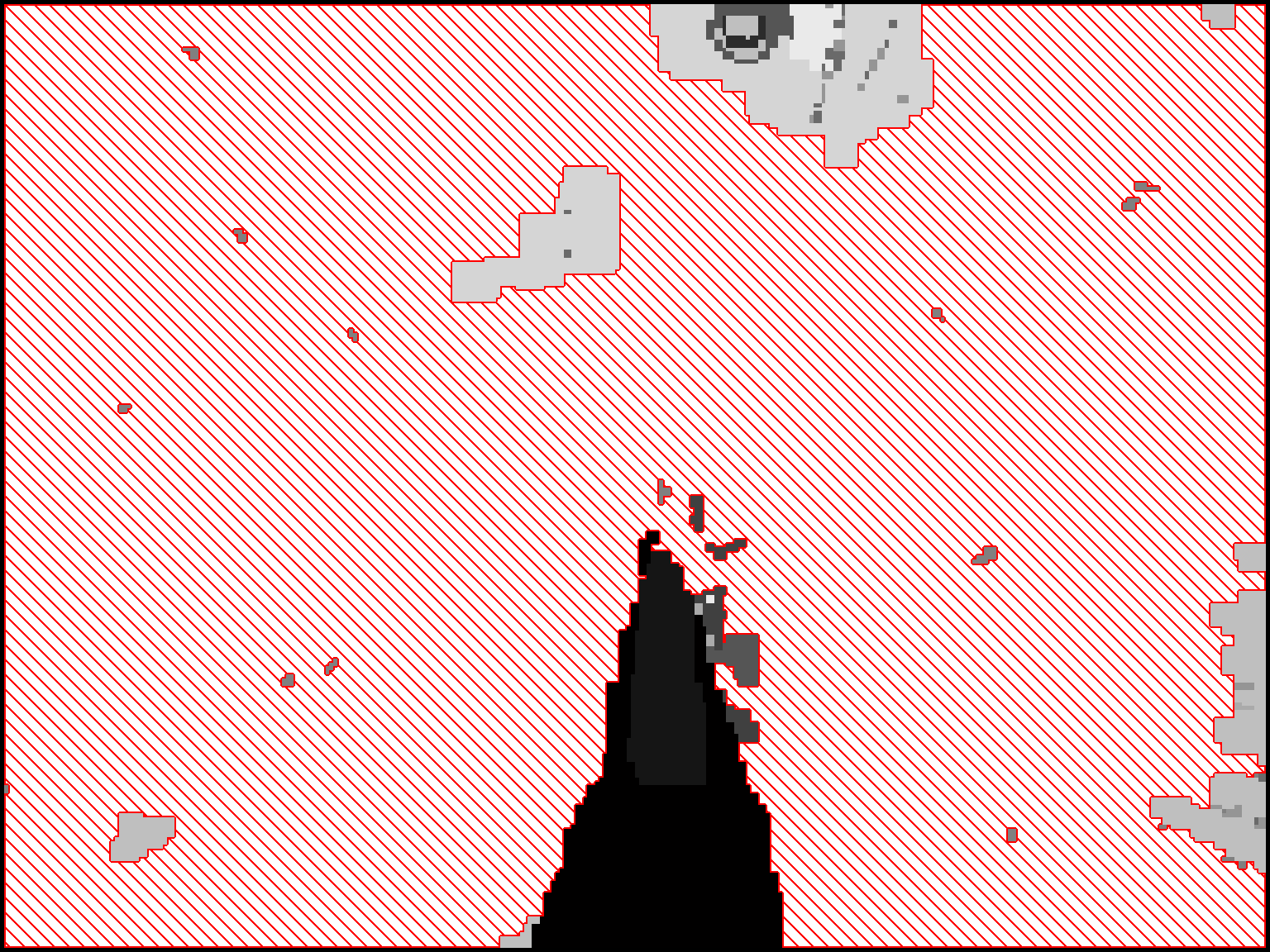}{\PBPTRWS: 13 min 10.43\%}}%
\end{tabular}
\\
\begin{tabular}{c}
\labelgraphics[width=0.45\linewidth]{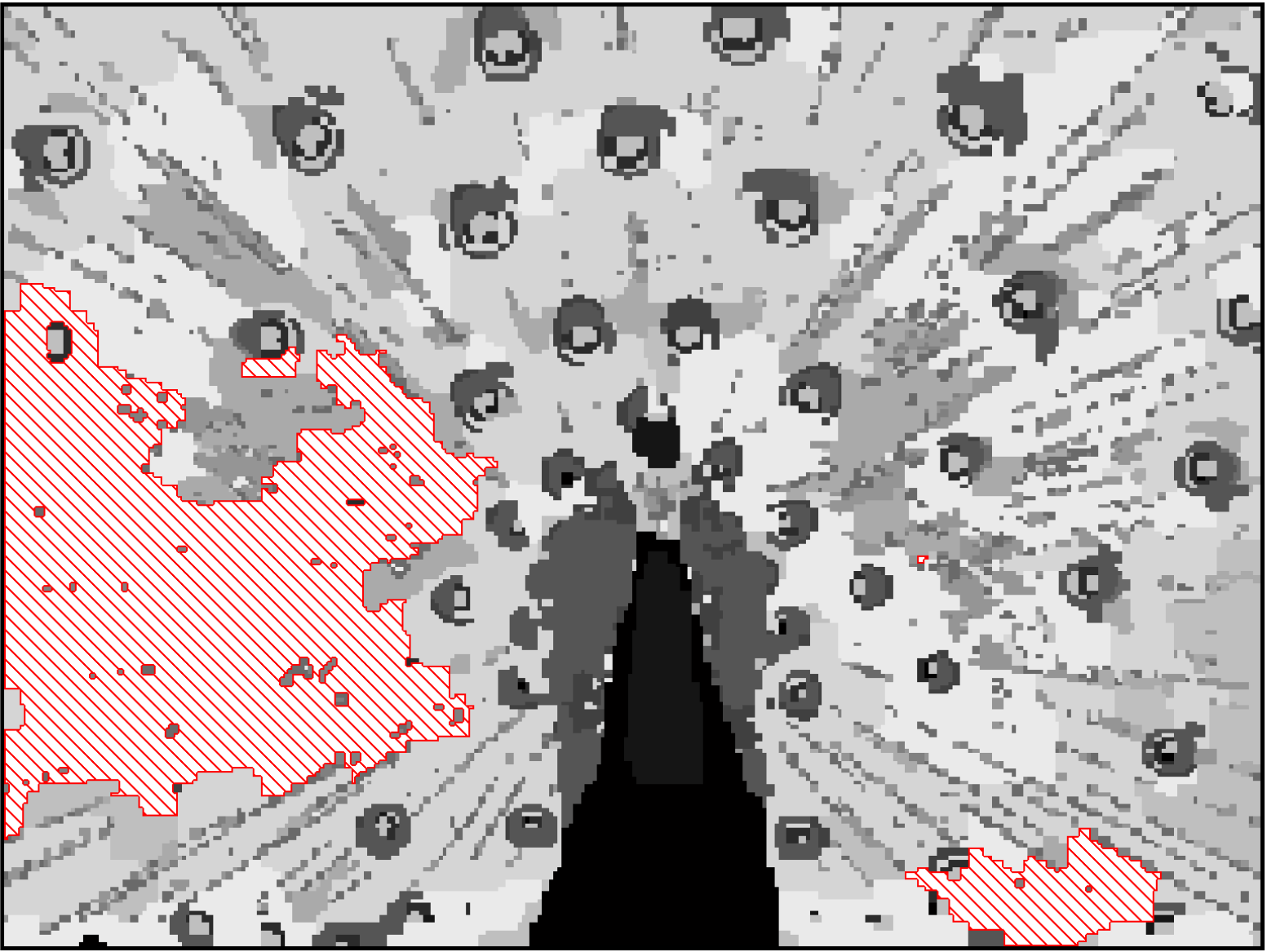}{Ours: 10+12s, 93.4\%}
\end{tabular}&\ 
\begin{tabular}{c}
\labelgraphics[width=0.5\linewidth]{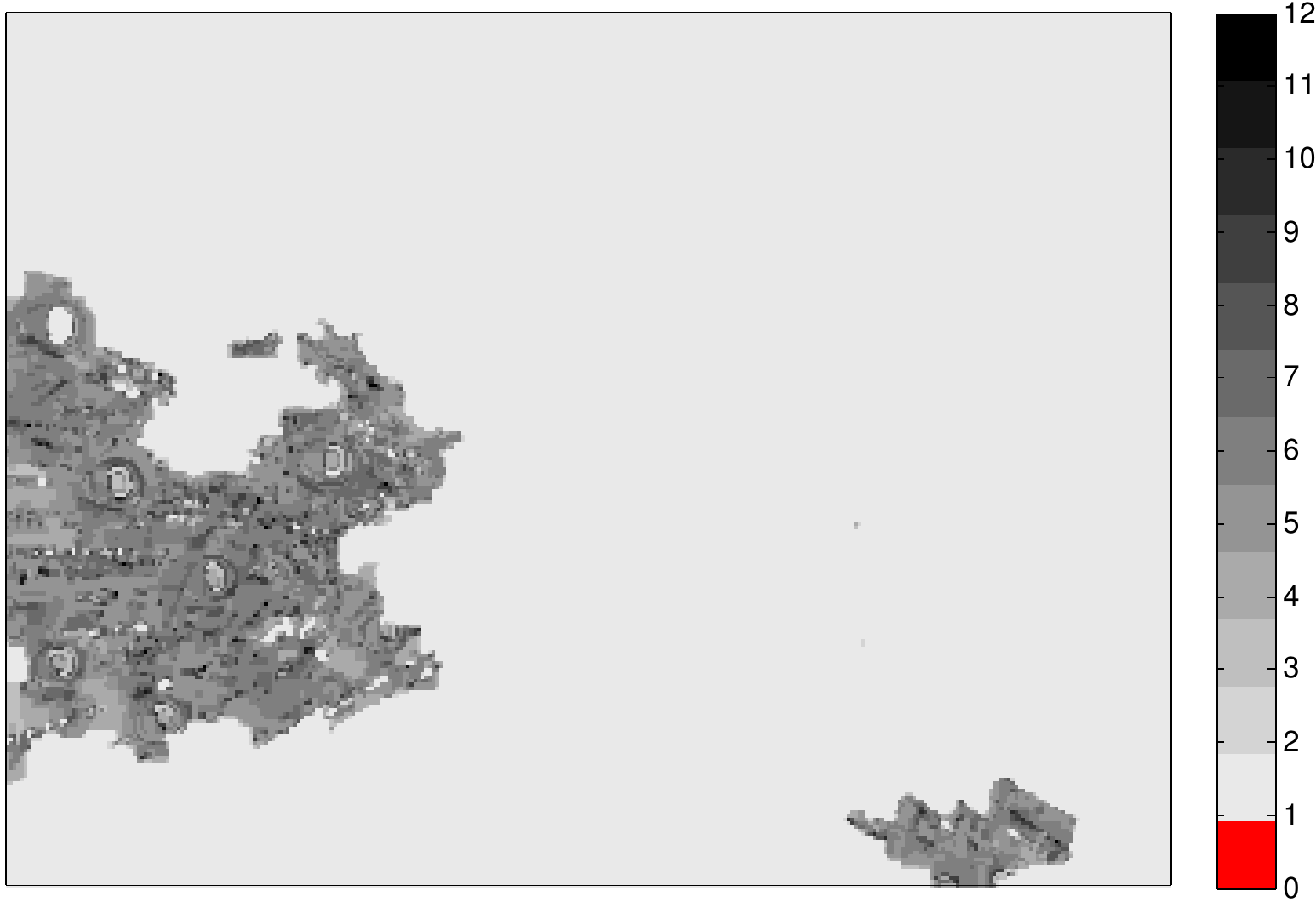}{Ours, remainder}
\end{tabular}
\end{tabular}
\caption{Performance on a hard segmentation problem. The remainder of the problem visualizes $|p_u(X_u)|$ for all pixels.
The result of the improved method \PBPTRWSO is the same as \PBPTRWS, because this is a Potts model.
\label{fig:pfau}}
\end{figure}

\begin{figure}[t]
\begin{tabular}{p{\linewidth}}
\begin{tabular}{C{0.333\linewidth}C{0.333\linewidth}C{0.333\linewidth}}
\labelgraphics[width=0.98\linewidth]{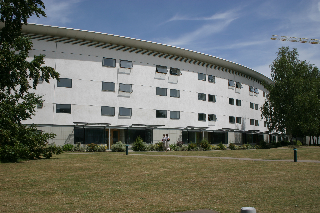}{\tt object-seg}&
\labelgraphics[width=0.98\linewidth]{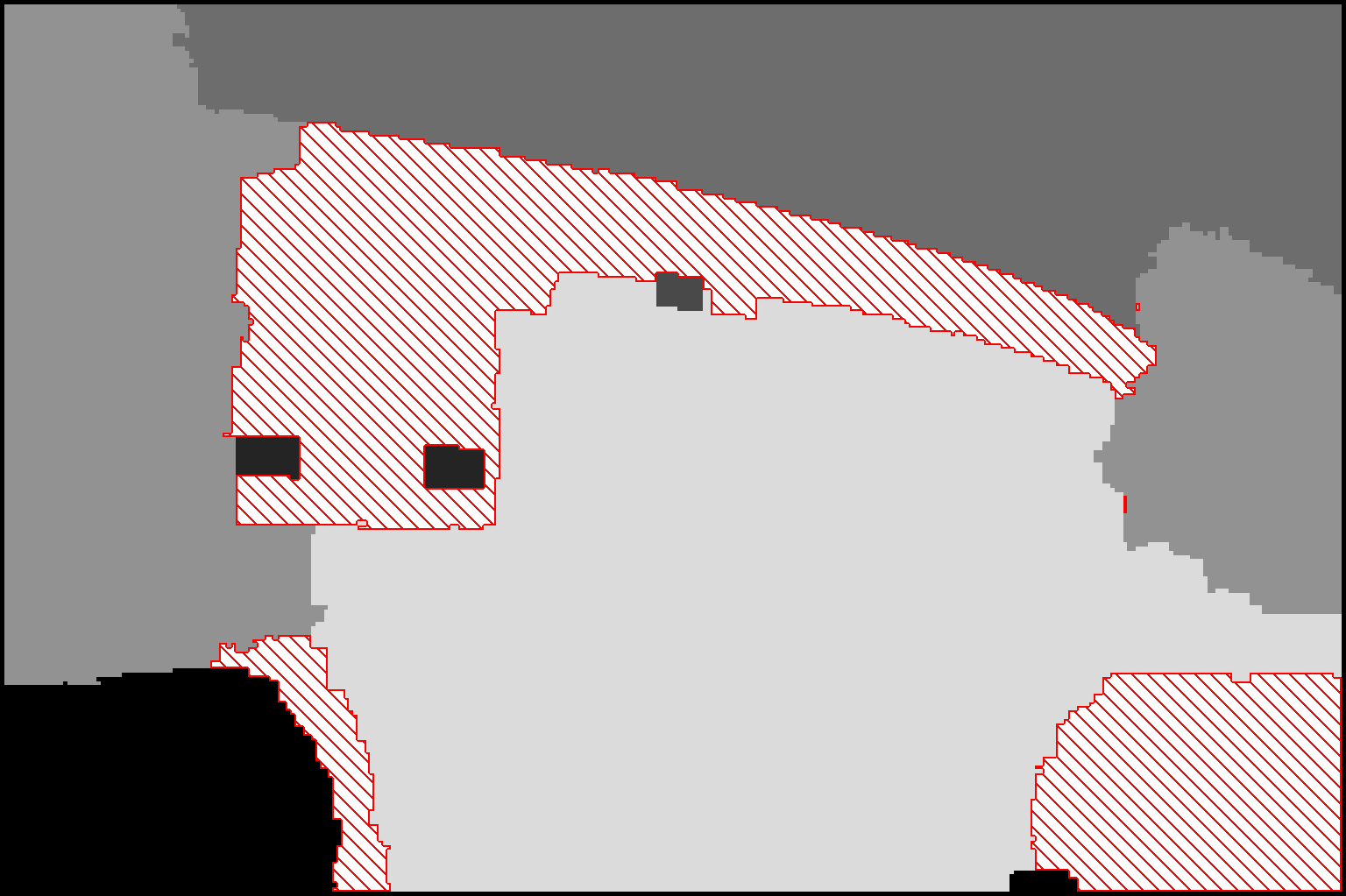}{\cite{Kovtun03}: 0.2s, 80.5\%}&
\labelgraphics[width=0.98\linewidth]{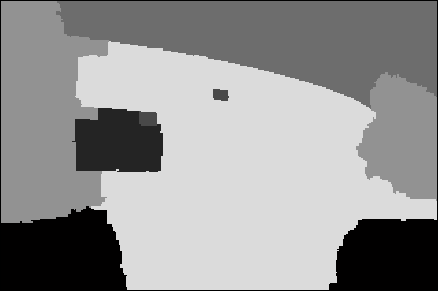}{Ours: 1.4s, 100\% \protect\\ (LP-tight)}
\end{tabular}
\end{tabular}
\caption{Examples of an easy problem, {\tt object-seg}. TRWS finds the optimal solution (zero integrality gap) and therefore our method as well as \PBPTRWS, \PBPTRWSO report 100\%. \label{fig:easy}
}
\end{figure}

\begin{figure}[t]
\centering
\begin{tabular}{C{0.45\linewidth}C{0.45\linewidth}}
\labelgraphics[width=0.98\linewidth]{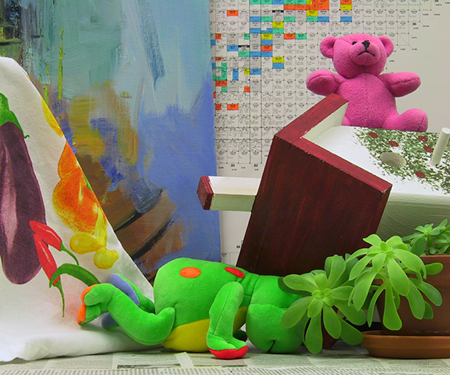}{\tt mrf-stereo}&
\labelgraphics[width=0.98\linewidth]{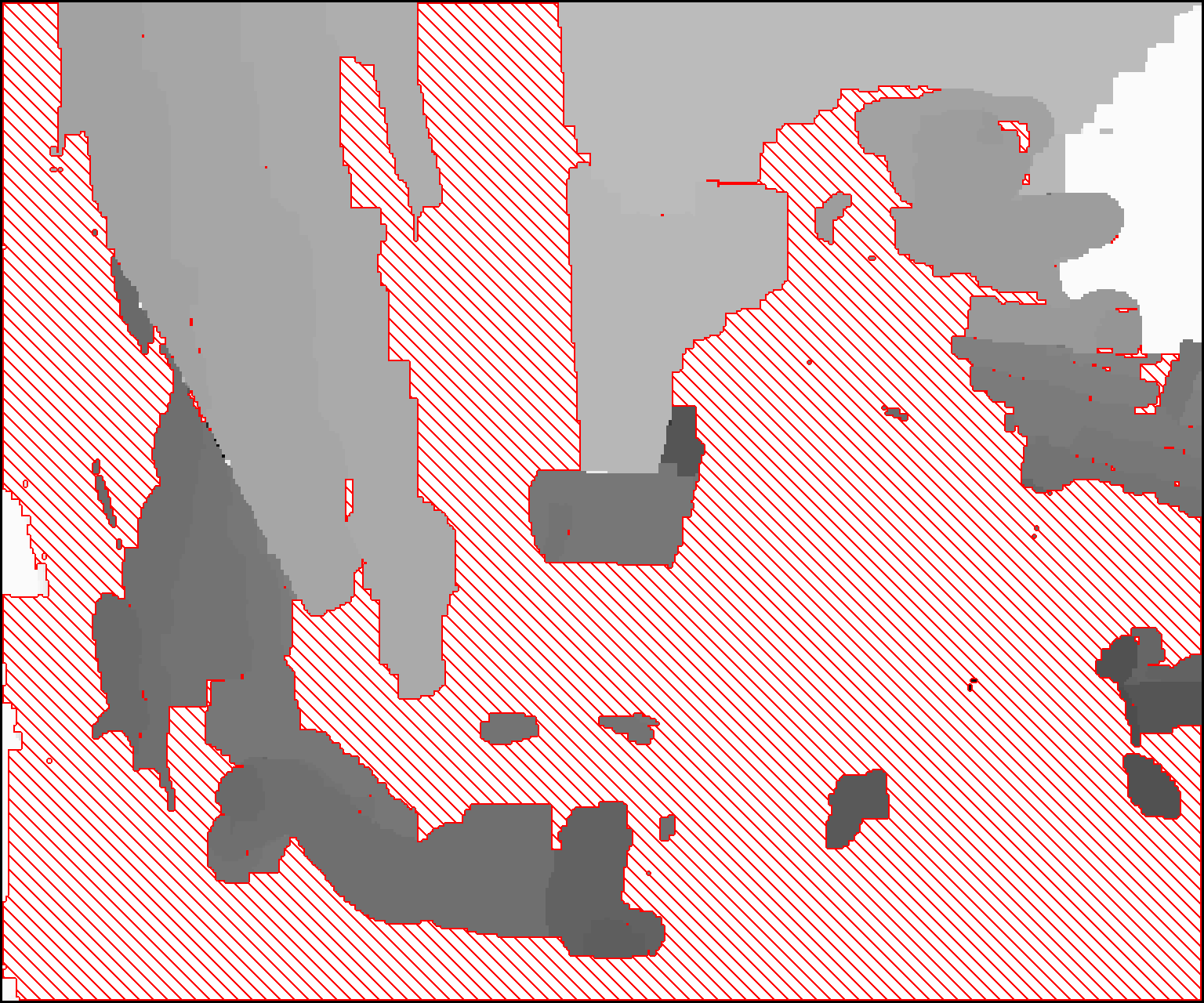}{\PBPTRWS: 1h, 38\%}\\
\labelgraphics[width=0.98\linewidth]{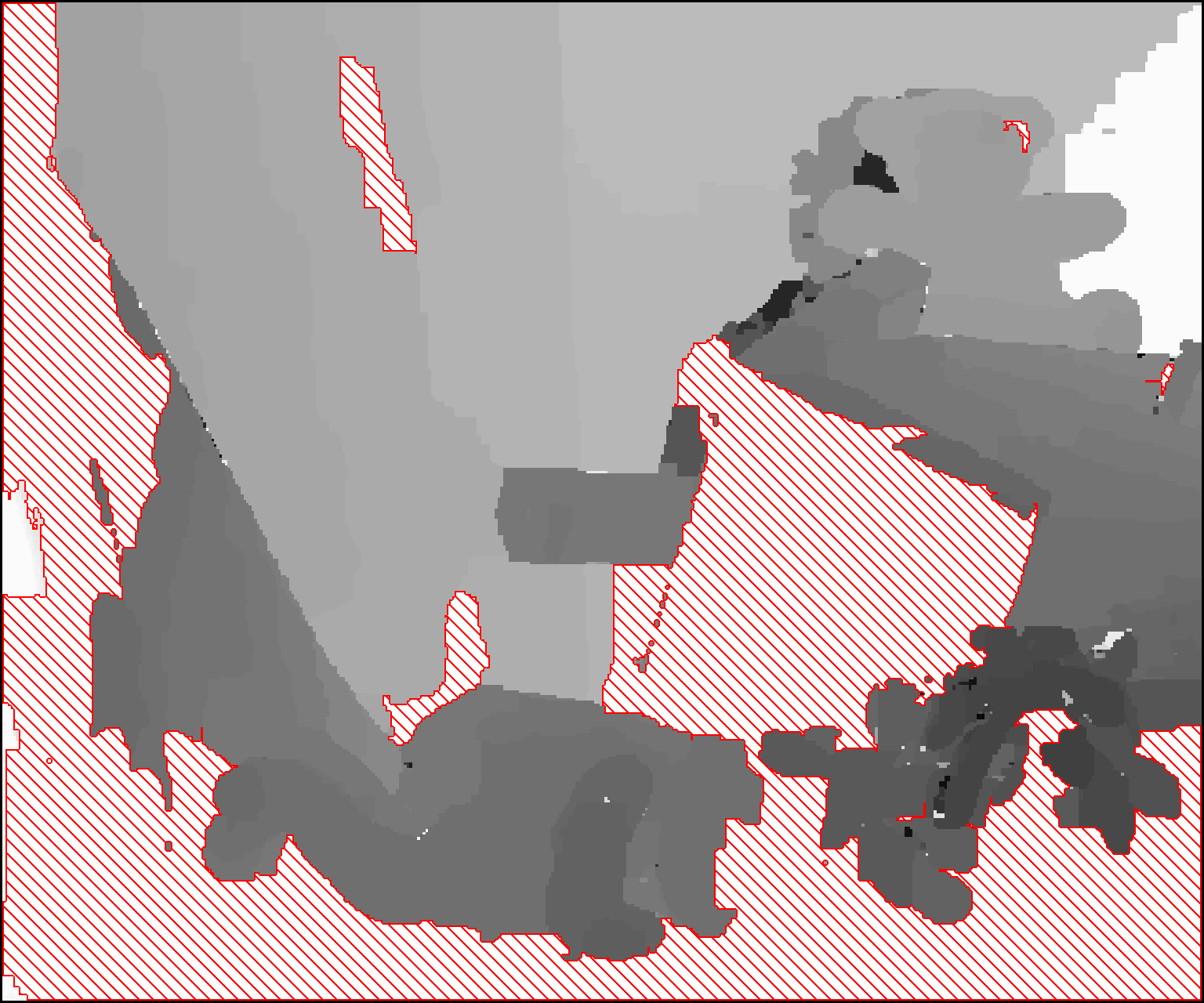}{Ours: 62+180s, 75\%}&
\labelgraphics[width=0.98\linewidth]{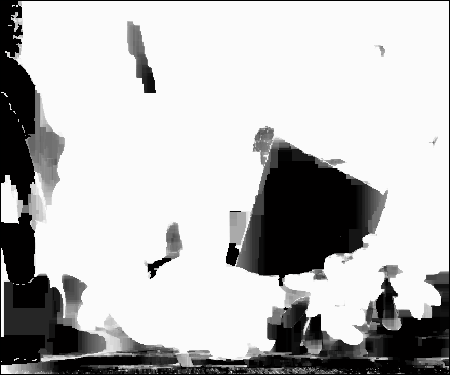}{Ours, remainder}
\end{tabular}
\caption{Examples of a hard stereo problem, {\tt ted-gm}. Method of \Kovtun gives 0.42\% in 2.5s and is therefore not displayed. The result \PBPTRWSO is the same as \PBPTRWS, because this is a Potts model too.
\label{fig:ted}
}
\end{figure}

\begin{figure*}[t]
\begin{tabular}{c}
\vspace{2pt}
\labelgraphics[width=0.2\linewidth]{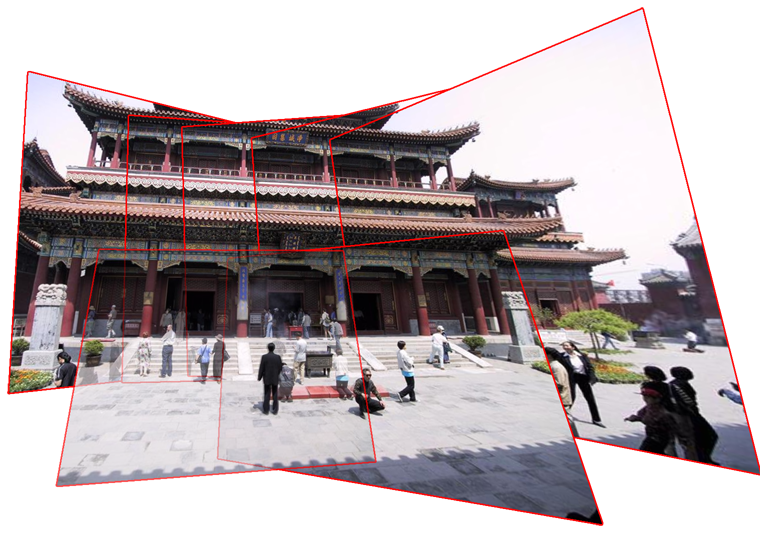}{\tt mrf-photomontage}
\end{tabular}\ %
\begin{tabular}{c}
\labelgraphics[width=0.25\linewidth]{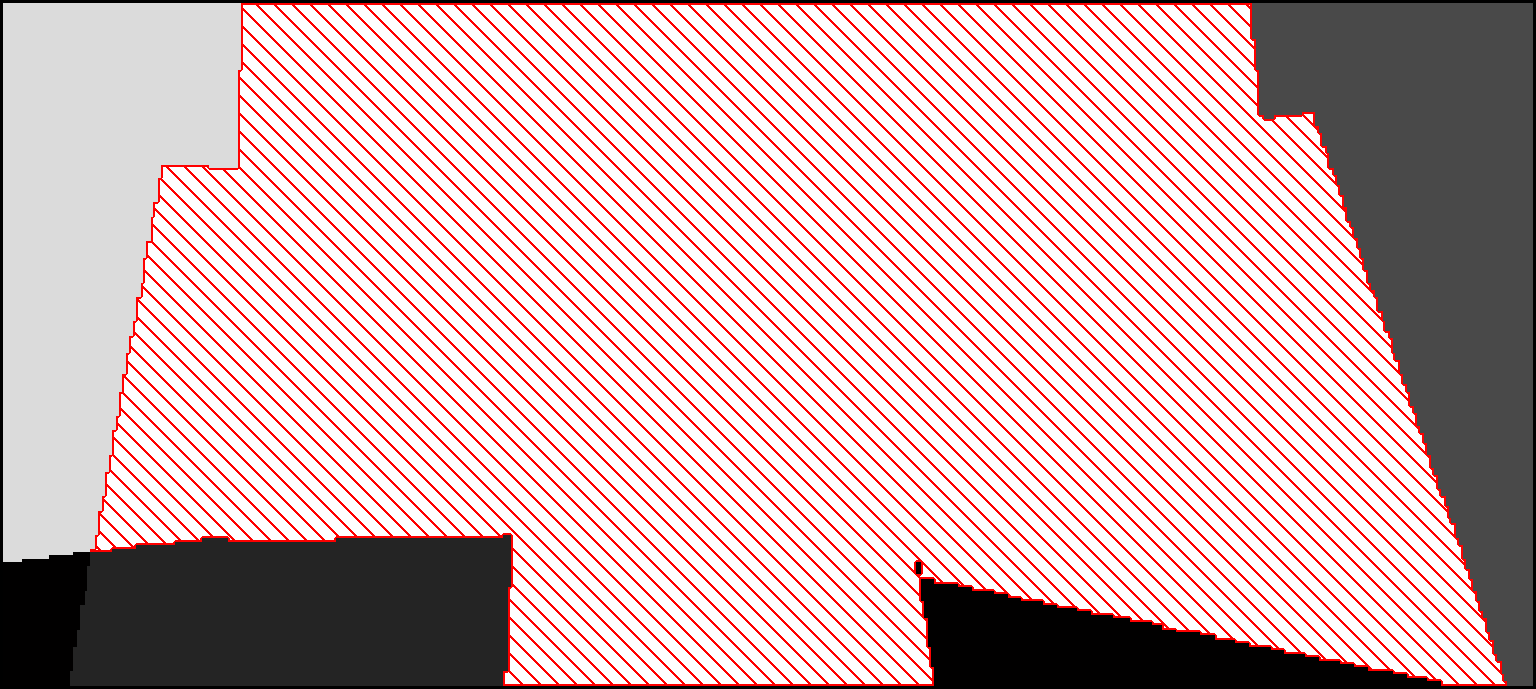}{\citet{Kovtun03}: 0.5s, 27.5\%}\\%
\end{tabular}\ %
\begin{tabular}{c}
\labelgraphics[width=0.25\linewidth]{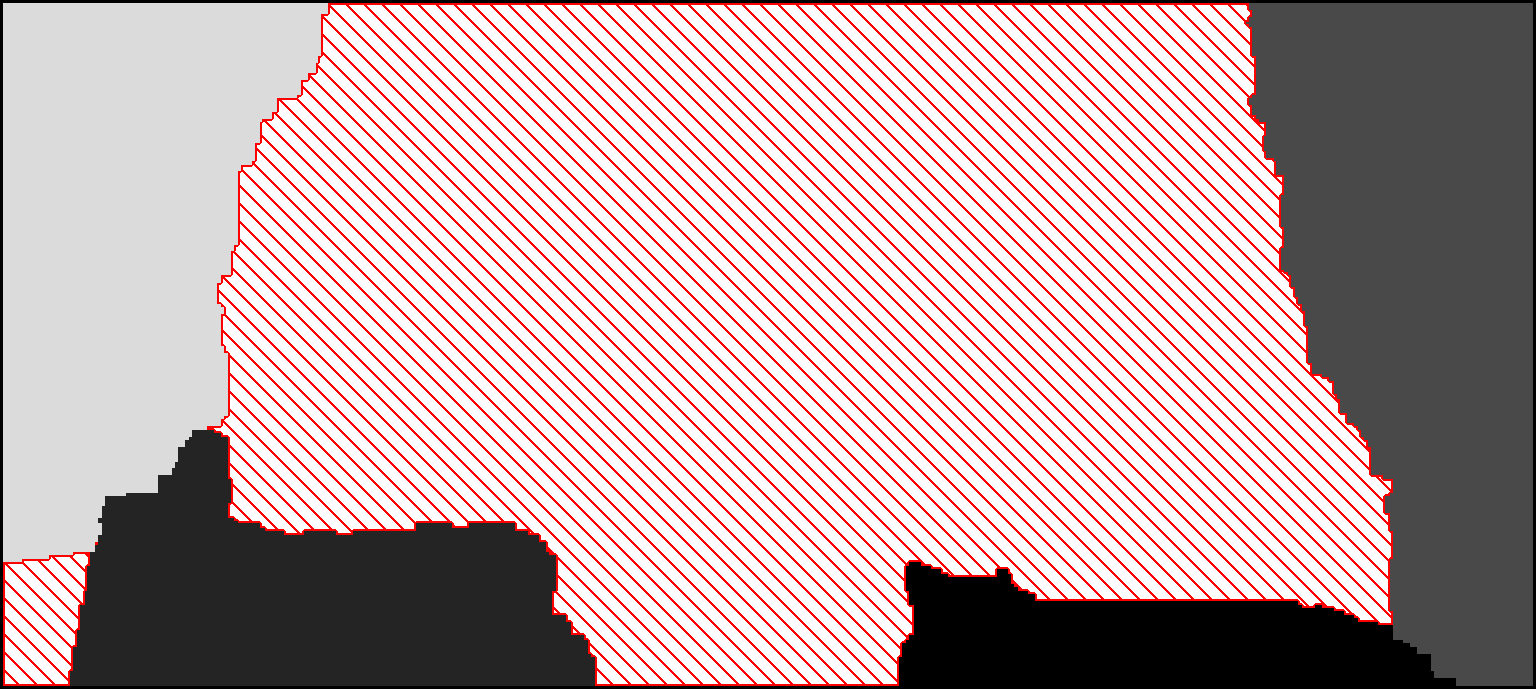}{\PBPTRWSO: 3h, 37.75\%}\\%
\end{tabular}\ %
\begin{tabular}{c}
\labelgraphics[width=0.28\linewidth]{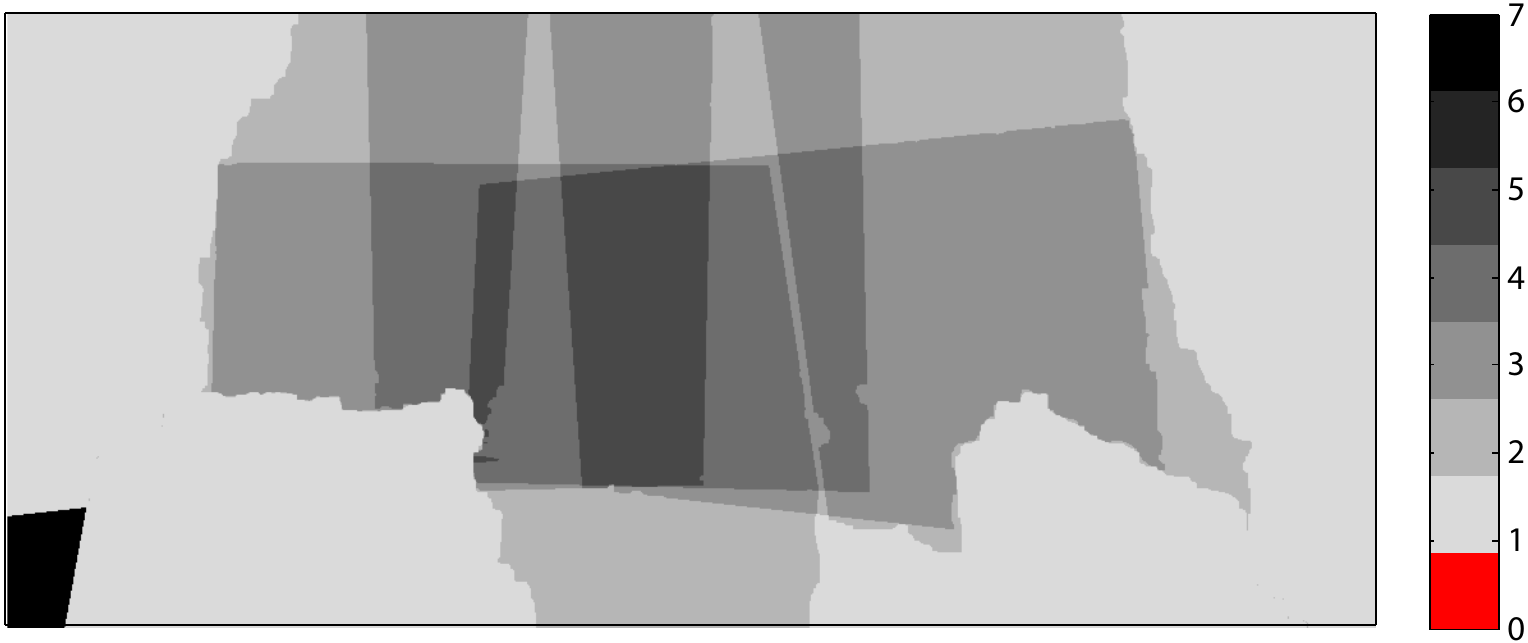}{Ours, remainder: 130+390s, 79.2\%}
\end{tabular}
\caption{Example of a very hard instance, {\tt photomontage}. The number of views covering a pixel is usually smaller than the total number of views  = number of labels. The $79\%$ by our method mostly originate from the elimination of these redundant labels. The result of \PBPTRWS gives $27.55\%$ (visually same as \citet{Kovtun03}) in 3.9h. \PBPTRWSO	improves on this due to choosing the optimal reparametrization~\cite{Swoboda-PAMI-16}. Note that our partial labeling (the part with one label remaining) is larger.
\label{fig:pano}
}
\end{figure*}

%% file: tex/tables/table_methods.tex
\begin{table}[tb]
\setlength{\tabcolsep}{2pt}
\resizebox{\linewidth}{!}{%
\begin{tabularx}{\linewidth}{|r|X|}
\hline
\raggedright
\small \IRICPLEX & Our \Algorithm{alg:iterative-LP} (Iterative Relaxed Inference) using CPLEX~\cite{CPLEX}.\\
\small \IRITRWS & Our \Algorithm{Alg:iterative-AC} using TRW-S~\cite{TRWSKolmogorov}. Initial solution uses at most 1000 iterations (or the method has converged). All speedups. \\
\small \PBPCPLEX & Method~\cite{SwobodaPersistencyCVPR2014} with CLPEX.\\

\small \PBPTRWS & Method~\cite{SwobodaPersistencyCVPR2014} with TRW-S. \\ 
\small \Shekhovt & Single LP formulation of the maximum strong persistency~\cite{shekhovtsov-14} solved with CPLEX.\\
\small \Kovtun & One-against-all method of Kovtun~\cite{Kovtun-10}.\\
\small \MQPBO & Multilabel QPBO~\cite{PartialOptimalityInMultiLabelMRFsKohli}.\\
\small \MQPBO-10 & MQPBO with 10 random permutations, accumulating persistency.\\
%
\hline
\end{tabularx}
}
\caption{List of Evaluated Methods}
\end{table}

%% file: tex/tables/table_random.tex
\begin{table*}[!tb]
\setlength{\tabcolsep}{5pt}
\centering
\begin{tabular}{|l|rl|rl|rl|rl|rl|}
\hline
Problem family & \multicolumn{2}{c|}{\PBPCPLEX}
&  \multicolumn{2}{c|}{\PBPTRWS}
&  \multicolumn{2}{c|}{\Shekhovt}  
&  \multicolumn{2}{c|}{\IRICPLEX}
&  \multicolumn{2}{c|}{\IRITRWS}\\
\hline
\tt 10x10 Potts-3 &  0.18s & 58.46\%  &  0.05s & 58.38\%  &  0.05s & {\bf 72.27}\%  &  0.18s & {\bf 72.27}\% &  0.04s & 72.21\% \\
\tt 10x10  full-3 &  0.24s & \ \,2.64\%  &  0.09s & \ \,1.22\%  &  0.06s & {\bf 62.90}\%  &  0.24s & {\bf 62.90}\% &  0.05s & 62.57\% \\
\hline
\tt 20x20 Potts-3 &  3.25s & 73.95\%  &  0.21s & 68.49\%  &  0.87s & {\bf 87.38}\%  &  2.43s & {\bf 87.38}\% &  0.06s & {\bf 87.38}\% \\
\tt 20x20  full-3 &  2.81s & \ \,0.83\%  &  0.37s & \ \,0.83\%  &  0.95s & {\bf 72.66}\%  &  3.03s & {\bf 72.66}\% &  0.07s & 72.31\% \\
\hline
\tt 20x20 Potts-4 & 12.45s & 23.62\%  &  0.39s & 18.43\%  & 19.40s & {\bf 74.28}\%  &  8.56s & {\bf 74.28}\% &  0.08s & 73.63\% \\
\tt 20x20  full-4 &  3.96s & \ \,0.01\%  &  0.39s & \ \,0.01\%  & 21.08s & \ \,{\bf6.58}\% 
& 12.41s & \ \,{\bf 6.58}\% &  0.08s & \ \,{\bf 6.58}\%\\
\hline
\end{tabular}
\caption{Performance evaluation on random instances of~\cite{shekhovtsov-14}. For each problem family (size, type of potentials and number of labels) average performance over 100 samples is given. To allow for precise comparison all methods are initialized with the same test labeling $y$ found by LP relaxation. \IRITRWS closely approximates \IRICPLEX, which matches \Shekhovt, and scales much better. 
}
\label{table:random}
\end{table*}

%% file: tex/tables/table_average.tex
\begin{table*}[!t]
\setlength{\tabcolsep}{3pt}
\centering
\begin{tabular}{|lggg|rr|rr|rr|rr|rr|}
\hline
Problem family & \#I& \#L& \#V &
\multicolumn{2}{c|}{
\MQPBO} &
\multicolumn{2}{c|}{\MMQPBO} &
\multicolumn{2}{c|}{\Kovtun} &
\multicolumn{2}{c|}{\PBPTRWS} &
\multicolumn{2}{c|}{\IRITRWS} \\
\hline
\texttt{mrf-stereo}        &3& 16-60 &$>100000$&\multicolumn{2}{c|}{$\dagger$}& \multicolumn{2}{c|}{$\dagger$} & 
{\bf 1.0s} & 0.23\% &
2.5h & 13\% & 
{\tblue 117s} & {\tblue \textbf{73.56\%}} \\
\texttt{mrf-photomontage}  &2& 5-7 &$\leq 514080$&
93s & 22\% &
866s & 16\% &
{\bf 0.4s} & 15.94\% &
3.7h & 16\% &
{\tblue 483s} & {\tblue\textbf{41.98\%}} \\
\texttt{color-seg}         &3& 3-4 &$\leq424720$&
22s & 11\% &
87s & 16\% &
{\bf 0.3s} & 98\% &
1.3h & 99.88\% &
{\tblue 61.8s} & {\tblue \textbf{99.95\%}} \\
\texttt{color-seg-n4}      &9& 3-12 &$\leq 86400$&
22s & 8\% &
398s & 14\% &
{\bf 0.2s} & 67\% &
321s & 90\% &
{\tblue 4.9s} & {\tblue \textbf{99.26\%}} \\ 
\texttt{ProteinFolding}    &21&$\leq 483$&$\leq 1972$&
685s & 2\% &
2705s & 2\% &
{\bf 0.02s} & 4.56\% &
48s & 18\% &
{\tblue 9.2s} & {\tblue \textbf{55.70\%}}\\
\texttt{object-seg}    & 5 & 4-8 & 68160 &
3.2s & 0.01\% &
\multicolumn{2}{c|}{$\dagger$}&
{\tblue \bf 0.1s} & 93.86\% &
138s & {\tblue 98.19\%}&
2.2s & {\tblue \textbf{100\%}}\\
\hline
\end{tabular}
\caption{Performance on OpenGM benchmarks. 
Columns \#I,\#L,\#V  denote the number of instances, labels and variables respectively. For each method an average over all instances in a family is reported. $\dagger$ -- result is not available (memory / implementation limitation).
}
\label{tab:Experiments}
\end{table*}

%% file: tex/tables/table_speedups.tex
\definecolor{azure}{rgb}{0.3, 0.8, 1.0}
\definecolor{lgreen}{rgb}{0.75, 1, 0.75}
\begin{table*}[!t]
\centering
\setlength{\tabcolsep}{2pt}
\begin{tabular}{|l|G|c|c|c|c|c|}
\hline
Instance & Initialization & \multicolumn{5}{c|}{Extra time for persistency}\\
 & (1000 it.) & no speedups & \cellcolor{lgreen} +reduction & \cellcolor{lgreen!60!azure} +node pruning & \cellcolor{lgreen!30!azure} +labeling pruning & \cellcolor{azure} +fast msgs\\
\hline
Protein folding {\tt 1CKK} & 8.5s & 268s (26.53\%) & \cellcolor{lgreen} 168s (26.53\%) & \cellcolor{lgreen!60!azure} 2.0s (26.53\%) & \cellcolor{lgreen!30!azure} 2.0s (26.53\%) & \cellcolor{azure} 2.0s (26.53\%) \\
colorseg-n4 {\tt pfau-small} & 9.3s & 439s (88.59\%) & \cellcolor{lgreen} 230s (93.41\%) & \cellcolor{lgreen!60!azure} 85s (93.41\%) & \cellcolor{lgreen!30!azure} 76s (93.41\%) & \cellcolor{azure} 19s (93.41\%)\\
\hline
\end{tabular}
\caption{Evaluation of speedups on selected examples: computational time drops, as from left to right we add techniques described in \Section{sec:problemReduction}.
{\tt 1CKK}: an example when the final time for persistency is only a fraction of the initialization time.
{\tt pfau-small}: an example when times for initialization and persistency are comparable; speedups also help to improve the persistency as they are based on exact criteria.
}
\label{exp:speedups}
\end{table*}

%% file: tex/tables/table_proteins.tex
\begin{table*}[!t]
\centering
\setlength{\tabcolsep}{2pt}
\begin{tabular}{|lgg|rr|rr|rr|rr|}
\hline
Instance & \#L & \#V &
\multicolumn{2}{c|}{\PBPCPLEX} &
\multicolumn{2}{c|}{\PBPTRWS} &
\multicolumn{2}{c|}{\IRICPLEX} &
\multicolumn{2}{c|}{\IRITRWS} \\
\hline
\texttt{1CKK} & $\leq 445$ & 38 &
2503s & 0\%&
46s & 0\%&
2758s & \bf \tblue 27\% &
\tblue 8.5+2s & 26.53\% \\
\texttt{1CM1} &$\leq 350$&37&
2388s & 0\% &
51s & 0\% &
4070s & \bf 34\% & 
\tblue 9+3.9s & \tblue 29.97\% \\
\texttt{1SY9} &$ \leq 425$&37&
1067s & 0\% &
67s & 0\% &
2629s & 51\% &
\tblue 11+4.2s & \bf \tblue 57.98\% \\
\texttt{2BBN} &$\leq 404$&37&
9777s & 0\% &
5421s & 0\% &
9677s & 9\% &
\tblue 16+4.3s & \bf \tblue 14.17\%\\
\texttt{PDB1B25} & $ \leq 81 $  &  1972 &
 325s & 22\% &
 120s & 22\% &
 1599s & 84\% &
\tblue 4.3+7.3s & \bf \tblue 87.84\% \\
\texttt{PDB1D2E} & $ \leq 81 $  &  1328 &
 483s & 59\% &
 83s & 59\% &
 154s & \bf 98\% &
\tblue 1.6+1.8s & \bf \tblue 98.25\% \\
\hline
\end{tabular}
\caption{Comparison to~\cite{SwobodaPersistencyCVPR2014} using exact and approximate LP solvers. Examples of hard {\tt ProteinFolding} instances~\cite{PIC2011,SideChainPredictionYanover}. For \IRITRWS the initialization + persistency time is given. An occasionally better persistency of \IRITRWS \vs\ \IRICPLEX is explained by different test labelings produced by the CPLEX and TRW-S solvers
(unlike in Table~\ref{table:random}). The results of \Shekhovt wold be identical to \IRICPLEX, as has been proven and verified on random instances.
}
\label{tab:proteins}
\end{table*}

%% file: tex/conclusion.tex
\section{Conclusions and Outlook}\label{sec:conclusions}

We presented an approach to find persistencies for an exp-APX-complete problem employing only solvers for a convex relaxation. Using a suboptimal solver for the relaxed problem, we still correctly identify persistencies while the whole approach becomes scalable.
Our method with an exact solver matches the maximum persistency~\cite{shekhovtsov-14} and with a suboptimal solver closely approximates it, outperforming state of the art persistency techniques~\cite{SwobodaPersistencyCVPR2014,PartialOptimalityInMultiLabelMRFsKohli,Kovtun-10}.
The speedups we have developed allow to achieve this at a reasonable computational cost making the method much more practical than the works~\cite{shekhovtsov-14,SwobodaPersistencyCVPR2014} we build on.
In fact, our approach takes an approximate solver, like TRW-S, and turns it into a method with partial optimality guarantees at a reasonable computational overhead.
\par
We believe that many of the presented results can be extended to higher order graphical models and tighter relaxations. Practical applicability with other approximate solvers can be explored.
A further research direction that seems promising is mixing different optimization strategies such as persistency and cutting plane methods.

%% file: tex/acknowlegement.tex
\section*{Acknowlegement}
\small
Alexander Shekhovtsov was supported by the Austrian Science Fund (FWF) under the START project BIVISION, No. Y729.
Paul Swoboda and Bogdan Savchynskyy were supported by the German Research Foundation (DFG) within the program ``Spatio-/Temporal Graphical Models and Applications in Image Analysis'', grant GRK 1653. Bogdan Savchynskyy was also supported by European Research Council (ERC) under the European Union’s Horizon 2020 research and innovation program (grant agreement No 647769).

%% file: tex/appendix.tex
\appendices
\addtocontents{toc}{\protect\setcounter{tocdepth}{2}}
\pagestyle{plain}
\section{Proofs}
\subsection*{Proofs of the Generic Algorithms}\label{sec:A-main}
\SIcomponents*
\begin{proof}
Direction $\Rightarrow$.
Let $p\in\BS_f$. Assume for contradiction that $(\exists v\in\SV\ \ \exists i\in\O^*_v)\ \ p_v(i) \neq i$. Since $i\in\O^*_v$ there exists $\mu\in\O^*$ such that $\mu_{v}(i)>0$. Its image $\mu' = [p]\mu$ has $\mu_{v}(i) = 0$ due to $p_v(i) \neq i$ by evaluating the extension~\eqref{P-transpose}. This contradicts $[p]\mu=\mu$.
\par
Direction $\Leftarrow$.
Now let $(\forall v\in\SV\ \ \forall i\in\O^*_v)\ \ p_v(i) = i$. Clearly, $[p]\mu = \mu$ holds for all $\mu$ on the support set given by $(\O_v^* \mid v\in\V)$, hence for $\O^*$. It remains to show that the value of the minimum in~\eqref{equ:linearStrictLPImprovingMapping} is zero. For $\mu\in\O^*$ we have $[p]\mu=\mu$ and the objective in~\eqref{equ:linearStrictLPImprovingMapping}, $\<(I-[p])\T f,\mu \> = \<f, \mu-[p]\mu \>$ vanishes.
\end{proof}

%

\TmaxMapping*
\begin{proof}
The two following lemmas form a basis for the proof.
\begin{lemma}[\cite{shekhovtsov-14-TR}, Thm. 3(b)]\label{A1-C1-0}
Let $q\in \BS_ f$; $q\leq p$ and let $\O^* = \argmin_{\mu\in\Lambda} \< (I-[p])^\top f, \mu\>$, \ie, as in line~\ref{alg1:test-p} of Algorithm~\ref{alg:iterative-LP}. \\Then
$
{(\forall \mu\in \O^*)\ \ [q] \mu = \mu}
$.
\end{lemma}
\par
In the case of substitutions from the class $\SP^{2,y}$, the statement additionally simplifies as follows.
\begin{corollary}\label{A1-C1} Assume conditions of \autoref{A1-C1-0} and additionally, let 
$q \in \SP^{2,y}$ and $\O^*_v := \{i\in\X_v \mid (\exists \mu\in \O^*)\ \mu_v(i)>0 \}$ then 
\begin{equation}
(\forall v\in\V, \forall i\in \O^*_v)\ \ q_v(i) = i.
\end{equation}
\end{corollary}
\begin{proof}
It follows similarly to \autoref{prop:linearStrictLPImprovingMapping}. Assume for contradiction that $(\exists v\in\SV\ \ \exists i\in\O^*_v)\ \ q_v(i) \neq i$. Since $i\in\O^*_v$ there exists $\mu\in\O^*$ such that $\mu_{v}(i)>0$. It's image $\mu' = [q]\mu$ has $\mu_{v}(i) = 0$ due to $q_v(i) \neq i$ by evaluating the extension~\eqref{P-transpose}. This contradicts to $[q]\mu=\mu$, the statement of \autoref{A1-C1-0}. 
\end{proof}
\begin{lemma}\label{A1-P2} Let $p^t$ denote the substitution $p$ computed in line~\ref{alg1:init-p} of \Algorithm{alg:iterative-LP} on iteration $t$. The algorithm maintains the invariant that $(\forall q\in \BS_ f \cap\SP^{2,y})$\ \ $q \leq p^t$. 
\end{lemma}
\begin{proof}
We prove by induction. The statement holds trivially for the first iteration. Assume it is true for the current iteration $t$. Then for any $q\in \BS_ f \cap \SP^{2,y}$ holds $q \leq p^t$ and therefore \autoref{A1-C1} applies. 
We can show that line~\ref{alg1:update-U} only prunes substitutions that are not in $\BS_ f \cap \SP^{2,y}$ as follows.

Let $p^{t+1}$ be the substitution on the next iteration, \ie computed by line~\ref{alg1:init-p} after pruning line~\ref{alg1:update-U}. 

Assume for contradiction that $\exists q \in \BS_ f \cap \SP^{2,y}$ such that $q \not\leq p^{t+1}$. By negating the definition and expanding,
\begin{subequations}
\begin{align}
(\exists v\in\SV)\ \ & p^{t+1}_v(\X_v) \not\subseteq q_v(\X_v), \\
\Leftrightarrow  (\exists v\in\SV\ \exists i\in \X_v)\ \ & i\in p^{t+1}_v(\X_v) \ \ \wedge \ \ i\not\in q_v(\X_v),\\
\label{A1-L2-eq1b}
\Leftrightarrow (\exists v\in\SV\ \exists i\in \X_v)\ \ & p^{t+1}_v(i)=i \ \ \wedge \ \ q_v(i)\neq i.
\end{align}
\end{subequations}
If $i$ was pruned in line~\ref{alg1:update-U}, $i\in\O_v^*$, then it must be that $q_v(i) = i$, which contradicts to~\eqref{A1-L2-eq1b}. Therefore 
\begin{equation}
(\exists v\in\SV\ \exists i\in \X_v\backslash \O_v^*)\ \ p^{t+1}_v(i)=i \ \ \wedge \ \ q_v(i)\neq i.
\end{equation}
However, in this case $p^{t+1}_v(i) = p^{t}_v(i) = i$ and $q \leq p^t$ fails to hold, which contradicts to the assumption of induction. Therefore $q \leq p^{t+1}$ holds by induction on every iteration.
\end{proof}
By \Proposition{A1-P1} the algorithm terminates and returns a substitution in $\BS_ f \cap \SP^{2,y}$. By Lemma~\ref{A1-P2} the returned substitution $p$ satisfies $p \geq q$ for all $q \in \BS_ f\cap \SP^{2,y}$. 
It is the maximum.
\end{proof}

\propstrictCSisAC*
\begin{proof}
 Condition $\O_v(\varphi) = \O_v^*$ implies that $\varphi$ satisfies strict complementarity with some primal optimal solution $\mu$. 
The strict complementarity implies that $(\forall i\in\X_v)$ $( f^{\varphi}_{u}(i) = 0 \Rightarrow \mu_{u}(i) >0)$. By feasibility of $\mu$, there must hold $(\forall {v}\in\N({u}))\ (\exists j\in\X_{v})\ \mu_{uv}(i,j)>0$. And by using complementary slackness again, it must be that $ f^{\varphi}_{uv}(i,j) = 0$. Similarly, the second condition of arc consistency is verified. It follows that $f^{\varphi}$ is arc consistent. 
\end{proof}
\input{tex/reduction_proof2.tex}
\input{tex/ac_theorems.tex}
\input{tex/trws_message_passing.tex}

\clearpage
\setlength{\tabcolsep}{5pt}
\input{tex/experiments_generated_supplement2.tex}

%% file: tex/reduction_proof2.tex
\subsection*{\vskip-2em\noindent Proofs of the Reduction}\label{sec:reduction-p2}
The proof of the reduction \Theorem{T:reduction 1-2} and \Lemma{lemma:LPrefinement} (used in speed-up heuristics) requires several intermediate results.
Recall that a correct pruning can be done when we have a guarantee to preserve all strictly improving substitutions $q$, assuming $q \leq p$. Therefore statements in this section are formulated for such pairs.
We will consider adjustments to the cost vector that preserve the set of strictly improving substitutions. These adjustments do not in general preserve optimal solutions to the associated LP relaxation.
\par
\begin{lemma}\label{T:q-B-f-g}
Let $q\leq p$. Then $q\in\BS_f$ iff $q\in\BS_g$ for $g = (I-[p])\T f$.
\end{lemma}
\begin{proof}
Let $Q = [q]$, $P = [p]$. Since $q\leq p$ there holds $PQ = P$. It implies $(I-P)(I-Q) = (I-Q)$.
Therefore, 
\begin{align}\label{q-BSfg}
&\<g,(I-Q)\mu\> = \<(I-P)\T f,(I-Q)\mu\> \\
&= \<f,(I-P)(I-Q)\mu\> = \<f,(I-Q)\mu\>.
\end{align}
Assume $\mu\in\Lambda$ is such that $Q \mu \neq \mu$. Equality~\eqref{q-BSfg} ensures that $\<g,(I-Q)\mu\> > 0$ iff $\<f,(I-Q)\mu\> > 0$. The theorem follows from definition of $\BS_f$, $\BS_g$.
\end{proof}
To reformulate the condition $q \in \BS_f$ we will use the following dual characterization.
\begin{theorem}[Characterization~\cite{shekhovtsov-14-TR}]\label{T:char}
Let $P = [p]$. Then 
\begin{align}
(\forall \mu\in\Lambda)\ \<f, P \mu \> \leq \<f,\mu\>
\end{align}
iff there exists a reparametrization $\varphi$ such that
\begin{align}\label{dual-improve-comp}
P\T f^\varphi \leq f^\varphi.
\end{align}
\end{theorem}
The following lemma assumes an arbitrary substitution $q$, not necessarily in $\SP^{2,y}$ and takes as input sets $U_u$ that are subsets of {\em immovable} labels. In the context of~\Theorem{T:reduction 1-2}, we will use $U_u = \X_u \backslash \Y_u$.
\begin{lemma}[Reduction 1]\label{T:reduction 1}
For a substitution $q$ let $U_u \subset \{i\in\X_u \mid q_u(i) = i\}$ for all $u\in\V$. Let $g_{uv}(i,j) = 0$ for all $(i,j) \in U_u\times U_v$ and let $\bar g$ be defined by \begin{subequations}\label{equ:reductionTheorem-copy}
 \begin{align}
	&\bar g_v = g_v,\ v \in\SV;\\
	\label{bar-g-uv}
	&\overline g_{uv}(i,j)= \left\{
	\setlength{\arraycolsep}{0pt}
  \begin{array}{ll}
   \min\limits_{i' \in U_u} g_{uv}(i',j), &\ \ i\in U_u, j\notin U_v\,,\\
   \min\limits_{j' \in U_v} g_{uv}(i,j'), &\ \ i\notin U_u, j \in U_v\,,\\
	 g_{uv}(i,j), &\ \ \mbox{otherwise}.
  \end{array}
\right.
 \end{align}
\end{subequations}
Then $q\in\BS_g$ iff $q\in \BS_{\bar g}$.
\end{lemma}
\begin{proof}

\par
Direction $\Leftarrow$. Let us verify the following inequality:
\begin{align}
\notag
(\forall ij\in\X_{uv})\ & g_{uv}(q_{u}(i),q_v(j))-\bar g_{uv}(q_{u}(i),q_v(j))\\
\label{bar g aux}
 & \leq g_{uv}(i,j)-\bar g_{uv}(i,j).
\end{align}
We need to consider only cases where $\bar g_{uv}(i,j) \neq g_{uv}(i,j)$. Let $i\in U_u$ and $j\notin U_v$ (the remaining case is symmetric). In this case $q_u(i)=i$. Substituting $\bar g$ we have to prove
\begin{align}
& g_{uv}(i,q_v(j)) - \min\limits_{i' \in U_u} g_{uv}(i',q_v(j)) \\
\notag
 \leq & g_{uv}(i,j) - \min\limits_{i' \in U_u} g_{uv}(i',j).
\end{align}
The left hand side is zero because all the respective components of $g$ are zero by assumption. At the same time the right hand side is non-negative since $i\in U_u$. The inequality~\eqref{bar g aux} holds. 
It implies (by multiplication with pairwise components of $\mu$ and using the equality of unary components of $g$ and $\bar g$) that
\begin{align}\label{bar g aux Q}
(\forall \mu\in\Lambda)\ \<g, Q \mu\> - \<\bar g, Q \mu\> \leq \<g, \mu\> - \<\bar g, \mu\>,
\end{align}
where $Q = [q]$. Note, cost vector $\bar g$ satisfying~\eqref{bar g aux Q} is called {\em auxiliary} for $g$ in~\cite{Kovtun-10,shekhovtsov-phd}. Inequality~\eqref{bar g aux Q} is equivalent to
\begin{align}\label{bar g aux mu}
\<\bar g, (I-Q) \mu\> \leq  \<g, (I - Q) \mu\>.
\end{align}
Whenever the left hand side of~\eqref{bar g aux mu} is strictly positive then so is the right hand side and therefore 
from $q\in\BS_{\bar g}$ follows $q \in \BS_{g}$.
\par
Direction $\Rightarrow$. Assume $q \in \BS_{g}$. 
By \Theorem{T:char}, there exist dual multipliers $\varphi$ such that 
$g^\varphi$ verifies inequality~\eqref{dual-improve-comp}, in components: 
\begin{align}\label{g' components}
&(\forall u\in\V,\ \forall i\in\X_{u})\ g^\varphi_u(q_u(i)) \leq g^\varphi_u(i);\\
\notag
&(\forall uv\in\SE,\ \forall ij\in\X_{uv})\ g^\varphi_{uv}(q_u(i),q_v(j)) \leq g^\varphi_{uv}(i,j).
\end{align}
Let us expand the pairwise inequality in the case $i\in U_u$, $j\notin U_v$. Let $q_v(j) = j^*$. Using $q_u(i)=i$ we obtain
\begin{align}
\notag
g_{uv}(i,j^*) - \varphi_{uv}(i)-\varphi_{vu}(j^*) 
& \leq g_{uv}(i,j) - \varphi_{uv}(i)-\varphi_{vu}(j);\\
g_{uv}(i,j^*) -\varphi_{vu}(j^*) & \leq g_{uv}(i,j) -\varphi_{vu}(j).
\end{align}
We take $\min$ over $i\in U_u$ of both sides:
\begin{align}
\min_{i\in U_u} g_{uv}(i,j^*) -\varphi_{vu}(j^*) \leq \min_{i\in U_u} g_{uv}(i,j) -\varphi_{vu}(j).
\end{align}
Finally we subtract $\varphi_{uv}(i)$ on both sides and obtain
\begin{align}
\bar g^\varphi_{uv}(i,j^*) \leq \bar g^\varphi_{uv}(i,j).
\end{align}
The case when $i\notin U_u$, $j\in U_v$ is symmetric. In the remaining cases, $\bar g^\varphi_{uv}(i,j) = \bar g(i,j) - \varphi_{uv}(i)-\varphi_{vu}(j) = g(i,j) - \varphi_{uv}(i)-\varphi_{vu}(j) = g^\varphi(i,j)$. In total, $\bar g^\varphi $ satisfies all component-wise inequalities that does $g^\varphi$ in~\eqref{g' components}. By \Theorem{T:char},
\begin{align}\label{bar g weak improving}
(\forall \mu\in\Lambda)\ \<\bar g, Q \mu\> \leq \<\bar g, \mu\>.
\end{align}
We have shown that $\<\bar g, (I-Q) \mu\> \geq 0$. It remains to prove that the inequality holds strictly when $Q \mu \neq \mu$. 
Since $q\in\BS_g$, there holds $\<g,\mu\> < \<Q\T g,\mu\>$. It is necessary that at least one of unary or pairwise inequalities~\eqref{g' components} from the support of $\mu$ holds strictly in which case inequality~\eqref{bar g weak improving} is also strict.
\end{proof}
\begin{lemma}[Reduction 2]\label{T:reduction 2}
For a substitution $q$ and cost vector $g$ let $\bar g = g - \Delta^+$, where $\Delta^+ \in\Real^\SI_+$  has zero unary components and its pairwise components read:
\begin{align}
\Delta^+_{uv}(i,j) = \max\big\{0,\ & g_{uv}(i,j)+ g_{uv}(q_u(i),q_v(j))\\
\notag
& - g_{uv}(i,q_v(j)) - g_{uv}(q_u(i),j)\big\}.
\end{align}
Then $q\in\BS_g$ iff $q\in \BS_{\bar g}$.
\end{lemma}
\begin{proof}
The scheme of the proof is similar to~\autoref{T:reduction 1}. The unary components of $g$ and $\bar g$ are equal. If we show inequality~\eqref{bar g aux}, the implication $q\in\BS_{\bar g}$ $\Rightarrow$ $q \in \BS_{g}$ will follow as in~\autoref{T:reduction 1}.
For our $\bar g$, inequality~\eqref{bar g aux} reduces to 
\begin{align}
\Delta^+_{uv}(q_u(i),q_v(j)) \leq \Delta^+_{uv}(i,j)
\end{align}
and due to idempotency of $q$ the left hand side is identically zero. Therefore inequality~\eqref{bar g aux} is verified.
\par
Direction $\Rightarrow$. Assume $q \in \BS_{g}$. By \Theorem{T:char}, there exist dual multipliers $\varphi$ satisfying inequalities~\eqref{g' components}.
Consider
\begin{align}
\bar g^\varphi = (g - \Delta^+)^\varphi = g^\varphi - \Delta^+.
\end{align}
Let us show that component-wise inequalities~\eqref{g' components} hold for $\bar g^\varphi$. Clearly they hold for unary components and for pairwise components where $\Delta^+_{uv}(i,j) = 0$. Let $uv\in \SE$ and $\Delta^+_{uv}(i,j) > 0$.
Let $i'=q_u(i)$ and $j'=q_u(j)$. It must be that $i' \neq i$ and $j' \neq j$. Let us denote $a=g^\varphi_{uv}(i',j')$, $b = g^\varphi_{uv}(i',j)$, $c = g^\varphi_{uv}(i,j')$ and $d = g^\varphi_{uv}(i,j)$. 
\par
Let $\bar d := g^\varphi_{uv}(i,j) - \Delta^+_{uv}(i,j) = d - (a+d-b-c) = b+c-a$.
From~\eqref{g' components} we have that $a \leq b,c,d$. It follows that $2 a \leq b+c$ or $a \leq b+c - a = \bar d$. We proved that $\bar g^\varphi_{uv}(q_u(i),q_v(j)) \leq \bar g^\varphi_{uv}(i,j)$. 
In total, $\bar g^\varphi$ satisfies all component-wise inequalities, same as $g^\varphi$ in~\eqref{g' components}. By \Theorem{T:char},
it follows that $\<\bar g, (I-Q) \mu\> \geq 0$. The strict inequality in case $Q\mu\neq \mu$ is considered similarly to~\autoref{T:reduction 1}.
\end{proof}
\TReduction*
\begin{proof}
Let $g = (I-P)\T f$. By \Theorem{T:q-B-f-g}, $q\in \BS_f$ iff $q\in \BS_{g}$.
We need to consider only pairwise terms. Let $uv\in\SE$. Since $q\leq p$, if $p_u(i) = i$ then necessarily $q_u(i) = i$. Let $p$ be defined using sets $\Y_u$ as in~\eqref{equ:pixelWiseMapping}. The reduction $\bar g$ in~\eqref{equ:reductionTheorem} will be composed of reductions by \Theorem{T:reduction 1} and \Theorem{T:reduction 2}.
\par
From $g = (I-P)\T f$ we have that for $i\in\X_u\backslash \Y_u$ and $j\in\X_v\backslash \Y_v$ $g_{uv}(i,j) = 0$. Conditions of~\Theorem{T:reduction 1} are satisfied with $U_u = \X_u\backslash \Y_u$. We obtain part of the reduction~\eqref{equ:reductionTheorem} for cases when $i\notin \Y_u$ or $j\notin\Y_v$. Let us denote the reduced vector $\bar g'$.
Applying~\Theorem{T:reduction 2} to it, we obtain $\bar g$ as defined in~\eqref{equ:reductionTheorem}.
\end{proof}
\par
\LPrefinement*
\begin{proof}
Let $\mu\in\O^*$. Assume for contradiction that $Q \mu \neq \mu$. In this case, by \Theorem{T:reduction 1-2}, we have that
$\<\bar g, Q \mu \> < \<\bar g, \mu \>$. Since $\mu\in\Lambda'$ and $Q(\Lambda')\subset \Lambda$ there holds $Q\mu\in\Lambda'$. It follows that $Q \mu$ is a feasible solution of a better cost than $\mu$ which contradicts optimality of $\mu$. It must be therefore that $Q \mu = \mu$. The claim $q(\O^*_v) = O^*_v$ follows.
\end{proof}

%% file: tex/ac_theorems.tex
\subsection*{\vskip-2em\noindent Termination with arc consistent Solvers}\label{sec:ac-theorems}
\begin{theorem}\label{T:AC-terminates}
Consider the verification LP defined by $g = (I-P\T)f$. Let $g^{\varphi}$ be an arc-consistent reparametrization and let $\Y_u = \{i \mid p(i)\neq i\}$. Then at least one of the two conditions is satisfied:
\begin{itemize}
\item[(a)] $g^\varphi_\emptyset = 0$ and $\varphi$ is dual optimal;
\item[(b)] $(\exists u\in\V)$ \ \ $\O_u(\varphi) \cap \Y_u \neq \emptyset$.
\end{itemize}
\end{theorem}
\begin{proof}
Assume (b) does not hold: $(\forall u\in\V)$ \ \ $\O_u(\varphi) \subseteq \X_u \backslash \Y_u$. For each node $u$ let us chose a label $z_u\in\O_u(\varphi)$. By arc consistency, for each edge $uv$ there is a label $j\in \O_v(\varphi) \subset \X_v \backslash \Y_v$ such that  $g^{\varphi}_{uv}(z_u,j)$ is active and similarly, there exists $i\in \O_u(\varphi) \subset \X_u \backslash \Y_u$ such that $g^{\varphi}_{uv}(i,z_v)$ is active.
\par
By construction, $g_{uv}(i',j') = 0$ for all $i'j'\in \X_{uv}\backslash \Y_{uv}$ and therefore the following modularity equality holds:
\begin{align}\label{ac-U-eq1}
&g^{\varphi}_{uv}(z_u,z_v) + g^{\varphi}_{uv}(i,j) \\
\notag
 =&\big(0-\varphi_{uv}(z_u)-\varphi_{uv}(z_v)) +\big(0-\varphi_{uv}(i)-\varphi_{uv}(j)\big)\\
\notag
=&g^{\varphi}_{uv}(z_u,j)+g^{\varphi}_{uv}(i,z_v).
\end{align}
From $g^{\varphi}_{uv}(z_u,j)$ being active we have
\begin{equation}\label{ac-U-eq2}
g^{\varphi}_{uv}(z_u,j) \leq g^{\varphi}_{uv}(i,j).
\end{equation}
By adding~\eqref{ac-U-eq1} and \eqref{ac-U-eq2} we obtain
$g^{\varphi}_{uv}(z_u,z_v)  \leq g^{\varphi}_{uv}(i,z_v)$
and hence $(z_u,z_v)$ is active. Therefore $\delta(z)$ and dual point $\varphi$ satisfy complementarity slackness and hence they are primal-dual optimal and $g^\varphi_\emptyset = E_{g}(z) = 0$.
\end{proof}
\PdualstopAC*
\begin{proof}
Corollary from Theorem~\ref{T:AC-terminates}.
\end{proof}

%% file: tex/trws_message_passing.tex
\subsection*{\vskip-2em\noindent Fast Message Passing}\label{sec:message-passing}
\begin{table}[t]
\resizebox{\linewidth}{!}{%
\small
\setlength{\tabcolsep}{4pt}
\tabulinesep=0.5mm
\begin{tabu}{|c|l|}
\hline
$i \notin \Y_u$ & $\bar g_u(i) = 0;$\\
$i \in \Y_u$ & $\bar g_u(i) = f_u(i)-f_u(y_u);$\\
\hline
$i{\notin}\Y_u$, $j{\notin}\Y_v$ & $\bar g_{uv}(i,j) = 0$; $\Delta_{uv}(i) := \Delta_{vu}(j) := 0$;
\\
\hline
$i{\notin}\Y_u$, $j{\in}\Y_v$ & 
$\begin{aligned}
& \textstyle \Delta_{vu}(j) := \min_{i'\notin \Y_u} \big[ f_{uv}(i',j) - f_{uv}(i',y_v)  \big],\\
& \textstyle \bar g_{uv}(i,j) = \Delta_{vu}(j);
\end{aligned}$\\
\hline
$i{\in}\Y_u$, $j{\notin}\Y_v$ & 
$\begin{aligned}
& \textstyle \Delta_{uv}(i) := \min_{j'\notin \Y_v} \big[ f_{uv}(i,j') - f_{uv}(y_u,j')  \big],\\
& \textstyle \bar g_{uv}(i,j) = \Delta_{uv}(i);
\end{aligned} $\\
\hline
$i{\in}\Y_u$, $j{\in}\Y_v$ &
$\begin{aligned}
\textstyle \bar g_{uv}(i,j) = \min \big\{ & \textstyle  f_{uv}(i,j)-f_{uv}(y_u,y_v),\\
& \textstyle  \Delta_{vu}(j)+\Delta_{uv}(i) \big\}.
\end{aligned}$\\
\hline
\end{tabu}
}
\caption{Components of the Reduced Verification Problem}
\label{table:bar g}
\end{table}
\TFastMsg*
\begin{proof}
The components of the reduced problem $\bar g$~\eqref{equ:reductionTheorem} can be expressed directly in components of $f$ as in Table~\ref{table:bar g}. Passing a message on edge $uv$ amounts to calculating 
$\varphi_{vu}(j) := \min_{i\in \X_u}\big[ a(i) +\bar g_{uv}(i,j) \big]$
for some vector $a \in\Real^{\X_u}$.
For $j\notin \Y_v$, substituting pairwise terms of $\bar g$, it expands as
\begin{align}
\label{msg-all->U}
\textstyle \varphi_{vu}(j) := \min_{i\in \X_u}\big[a(i)+\Delta_{uv}(i)\big]. 
\end{align}
Since the message is equal for all $j\notin \Y_v$, it is sufficient to represent it by $\varphi_{vu}(y_v)$ (recall that $y_v\notin \Y_v$).
For $j\in \Y_v$, substituting pairwise terms of $\bar g$ and denoting $c =  f_{uv}(y_u,y_v)$, 
\begin{align}\label{msg-all->notU}
\textstyle
&\textstyle \varphi_{vu}(j) := \min\big\{\min_{i\notin \Y_u} a(i) + \Delta_{vu}(j),\\
\notag
&\textstyle \ \ \min_{i\in \Y_u}\big[ a(i) +\min\big\{f_{uv}(i,j)-c, \Delta_{uv}(i)+\Delta_{vu}(j)\big\} \big] \big\}
\end{align}
\vskip-1.5em
\begin{subequations}
\begin{align}
 {=}\min\big\{
 & \label{phi-min1} \textstyle \min_{i\notin \Y_u} a(i) + \Delta_{uv}(j),\\
 & \label{phi-min2} \textstyle \min_{i\in \Y_u}\big[ a(i){+}f_{uv}(i,j) \big] - c,\\
 & \label{phi-min3} \textstyle \min_{i\in \Y_u}\big[ a(i) + \Delta_{uv}(i)\big] + \Delta_{vu}(j) \big\}
\end{align}
\end{subequations}
Adding $\Delta_{uv}(i)$ inside~\eqref{phi-min1} (it is zero for $i \notin \Y_u$) and grouping~\eqref{phi-min1} and \eqref{phi-min3} together, we obtain for $j\in\Y_v$, $\varphi_{vu}(j) = $
\begin{align}\label{msg-part12}
\textstyle
\min \big\{ \min\limits_{i\in \Y_u} \big[ a(i){+}f_{uv}(i,j) \big] - c, \varphi_{uv}(y_v){+}\Delta_{vu}(j) \big\}.
\end{align}
Expression~\eqref{phi-min2} is a message passing for $f$, but the minimum is only over $\Y_u$ and the result is needed only for $j\in\Y_v$. This message can be computed in time $O(|\Y_u| + |\Y_v|)$ using the same algorithms~\cite{Hirata-1996,Meijster2000,Aggarwal-1987} (see also non-uniform min-convolution in~\cite{Zach-14}). Evaluating~\eqref{msg-part12} takes additional $O(|\Y_v|)$ time and minimum in~\eqref{msg-all->U} takes $O(|\Y_u|)$ time assuming that components of $a(i)$ are equal for $j\notin \Y_v$ (because it is already true for $\bar g$ and $\varphi$). 
\end{proof}

%% file: tex/experiments_generated_supplement2.tex
\onecolumn
{
\centering
\section{Results Per Instance}\label{sec:results-per-instance}
\footnotesize
\begin{longtable}{lllllllllll}
\toprule
Instance &
Algorithm & 
\parbox{1.5cm}{Time\\needed\\overall (s)} & \parbox{1.5cm}{Time for\\initial\\ solution (s)} & \parbox{1.5cm}{\#iterations\\Algorithm~\ref{alg:iterative-LP},\ref{alg:iterative-DualLP}} & \parbox{1.5cm}{\#iterations\\TRWS} & \parbox{1.5cm}{Logarithmic\\percentage\\partial\\optimality} & \parbox{1.5cm}{Percentage\\excluded\\labels}\\
\hline
\endhead%

\hline
 \multicolumn{ 8 }{c}{ ProteinFolding } \\\hline

\hline
1CKK & \IRICPLEX & 2757.62 & 1177.62 & 5 & $\dagger$ & 14.24\% & 27.04\%\\
 & \IRITRWS & 5.76 & 5.00 & 3 & 1000+15 & 13.83\% & 26.53\%\\
 & \MMQPBO & 5670.00 & 0.00 & 0 & $\dagger$ & 0.00\% & 0.00\%\\
 & \MQPBO & 825.00 & 0.00 & 0 & $\dagger$ & 0.00\% & 0.00\%\\
 & \PBPCPLEX & 2502.69 & 2493.65 & 1 & $\dagger$ & 0.00\% & 0.00\%\\
 & \PBPTRWS & 47.57 & 30.19 & 2 & 288+185 & 0.00\% & 0.00\%\\

\hline
1CM1 & \IRICPLEX & 4070.00 & 992.15 & 7 & $\dagger$ & 8.38\% & 34.28\%\\
 & \IRITRWS & 6.03 & 4.70 & 4 & 1000+65 & 8.07\% & 29.98\%\\
 & \MMQPBO & 5520.00 & 0.00 & 0 & $\dagger$ & 0.00\% & 0.00\%\\
 & \MQPBO & 723.00 & 0.00 & 0 & $\dagger$ & 0.00\% & 0.00\%\\
 & \PBPCPLEX & 2388.46 & 2198.04 & 3 & $\dagger$ & 0.00\% & 0.00\%\\
 & \PBPTRWS & 51.33 & 21.60 & 3 & 242+358 & 0.00\% & 0.00\%\\

\hline
1SY9 & \IRICPLEX & 2628.72 & 416.74 & 5 & $\dagger$ & 25.34\% & 51.30\%\\
 & \IRITRWS & 6.88 & 5.50 & 4 & 1000+15 & 28.06\% & 57.98\%\\
 & \MMQPBO & 7494.00 & 0.00 & 0 & $\dagger$ & 0.00\% & 0.00\%\\
 & \MQPBO & 2112.00 & 0.00 & 0 & $\dagger$ & 0.00\% & 0.11\%\\
 & \PBPCPLEX & 1067.46 & 910.90 & 4 & $\dagger$ & 0.00\% & 0.00\%\\
 & \PBPTRWS & 66.73 & 46.77 & 5 & 400+174 & 0.00\% & 0.00\%\\

\hline
2BBN & \IRICPLEX & 9677.42 & 5476.81 & 5 & $\dagger$ & 2.12\% & 8.58\%\\
 & \IRITRWS & 10.00 & 8.60 & 3 & 1000+10 & 2.64\% & 14.17\%\\
 & \MMQPBO & 1736.00 & 0.00 & 0 & $\dagger$ & 0.00\% & 0.00\%\\
 & \MQPBO & 2429.00 & 0.00 & 0 & $\dagger$ & 0.00\% & 0.00\%\\
 & \PBPCPLEX & 9776.60 & 9771.18 & 1 & $\dagger$ & 0.00\% & 0.00\%\\
 & \PBPTRWS & 54.21 & 42.90 & 2 & 242+146 & 0.00\% & 0.00\%\\

\hline
2BCX & \IRICPLEX & 36222.90 & 6998.66 & 5 & $\dagger$ & 4.81\% & 15.66\%\\
 & \IRITRWS & 9.14 & 7.90 & 3 & 1000+55 & 4.39\% & 14.21\%\\
 & \MMQPBO & 1008.00 & 0.00 & 0 & $\dagger$ & 0.00\% & 0.00\%\\
 & \MQPBO & 1288.00 & 0.00 & 0 & $\dagger$ & 0.00\% & 0.00\%\\
 & \PBPCPLEX & 11419.60 & 11409.90 & 2 & $\dagger$ & 0.00\% & 0.00\%\\
 & \PBPTRWS & 55.26 & 39.60 & 2 & 252+194 & 0.00\% & 0.00\%\\

\hline
2BE6 & \IRICPLEX & 1381.60 & 765.84 & 4 & $\dagger$ & 9.14\% & 17.68\%\\
 & \IRITRWS & 4.67 & 3.91 & 4 & 1000+60 & 8.96\% & 15.12\%\\
 & \MMQPBO & 3728.00 & 0.00 & 0 & $\dagger$ & 0.00\% & 0.05\%\\
 & \MQPBO & 540.00 & 0.00 & 0 & $\dagger$ & 0.00\% & 0.00\%\\
 & \PBPCPLEX & 1552.95 & 1552.88 & 1 & $\dagger$ & 0.00\% & 0.00\%\\
 & \PBPTRWS & 40.12 & 28.12 & 2 & 363+230 & 0.00\% & 0.00\%\\

\hline
2F3Y & \IRICPLEX & 3628.90 & 2546.68 & 5 & $\dagger$ & 6.22\% & 10.66\%\\
 & \IRITRWS & 5.83 & 5.20 & 3 & 1000+10 & 8.39\% & 13.74\%\\
 & \MMQPBO & 5138.00 & 0.00 & 0 & $\dagger$ & 0.00\% & 0.00\%\\
 & \MQPBO & 928.00 & 0.00 & 0 & $\dagger$ & 0.00\% & 0.05\%\\
 & \PBPCPLEX & 4618.78 & 4618.76 & 1 & $\dagger$ & 0.00\% & 0.00\%\\
 & \PBPTRWS & 41.87 & 33.03 & 3 & 321+164 & 0.00\% & 0.00\%\\

\hline
2FOT & \IRICPLEX & 7458.75 & 1996.55 & 5 & $\dagger$ & 4.10\% & 11.64\%\\
 & \IRITRWS & 6.25 & 5.30 & 4 & 1000+25 & 4.01\% & 11.01\%\\
 & \MMQPBO & 4961.00 & 0.00 & 0 & $\dagger$ & 0.00\% & 0.09\%\\
 & \MQPBO & 1054.00 & 0.00 & 0 & $\dagger$ & 0.00\% & 0.07\%\\
 & \PBPCPLEX & 4473.58 & 4440.51 & 1 & $\dagger$ & 0.00\% & 0.00\%\\
 & \PBPTRWS & 61.92 & 44.42 & 2 & 398+222 & 0.00\% & 0.00\%\\

\hline
2HQW & \IRICPLEX & 5721.95 & 1946.20 & 6 & $\dagger$ & 10.30\% & 17.30\%\\
 & \IRITRWS & 6.49 & 4.80 & 6 & 1000+160 & 8.33\% & 18.08\%\\
 & \MMQPBO & 7228.00 & 0.00 & 0 & $\dagger$ & 0.00\% & 0.00\%\\
 & \MQPBO & 1193.00 & 0.00 & 0 & $\dagger$ & 0.00\% & 0.00\%\\
 & \PBPCPLEX & 2163.98 & 2161.07 & 1 & $\dagger$ & 0.00\% & 0.00\%\\
 & \PBPTRWS & 44.07 & 35.46 & 2 & 382+121 & 0.00\% & 0.00\%\\

\hline
2O60 & \IRICPLEX & 12085.40 & 3007.95 & 6 & $\dagger$ & 4.22\% & 12.81\%\\
 & \IRITRWS & 7.74 & 6.50 & 3 & 1000+55 & 4.94\% & 15.55\%\\
 & \MMQPBO & 7516.00 & 0.00 & 0 & $\dagger$ & 0.00\% & 0.00\%\\
 & \MQPBO & 1997.00 & 0.00 & 0 & $\dagger$ & 0.00\% & 0.00\%\\
 & \PBPCPLEX & 6137.07 & 6128.14 & 1 & $\dagger$ & 0.00\% & 0.00\%\\
 & \PBPTRWS & 93.87 & 46.42 & 2 & 352+369 & 0.00\% & 0.00\%\\

\hline
3BXL & \IRICPLEX & 3247.11 & 915.86 & 7 & $\dagger$ & 4.97\% & 17.18\%\\
 & \IRITRWS & 7.44 & 5.90 & 4 & 1000+60 & 4.66\% & 12.35\%\\
 & \MMQPBO & 6709.00 & 0.00 & 0 & $\dagger$ & 0.00\% & 0.00\%\\
 & \MQPBO & 1291.00 & 0.00 & 0 & $\dagger$ & 0.00\% & 0.00\%\\
 & \PBPCPLEX & 1776.23 & 1598.07 & 2 & $\dagger$ & 0.00\% & 0.00\%\\
 & \PBPTRWS & 44.71 & 25.52 & 2 & 227+216 & 0.00\% & 0.00\%\\

\hline
pdb1b25 & \IRICPLEX & 1599.67 & 55.01 & 28 & $\dagger$ & 76.76\% & 84.05\%\\
 & \IRITRWS & 5.18 & 2.92 & 18 & 530+150 & 83.00\% & 87.84\%\\
 & \MMQPBO & 27.00 & 0.00 & 0 & $\dagger$ & 0.00\% & 2.53\%\\
 & \MQPBO & 2.00 & 0.00 & 0 & $\dagger$ & 0.00\% & 1.99\%\\
 & \PBPCPLEX & 324.64 & 72.11 & 14 & $\dagger$ & 18.84\% & 22.32\%\\
 & \PBPTRWS & 119.71 & 27.62 & 14 & 443+1238 & 18.92\% & 22.34\%\\

\hline
pdb1d2e & \IRICPLEX & 154.76 & 25.44 & 5 & $\dagger$ & 97.30\% & 97.98\%\\
 & \IRITRWS & 1.67 & 1.13 & 7 & 420+75 & 96.97\% & 98.25\%\\
 & \MMQPBO & 12.00 & 0.00 & 0 & $\dagger$ & 0.00\% & 4.61\%\\
 & \MQPBO & 0.00 & 0.00 & 0 & $\dagger$ & 0.00\% & 2.74\%\\
 & \PBPCPLEX & 483.55 & 34.59 & 25 & $\dagger$ & 55.53\% & 58.94\%\\
 & \PBPTRWS & 83.82 & 6.16 & 47 & 190+2775 & 55.69\% & 58.98\%\\

\hline
pdb1fmj & \IRICPLEX & 99.33 & 12.35 & 7 & $\dagger$ & 92.58\% & 94.90\%\\
 & \IRITRWS & 1.05 & 0.60 & 14 & 540+135 & 83.18\% & 87.09\%\\
 & \MMQPBO & 6.00 & 0.00 & 0 & $\dagger$ & 0.00\% & 2.92\%\\
 & \MQPBO & 0.00 & 0.00 & 0 & $\dagger$ & 0.00\% & 2.04\%\\
 & \PBPCPLEX & 77.30 & 16.97 & 11 & $\dagger$ & 15.94\% & 18.83\%\\
 & \PBPTRWS & 16.67 & 3.10 & 11 & 186+677 & 16.18\% & 18.91\%\\

\hline
pdb1i24 & \IRITRWS & 0.06 & 0.02 & 2 & 60+5 & 99.73\% & 99.94\%\\
 & \MMQPBO & 3.00 & 0.00 & 0 & $\dagger$ & 0.00\% & 2.85\%\\
 & \MQPBO & 0.00 & 0.00 & 0 & $\dagger$ & 0.00\% & 3.43\%\\
 & \PBPCPLEX & 5.66 & 5.66 & 0 & $\dagger$ & 100.00\% & 100.00\%\\
 & \PBPTRWS & 0.82 & 0.82 & 0 & 115+0 & 100.00\% & 100.00\%\\

\hline
pdb1iqc & \IRICPLEX & 111.58 & 18.51 & 5 & $\dagger$ & 99.10\% & 99.63\%\\
 & \IRITRWS & 0.74 & 0.40 & 8 & 200+35 & 96.39\% & 97.10\%\\
 & \MMQPBO & 8.00 & 0.00 & 0 & $\dagger$ & 0.00\% & 6.06\%\\
 & \MQPBO & 0.00 & 0.00 & 0 & $\dagger$ & 0.00\% & 4.90\%\\
 & \PBPCPLEX & 229.09 & 24.36 & 20 & $\dagger$ & 35.32\% & 41.15\%\\
 & \PBPTRWS & 36.03 & 4.06 & 28 & 169+2058 & 40.50\% & 45.56\%\\

\hline
pdb1jmx & \IRICPLEX & 142.20 & 15.52 & 9 & $\dagger$ & 97.24\% & 98.69\%\\
 & \IRITRWS & 0.67 & 0.29 & 10 & 200+75 & 93.46\% & 95.83\%\\
 & \MMQPBO & 8.00 & 0.00 & 0 & $\dagger$ & 0.00\% & 3.76\%\\
 & \MQPBO & 0.00 & 0.00 & 0 & $\dagger$ & 0.00\% & 3.73\%\\
 & \PBPCPLEX & 121.71 & 16.21 & 19 & $\dagger$ & 35.86\% & 39.98\%\\
 & \PBPTRWS & 20.02 & 3.59 & 24 & 188+1098 & 35.26\% & 39.12\%\\

\hline
pdb1kgn & \IRICPLEX & 196.03 & 17.97 & 10 & $\dagger$ & 89.22\% & 93.23\%\\
 & \IRITRWS & 1.37 & 0.76 & 9 & 400+170 & 88.92\% & 93.16\%\\
 & \MMQPBO & 9.00 & 0.00 & 0 & $\dagger$ & 0.00\% & 3.24\%\\
 & \MQPBO & 0.00 & 0.00 & 0 & $\dagger$ & 0.00\% & 2.27\%\\
 & \PBPCPLEX & 161.57 & 24.37 & 12 & $\dagger$ & 39.42\% & 39.67\%\\
 & \PBPTRWS & 53.49 & 6.45 & 17 & 268+1824 & 13.20\% & 13.36\%\\

\hline
pdb1kwh & \IRICPLEX & 105.77 & 9.63 & 10 & $\dagger$ & 79.09\% & 85.64\%\\
 & \IRITRWS & 0.46 & 0.27 & 8 & 440+50 & 76.53\% & 83.26\%\\
 & \MMQPBO & 5.00 & 0.00 & 0 & $\dagger$ & 0.00\% & 2.99\%\\
 & \MQPBO & 0.00 & 0.00 & 0 & $\dagger$ & 0.00\% & 3.43\%\\
 & \PBPCPLEX & 51.33 & 12.89 & 9 & $\dagger$ & 25.54\% & 31.15\%\\
 & \PBPTRWS & 9.15 & 2.43 & 8 & 208+401 & 25.43\% & 31.13\%\\

\hline
pdb1m3y & \IRICPLEX & 73.60 & 18.58 & 3 & $\dagger$ & 98.54\% & 99.47\%\\
 & \IRITRWS & 0.79 & 0.65 & 3 & 340+10 & 97.53\% & 99.08\%\\
 & \MMQPBO & 8.00 & 0.00 & 0 & $\dagger$ & 0.00\% & 6.38\%\\
 & \MQPBO & 0.00 & 0.00 & 0 & $\dagger$ & 0.00\% & 5.72\%\\
 & \PBPCPLEX & 120.60 & 25.18 & 14 & $\dagger$ & 31.05\% & 27.97\%\\
 & \PBPTRWS & 28.19 & 4.82 & 12 & 200+1135 & 31.08\% & 27.98\%\\

\hline
pdb1qks & \IRICPLEX & 138.12 & 15.19 & 8 & $\dagger$ & 98.30\% & 98.93\%\\
 & \IRITRWS & 0.30 & 0.12 & 4 & 80+20 & 98.57\% & 99.38\%\\
 & \MMQPBO & 9.00 & 0.00 & 0 & $\dagger$ & 0.00\% & 5.09\%\\
 & \MQPBO & 0.00 & 0.00 & 0 & $\dagger$ & 0.00\% & 3.68\%\\
 & \PBPCPLEX & 96.77 & 15.82 & 12 & $\dagger$ & 28.18\% & 26.37\%\\
 & \PBPTRWS & 27.99 & 3.24 & 15 & 161+1154 & 30.63\% & 28.46\%\\

\hline
 \multicolumn{ 8 }{c}{ color-seg } \\\hline

\hline
colseg-cow3 & \IRITRWS & 66.30 & 48.10 & 6 & 1000+140 & 99.96\% & 99.97\%\\
 & \Kovtun & 1.00 & 0.00 & 0 & $\dagger$ & 0.00\% & 99.89\%\\
 & \MMQPBO & 206.00 & 0.00 & 0 & $\dagger$ & 0.00\% & 43.55\%\\
 & \MQPBO & 24.00 & 0.00 & 0 & $\dagger$ & 0.00\% & 32.06\%\\
 & \PBPTRWS & 7530.72 & 690.55 & 14 & 826+6056 & 99.95\% & 99.95\%\\

\hline
colseg-cow4 & \IRITRWS & 91.26 & 48.31 & 13 & 1000+310 & 99.92\% & 99.93\%\\
 & \Kovtun & 2.00 & 0.00 & 0 & $\dagger$ & 0.00\% & 99.90\%\\
 & \MMQPBO & 46.00 & 0.00 & 0 & $\dagger$ & 0.00\% & 0.56\%\\
 & \MQPBO & 40.00 & 0.00 & 0 & $\dagger$ & 0.00\% & 0.37\%\\
 & \PBPTRWS & 7395.03 & 742.58 & 10 & 848+6349 & 99.80\% & 99.80\%\\

\hline
colseg-garden4 & \IRITRWS & 0.49 & 0.15 & 5 & 70+20 & 99.91\% & 99.94\%\\
 & \Kovtun & 0.00 & 0.00 & 0 & $\dagger$ & 0.00\% & 94.96\%\\
 & \MMQPBO & 14.00 & 0.00 & 0 & $\dagger$ & 0.00\% & 4.27\%\\
 & \MQPBO & 1.00 & 0.00 & 0 & $\dagger$ & 0.00\% & 0.21\%\\
 & \PBPTRWS & 33.68 & 6.75 & 5 & 167+488 & 99.89\% & 99.89\%\\

\hline
 \multicolumn{ 8 }{c}{ color-seg-n4 } \\\hline

\hline
clownfish-small & \IRITRWS & 1.72 & 0.68 & 3 & 80+10 & $>$99.99\% & $>$99.99\%\\
 & \Kovtun & 1.00 & 0.00 & 0 & $\dagger$ & 0.00\% & 74.11\%\\
 & \MMQPBO & 536.00 & 0.00 & 0 & $\dagger$ & 0.00\% & 15.83\%\\
 & \MQPBO & 41.00 & 0.00 & 0 & $\dagger$ & 0.00\% & 4.67\%\\
 & \PBPTRWS & 151.98 & 30.01 & 6 & 223+610 & 99.97\% & 99.97\%\\

\hline
crops-small & \IRITRWS & 1.87 & 1.02 & 2 & 120+5 & 100.00\% & 100.00\%\\
 & \Kovtun & 1.00 & 0.00 & 0 & $\dagger$ & 0.00\% & 64.70\%\\
 & \MMQPBO & 577.00 & 0.00 & 0 & $\dagger$ & 0.00\% & 14.32\%\\
 & \MQPBO & 33.00 & 0.00 & 0 & $\dagger$ & 0.00\% & 0.71\%\\
 & \PBPTRWS & 677.08 & 34.88 & 40 & 260+3578 & 99.00\% & 99.00\%\\

\hline
fourcolors & \IRITRWS & 0.57 & 0.08 & 2 & 20+5 & 99.96\% & 99.97\%\\
 & \Kovtun & 0.00 & 0.00 & 0 & $\dagger$ & 0.00\% & 69.52\%\\
 & \MMQPBO & 37.00 & 0.00 & 0 & $\dagger$ & 0.00\% & 0.00\%\\
 & \MQPBO & 3.00 & 0.00 & 0 & $\dagger$ & 0.00\% & 0.00\%\\
 & \PBPTRWS & 31.28 & 2.60 & 8 & 34+238 & 99.92\% & 99.92\%\\

\hline
lake-small & \IRITRWS & 1.28 & 0.43 & 2 & 50+5 & 100.00\% & 100.00\%\\
 & \Kovtun & 1.00 & 0.00 & 0 & $\dagger$ & 0.00\% & 74.87\%\\
 & \MMQPBO & 607.00 & 0.00 & 0 & $\dagger$ & 0.00\% & 15.31\%\\
 & \MQPBO & 31.00 & 0.00 & 0 & $\dagger$ & 0.00\% & 6.65\%\\
 & \PBPTRWS & 13.75 & 13.75 & 0 & 95+-95 & 100.00\% & 100.00\%\\

\hline
palm-small & \IRITRWS & 2.48 & 1.37 & 3 & 160+10 & $>$99.99\% & $>$99.99\%\\
 & \Kovtun & 1.00 & 0.00 & 0 & $\dagger$ & 0.00\% & 68.65\%\\
 & \MMQPBO & 510.00 & 0.00 & 0 & $\dagger$ & 0.00\% & 0.48\%\\
 & \MQPBO & 19.00 & 0.00 & 0 & $\dagger$ & 0.00\% & 0.00\%\\
 & \PBPTRWS & 846.27 & 39.97 & 19 & 291+4582 & 98.20\% & 98.20\%\\

\hline
penguin-small & \IRITRWS & 1.21 & 0.54 & 2 & 90+5 & 100.00\% & 100.00\%\\
 & \Kovtun & 0.00 & 0.00 & 0 & $\dagger$ & 0.00\% & 91.99\%\\
 & \MMQPBO & 193.00 & 0.00 & 0 & $\dagger$ & 0.00\% & 1.42\%\\
 & \MQPBO & 13.00 & 0.00 & 0 & $\dagger$ & 0.00\% & 1.03\%\\
 & \PBPTRWS & 15.67 & 15.67 & 0 & 152+-152 & 100.00\% & 100.00\%\\

\hline
pfau-small & \IRITRWS & 18.77 & 7.22 & 48 & 950+470 & 89.43\% & 93.41\%\\
 & \Kovtun & 1.00 & 0.00 & 0 & $\dagger$ & 0.00\% & 5.59\%\\
 & \MMQPBO & 591.00 & 0.00 & 0 & $\dagger$ & 0.00\% & 0.70\%\\
 & \MQPBO & 16.00 & 0.00 & 0 & $\dagger$ & 0.00\% & 0.00\%\\
 & \PBPTRWS & 799.08 & 79.34 & 44 & 654+10857 & 10.43\% & 10.43\%\\

\hline
snail & \IRITRWS & 0.79 & 0.23 & 2 & 50+5 & 99.99\% & 99.99\%\\
 & \Kovtun & 0.00 & 0.00 & 0 & $\dagger$ & 0.00\% & 97.77\%\\
 & \MMQPBO & 7.00 & 0.00 & 0 & $\dagger$ & 0.00\% & 77.91\%\\
 & \MQPBO & 1.00 & 0.00 & 0 & $\dagger$ & 0.00\% & 58.35\%\\
 & \PBPTRWS & 46.20 & 6.47 & 5 & 83+332 & 99.98\% & 99.98\%\\

\hline
strawberry-glass-2-small & \IRITRWS & 1.35 & 0.60 & 2 & 80+5 & 100.00\% & 100.00\%\\
 & \Kovtun & 1.00 & 0.00 & 0 & $\dagger$ & 0.00\% & 54.99\%\\
 & \MMQPBO & 528.00 & 0.00 & 0 & $\dagger$ & 0.00\% & 2.78\%\\
 & \MQPBO & 39.00 & 0.00 & 0 & $\dagger$ & 0.00\% & 0.00\%\\
 & \PBPTRWS & 311.54 & 31.00 & 11 & 259+1721 & 99.31\% & 99.31\%\\

\hline
 \multicolumn{ 8 }{c}{ mrf-photomontage } \\\hline

\hline
family-gm & \IRITRWS & 286.40 & 93.08 & 77 & 1000+1265 & 4.75\% & 4.80\%\\
 & \MMQPBO & 1087.00 & 0.00 & 0 & $\dagger$ & 0.00\% & 4.41\%\\
 & \MQPBO & 90.00 & 0.00 & 0 & $\dagger$ & 0.00\% & 4.34\%\\
 & \PBPTRWS & 12726.45 & 1291.11 & 50 & 1015+22483 & 4.41\% & 4.41\%\\

\hline
pano-gm & \IRITRWS & 320.00 & 112.17 & 59 & 1000+1105 & 67.73\% & 79.17\%\\
 & \MMQPBO & 646.00 & 0.00 & 0 & $\dagger$ & 0.00\% & 28.06\%\\
 & \MQPBO & 97.00 & 0.00 & 0 & $\dagger$ & 0.00\% & 40.37\%\\
 & \PBPTRWS & 14360.45 & 1871.14 & 33 & 911+11193 & 27.55\% & 27.55\%\\

\hline
 \multicolumn{ 8 }{c}{ mrf-stereo } \\\hline

\hline
ted-gm & \IRITRWS & 231.97 & 72.67 & 119 & 1000+715 & 67.27\% & 72.05\%\\
 & \PBPTRWS & 3837.51 & 436.30 & 28 & 689+10383 & 38.13\% & 38.13\%\\

\hline
tsu-gm & \IRITRWS & 19.75 & 14.67 & 10 & 670+75 & 99.91\% & 99.94\%\\
 & \PBPTRWS & 9277.99 & 267.55 & 54 & 377+17421 & 0.39\% & 0.39\%\\

\hline
ven-gm & \IRITRWS & 108.73 & 94.44 & 9 & 1000+40 & 0.01\% & 0.02\%\\
 & \PBPTRWS & 14737.47 & 1451.83 & 55 & 993+16592 & 0.00\% & 0.00\%\\
\bottomrule
\caption{%
\protect\vphantom{$g^{g^{g^g}}$}%
Detailed experimental evaluation for Algorithm~\ref{alg:iterative-LP} utilising CPLEX~\cite{CPLEX} as a subsolver, denoted as \IRICPLEX, Algorithm~\ref{alg:iterative-DualLP} utilising TRW-S~\cite{TRWSKolmogorov} as a subsolver, denoted as \IRITRWS,
   their counterparts from~\cite{SwobodaPersistencyCVPR2014} denoted by \PBPCPLEX and \PBPTRWS
	   and MQPBO~\cite{PartialOptimalityInMultiLabelMRFsKohli} run for one iteration with predefined label order, denoted by \MQPBO, and run 10 iterations in 10 random label orders, denoted by \MMQPBO.}
\label{tab:DetailedExperimentalEvaluation}
\end{longtable}
}
\twocolumn